\documentclass{article} 
\usepackage{iclr2024_conference,times}

\usepackage{amsmath,amsfonts,bm}

\def\eqref#1{equation~\ref{#1}}

\def\1{\bm{1}}

\DeclareMathAlphabet{\mathsfit}{\encodingdefault}{\sfdefault}{m}{sl}
\SetMathAlphabet{\mathsfit}{bold}{\encodingdefault}{\sfdefault}{bx}{n}

\newcommand{\Cov}{\mathrm{Cov}}

\usepackage{amsmath}
\usepackage{amssymb}
\usepackage{mathtools}
\usepackage{amsthm}
\usepackage{bbm}
\usepackage{algorithmic}
\usepackage[ruled,vlined,noend,linesnumbered]{algorithm2e}
\SetKwRepeat{Do}{do}{while}

\usepackage{hyperref}
\usepackage{url}

\SetNlSty{}{}{}
\let\oldnl\nl
\newcommand\nonl{%
  \renewcommand{\nl}{\let\nl\oldnl}}
  
\SetCommentSty{mycommfont}

\usepackage{caption}
\usepackage{subcaption}
\usepackage{enumitem}
\usepackage[page]{appendix}

\usepackage{thmtools,thm-restate}
\newtheorem{theorem}{Theorem}
\newtheorem{corollary}{Corollary}
\declaretheorem[sibling=theorem]{lemma}
\newtheorem{definition}{Definition}

\newtheorem{condition}{Condition}
\newtheorem{example}{Example}
\newtheorem{assumption}{Assumption}
\newtheorem{remark}{Remark}

\newcommand{\indep}{\perp \!\!\! \perp}
\newcommand{\node}[1]{\mathsf{#1}}
\newcommand{\set}[1]{\mathbf{#1}}
\newcommand{\setset}[1]{\mathcal{#1}}
\newcommand{\setsetset}[1]{\mathbb{#1}}
\newcommand{\graph}{\mathcal{G}}
\newcommand{\graphp}{\mathcal{G'}}
\newcommand{\graphpp}{\mathcal{G''}}

\newcommand{\descendant}{\text{De}_{\graph}}

\newcommand{\parents}{\text{Pa}_{\graph}}
\newcommand{\parentsp}{\text{Pa}_{\graphp}}

\newcommand{\purechildren}{\text{PCh}_{\graph}}
\newcommand{\purechildrenp}{\text{PCh}_{\graphp}}

\newcommand{\meassureddes}{\text{MDe}_{\graph}}

\newcommand{\sep}{\text{Sep}}

\title{
A Versatile Causal Discovery Framework to Allow Causally-Related Hidden Variables
}

\author{%
    Xinshuai Dong$^{*1}$ \quad
    Biwei Huang \thanks{Equal contribution.} ~$^{2}$  \quad
    Ignavier Ng$^{1}$ \quad
    Xiangchen Song$^{1}$ \quad
    Yujia Zheng$^{1}$ \quad \\
    \textbf{Songyao Jin}$^{4}$ \quad
    \textbf{Roberto Legaspi}$^{3}$ \quad
    \textbf{Peter Spirtes}$^{1}$ \quad
    \textbf{Kun Zhang}$^{1,4}$\\
    $^1$Carnegie Mellon University\\
    $^2$University of California San Diego\\
    $^3$KDDI Research\\
    $^4$Mohamed bin Zayed University of Artificial Intelligence
}

\iclrfinalcopy 
\begin{document}

\maketitle

\begin{abstract}
  %
  
Most existing causal discovery methods rely on the assumption of no latent confounders, limiting their applicability in solving real-life problems. In this paper, we introduce a novel, versatile framework for causal discovery that accommodates the presence of causally-related hidden variables
 almost everywhere 
 in the causal network (for instance, they can be effects of observed variables), based on rank information of covariance matrix over observed variables. We start by investigating the efficacy of rank in comparison to conditional independence and, theoretically, establish necessary and sufficient conditions for the identifiability of certain latent structural patterns. Furthermore, we develop a  Rank-based Latent Causal Discovery algorithm, RLCD, that can efficiently locate hidden variables, determine their cardinalities, and discover the entire causal structure over both measured and hidden ones.
We also show that, under certain graphical conditions, RLCD correctly identifies the Markov Equivalence Class of the whole latent causal graph asymptotically. Experimental results on both synthetic and real-world personality data sets demonstrate the efficacy of the proposed approach in finite-sample cases. Our code will be publicly available.
    
\end{abstract}

\section{Introduction and Related Work}
Causal discovery aims at finding causal relationships from observational data
and has received successful applications in many fields \citep{spirtes2000causation,spirtes2010automated,pearl2019seven}.
However, traditional methods, such as PC \citep{spirtes2000causation}, GES \citep{chickering2002optimal}, and LiNGAM \citep{shimizu2006linear}, generally assume that there are no latent confounders in the graph, which hardly holds in many real-world scenarios. 
%
%
%
Extensive efforts have been dedicated to addressing this issue for causal structure learning.
%
{
One line of research 
focuses on inferring the causal structure among the observed variables, despite the possible existence of latent confounders.
Notable approaches include FCI
and its variants \citep{spirtes2000causation,Pearl:2000:CMR:331969,colombo2012learning,Akbari2021} 
 that leverage conditional independence tests 
 }, and over-complete ICA-based techniques~\citep{hoyer2008estimation,salehkaleybar2020learning}
that further leverage non-Gaussianity. 

Another line of thought 
 focuses more on uncovering the causal structure among latent variables, by assuming observed variables are not directly adjacent. This includes Tetrad condition-based \citep{Silva-linearlvModel, Kummerfeld2016}, high-order moments-based
 \citep{shimizu2009estimation,cai2019triad,xie2020generalized,adams2021identification,chen2022identification}, 
 matrix decomposition-based  \citep{anandkumar2013learning}, and  mixture oracles-based  \citep{kivva2021learning} approaches. Recently, \citet{huang2022latent} propose an approach that makes use of rank constraints to identify general latent hierarchical structures,
 and yet observed variables can only be leaf nodes. 
 Although \citet{RankSparsity_12} allow direct causal influences within observed variables and the existence of latent variables, it cannot recover the causal relationships among latent variables and has strong graphical constraints.
 For a more detailed discussion of related work, please refer to our Appx.~\ref{sec appendix: related work}.
 


 In this paper, we aim to handle a more general scenario for causal discovery with latent variables, where observed variables are allowed to be directly adjacent, and latent variables to be flexibly related to all the other variables.  That is, hidden variables can serve as confounders, mediators, or effects of latent or observed variables,
and even form a hierarchical structure
(see an illustrated example in Figure \ref{fig:example1}).
This setting is rather general and practically meaningful to deal with many real-world problems.


To address such a challenging problem,  we are confronted with three fundamental questions:
(i) What information and constraints can be discovered from the observed variables to reveal the underlying causal structure? (ii) How can we effectively and efficiently search for these constraints? (iii) What graphical conditions are needed to uniquely locate latent variables and ascertain the complete causal structure?  Remarkably, these questions can be addressed by harnessing the power of rank deficiency constraints on the covariance of observed variables. By carefully identifying and utilizing rank properties in specific ways, we are able to determine the Markov equivalence class of the entire graph. Our  contributions are mainly three-fold:
\begin{itemize}[leftmargin=*, topsep=-1pt, itemsep=0pt]
    \item  We investigate the efficacy of rank in comparison to conditional independence  in latent causal graph discovery, and theoretically introduce necessary and sufficient conditions for the identifiability of certain latent structural properties. For instance, the condition we proposed for nonadjacency generalizes the counterpart in \citet{spirtes2000causation} to graphs with latent variables.
    \item We develop RLCD, an efficient 
    three-phase causal discovery algorithm that is able to locate latent variables, determine their cardinalities, and identify the whole causal structure involving measured and latent variables, by properly leveraging rank properties.
    In the special case with no latent variables, it asymptotically returns the same graph as the PC algorithm \citep{spirtes2000causation} does.
    \item We provide a set of graphical conditions that are sufficient for RLCD to asymptotically identify the correct Markov Equivalence Class of the latent causal graph; notably, these graphical conditions are significantly weaker than those in previous works. 
    Our empirical study on both synthetic and real-world datasets validates RLCD on finite samples. 
\end{itemize}

\begin{figure}[t]
  \centering 
  \includegraphics[width=.6\linewidth]{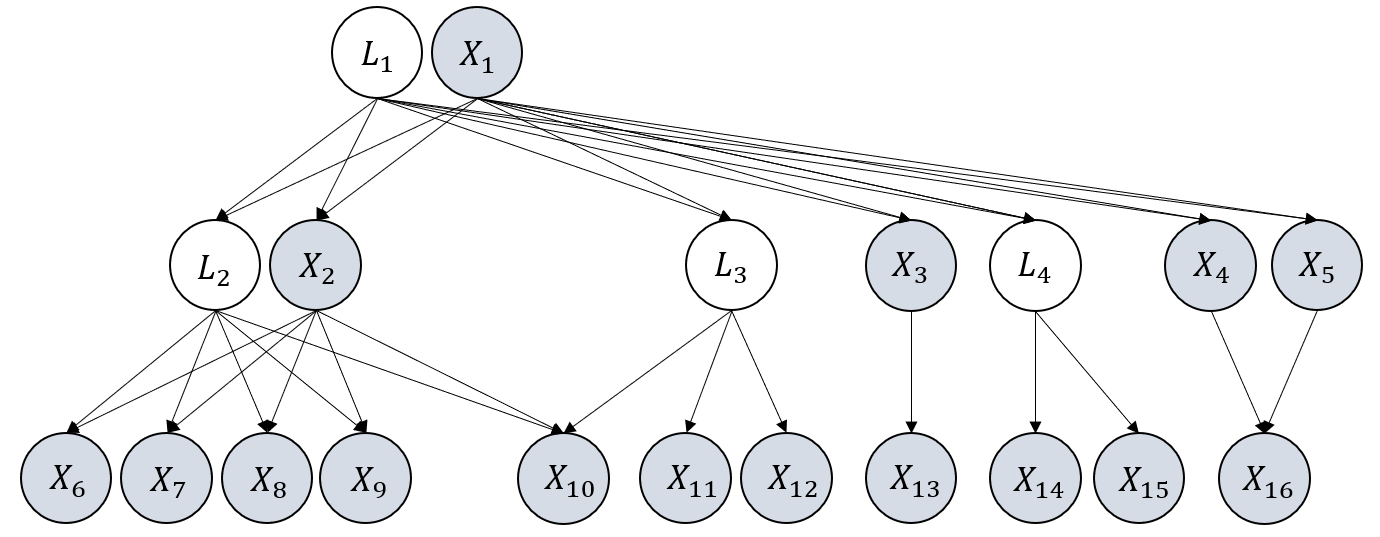}
  \caption{\small An illustrative graph that we aim to handle, where latent variables are denoted by $\node{L}$ and observed variables are denoted by $\node{X}$. The latent variables can act as a cause, effect, or mediator for both observed variables and other latent variables. See Appx. \ref{compare graphs with each method} for a comparison of graphs that each method can handle.}  \label{fig:example1}
\end{figure}

\section{Problem Setting}
\label{sec:pre}
In this paper, we aim to identify the causal structure of a latent linear causal model defined as follows.

\begin{definition} (Latent Linear Causal Models)
  Suppose a directed acyclic graph $\graph:=(\set{V}_\graph,\set{E}_\graph)$,
  where each variable $V_i \in \mathbf{V}_{\mathcal{G}}$ is generated following a linear causal structural model:
\begin{align}
  \label{eq:lem}
\node{V}_i=\sum \nolimits_{\node{V}_j \in \parents(\node{V}_i)} a_{ij} \node{V}_j + \varepsilon_{\node{V}_i},
\end{align}
where $\set{V}_\graph:=\set{L}_\graph\cup\set{X}_\graph$ contains a set of $n+m$ random variables, 
with $m$ latent variables $\set{L}_\graph:=\{\node{L}_i\}_{i=1}^{m}$,
and $n$ observed variables $\set{X}_\graph:=\{\node{X}_i\}_{i=1}^{n}$.
$\parents(V_i)$ denotes the parent set of $V_i$, $a_{ij}$ the causal coefficient from $V_j$ to $V_i$, and $\varepsilon_{\node{V}_i}$ represents the noise term.
\label{definition:llcm}
\end{definition}

We further have a basic assumption for latent linear causal models given as follows.

\begin{assumption} [Basic Assumptions for Latent Linear Causal Models] 
   (i) Leaf nodes are observed; or equivalently, a latent variable should have at least one observed descendant. 
   (ii) Rank faithfulness. A probability distribution $p$ is rank faithful to $\mathcal{G}$ if every rank constraint on a sub-covariance matrix that holds in $p$ is entailed by every linear structural model with respect to $\mathcal{G}$. 
   \label{assumption:basic}
   \end{assumption}  

 The inclusion of Assumption~\ref{assumption:basic} does not compromise the generality.
 If (i) does not hold,
  we can simply remove the latent variables
  that lack observed descendants, 
  since they provide no information that can be inferred for any other variable.
  Further, (ii) is the classical faithfulness assumption that is critical and prevalent in causal discovery \citep{spirtes2000causation,huang2022latent};
  it holds generically on infinite data, as the set of values of the SCM's free parameters 
  for which rank is not faithful is of Lebesgue measure 0   \citep{spirtes2013calculation-t-separation}. 
   On the other hand, if faithfulness is violated 
   (an example in Appx.~\ref{violation of faithfulness}), 
     even classical methods like PC cannot guarantee asymptotic correctness.

Our objective is to identify the underlying causal structure $\graph$ over all the variables $\set{L}_\graph\cup\set{X}_\graph$ (detailed in Sec.~\ref{sec:theory})
that are generated according to a latent linear causal model,
given i.i.d. samples of observed variables $\set{X}_\graph$ only. 
To address this challenging problem,
 traditional wisdom often relies on strong 
graphical constraints  \citep{Pearl88,zhang2004hierarchical,huang2022latent,maeda2020rcd}(detailed in Appx.~\ref{compare graphs with each method} with illustrative graphs).
In contrast, Definition~\ref{definition:llcm} allows all the variables including observed and latent variables to be very flexibly related. 
We basically allow the presence of edges between any two variables such that
a node $\node{V}$, no matter whether it is observed or not, can act as a cause, effect, or mediator for both observed and latent variables.

 A summary of 
 notations 
  is  in Tab.~\ref{tab:notation}. The rest of the paper is organized as follows. In Sec.~\ref{sec:why rank},
we motivate the use of rank 
and propose conditions
for nonadjacency and the existence of latent variables.   
%
In Sec.~\ref{sec:allalgorithm}, we establish the minimal identifiable
 substructure of a linear latent  graph,
 based on which we propose RLCD for latent variable causal discovery.
 In Sec.~\ref{sec:theory}, we introduce the identifiability of RLCD.
 In Sec.~\ref{sec:exp}, we validate our method using both synthetic and real-life data.
%



\begin{table}[t]
\small
  \caption{\small Graphical notations used throughout this paper.}
  \centering
  \begin{tabular}{|l|l|l|l|}
  \hline
  Pa: Parents & $\set{V}$: Variables & $\node{V}$: Variable & $\graph$: The underlying graph\\
  \hline
  Ch: Children & $\set{L}$: Latent variables & $\node{L}$: Latent variable & $\graphp$: Output Graph \\
  \hline
  PCh: Pure children & $\set{X}$: Observed variables & $\node{X}$: Observed variable & $\setset{S}$: Set of covers\\
  \hline
  Sib: Siblings & MDe: Measured descendants & PDe: Pure descendants& $\setsetset{S}$: Set of sets of covers \\
  \hline
  \end{tabular}
\label{tab:notation}
\end{table}

%

\section{Why Use Rank Information?}
\label{sec:why rank}
In this section, 
we first motivate the use of rank constraints 
for causal discovery in the presence of latent variables,
 and then establish some fundamental theories about what rank  
 implies graphically.
 
\subsection{Preliminaries about Treks and Rank}
When there is no latent variable,
 a common approach for causal discovery 
 is to use conditional independence (CI) relationships
  to identify d-separations in a graph; see, e.g.,
   the PC algorithm \citep{spirtes2000causation}.
   The following theorem illustrates this idea.

\begin{theorem} [Conditional Independence and D-separation \citep{Pearl88}]
  Under the Markov and faithfulness assumption, for disjoint sets of variables $\set{A}$, $\set{B}$ and $\set{C}$,
  $\set{C}$ d-separates $\set{A}$ and $\set{B}$ in graph $\graph$, iff
  $\set{A} \indep \set{B} | \set{C}$ holds for every distribution in the graphical model associated to $\mathcal{G}$.
\label{theorem:CI and d-sep}
\end{theorem}

As for latent linear causal models,
 trek-separations (t-separations) provide more information than d-separations 
 (for readers who are not very familiar with treks and t-separations, 
kindly refer to Appx.~\ref{example: trek} for examples).
The definitions of treks and t-separation are given as follows, together with Theorem~\ref{theorem:tsep and dsep}
showing the relations between t-separations and d-separations.

\begin{definition} [Treks \citep{sullivant2010trek}]
   In $\graph$, a trek from $\node{X}$ to $\node{Y}$ is an ordered pair of directed paths 
   $(P_1,P_2)$ where $P_1$ has a sink $\mathsf{X}$, 
   $P_2$ has a sink $\mathsf{Y}$,
    and both $P_1$ and $P_2$ have the same source $\mathsf{Z}$. 
\end{definition}

\begin{definition} [T-separation \citep{sullivant2010trek}]
Let $\mathbf{A}$, $\mathbf{B}$, $\mathbf{C}_{\mathbf{A}}$, 
and $\mathbf{C}_{\mathbf{B}}$ be four subsets of $\mathbf{V}_{\mathcal{G}}$ 
in graph $\mathcal{G}$ (not necessarilly disjoint). 
($\mathbf{C}_{\mathbf{A}}$,$\mathbf{C}_{\mathbf{B}}$) t-separates $\mathbf{A}$ from $\mathbf{B}$ if for every trek ($P_1$,$P_2$) from a vertex in $\mathbf{A}$ to a vertex in $\mathbf{B}$, either $P_1$ contains a vertex in  $\mathbf{C}_{\mathbf{A}}$ or $P_2$ contains a vertex in  $\mathbf{C}_{\mathbf{B}}$. 
\label{definition:t-sep}
\end{definition}

\begin{theorem} [T- and D-sep \citep{di2009t}]
  For disjoint sets $\mathbf{A}$, $\mathbf{B}$ and $\mathbf{C}$,
  $\mathbf{C}$ d-separates $\mathbf{A}$ and $\mathbf{B}$ in graph $\mathcal{G}$, iff there is a partition $\mathbf{C} = \mathbf{C}_{\mathbf{A}} \cup \mathbf{C}_{\mathbf{B}}$ such that
  ($\mathbf{C}_{\mathbf{A}}$,$\mathbf{C}_{\mathbf{B}}$) t-separates $\mathbf{A}\cup\mathbf{C}$ from $\mathbf{B}\cup\mathbf{C}$.
\label{theorem:tsep and dsep}
\end{theorem}

The theorem above reveals that all d-sep can be reformulated by t-sep, and thus, t-sep encompass d-sep information. 
Just as we use CI tests to find d-sep, 
t-sep can be identified by the rank of cross-covariance matrix over specific combinations of variables,
which is formally stated as follows.

\begin{theorem} [Rank and T-separation \citep{sullivant2010trek}]
  Given two sets of variables $\set{A}$ and $\set{B}$ from a linear model with graph $\graph$, we have  
  $\text{rank}(\Sigma_{\set{A},{\set{B}}}) = 
  \min \{|\set{C}_{\set{A}}|+
  |\set{C}_{\set{B}}|:(\set{C}_{\set{A}},\set{C}_{\set{B}})~
  \text{t-separates}~\set{A}~\text{from}~\set{B}~\text{in}~\graph\}$, where $\Sigma_{\set{A},{\set{B}}}$ 
is the cross-covariance over $\set{A}$ and $\set{B}$.
\label{theorem:rank and t}
\end{theorem}

With some abuse of notation, sometimes we also use $\Sigma_{\setset{A},{\setset{B}}}$ to refer to cross-covariance over sets of sets (see Appx.~\ref{appendix: covariance} for definition and examples). 
Notably, when all variables are observed,
 rank and conditional independence are equally informative 
 about the underlying DAG.
  However, in the presence of latent variables, 
   t-separations, which can be inferred from rank by Theorem~\ref{theorem:rank and t},
   offer more graphical information compared to d-separations.
   Therefore,
    we next demonstrate 
    how rank constraints play
     a pivotal role in identifying latent causal structures.

\begin{figure}[t]
 \setlength{\belowcaptionskip}{-2.5mm}
   \begin{subfigure}[t]{0.16\textwidth}
     \includegraphics[width=\textwidth]{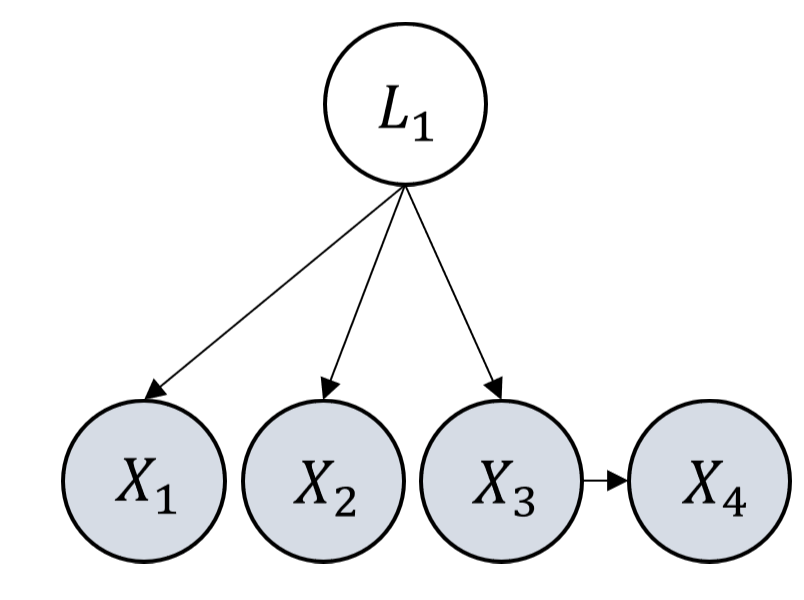}
     \caption{\centering\small $\graph_1$.}
 \end{subfigure}
 \hfill
 \begin{subfigure}[t]{0.16\textwidth}
   \includegraphics[width=\textwidth]{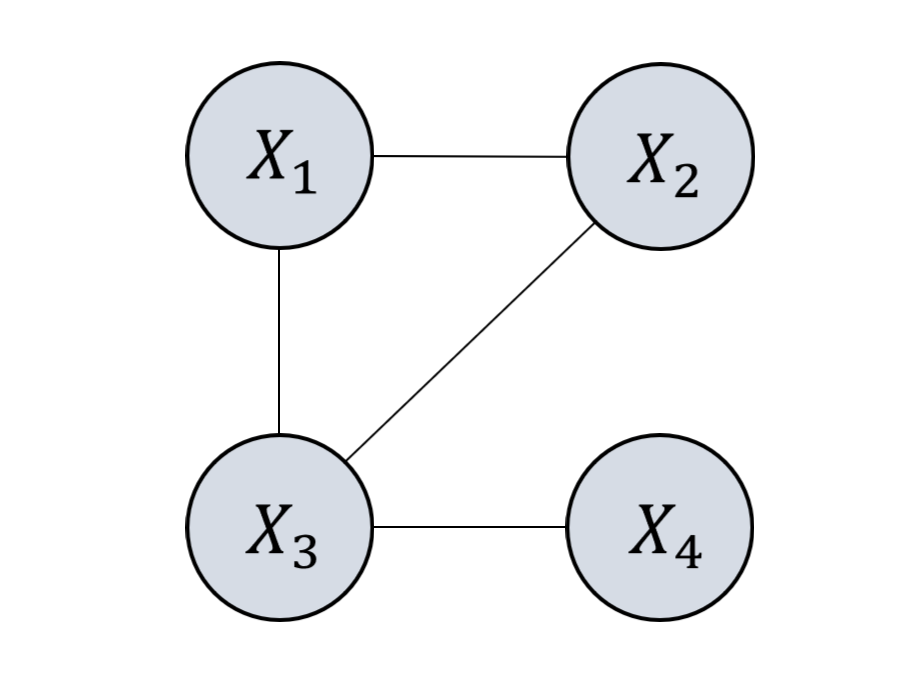}
   \caption{\centering\small   CI skeleton of $\graph_1$.}
 \end{subfigure}
 \hfill
 \begin{subfigure}[t]{0.16\textwidth}
   \includegraphics[width=\textwidth]{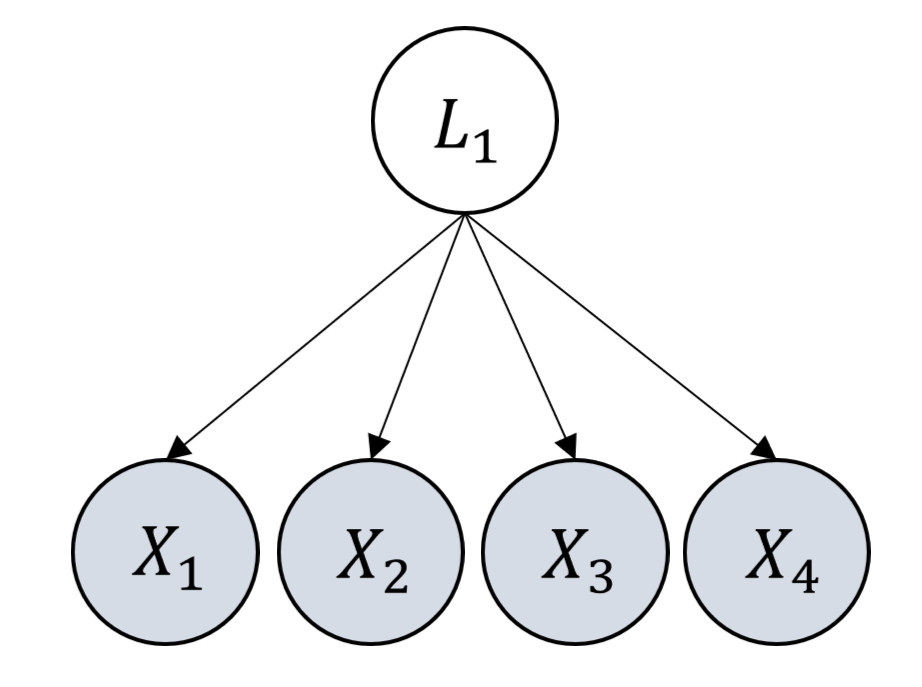}
   \caption{\centering\small $\graph_2$.}
 \end{subfigure}
 \hfill
 \begin{subfigure}[t]{0.16\textwidth}
   \includegraphics[width=\textwidth]{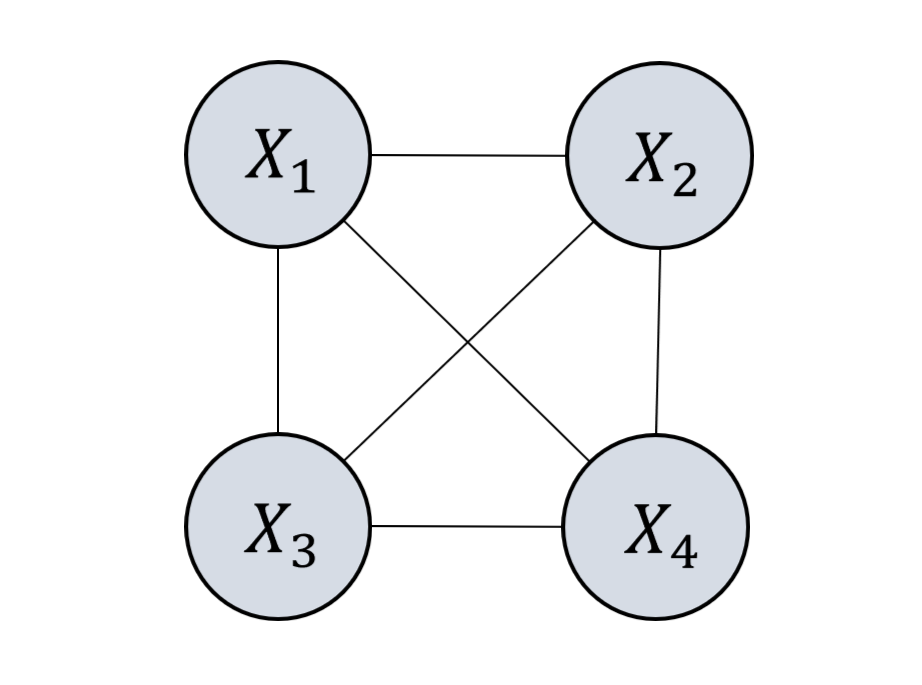}
   \caption{\centering\small   CI skeleton of $\graph_2$.}
 \end{subfigure}
 \hfill
 \begin{subfigure}[t]{0.16\textwidth}
   \includegraphics[width=\textwidth]{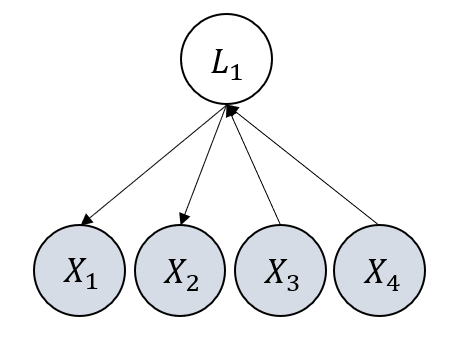}
   \caption{\centering\small $\graph_3$.}
 \end{subfigure}
 \hfill
 \begin{subfigure}[t]{0.16\textwidth}
   \includegraphics[width=\textwidth]{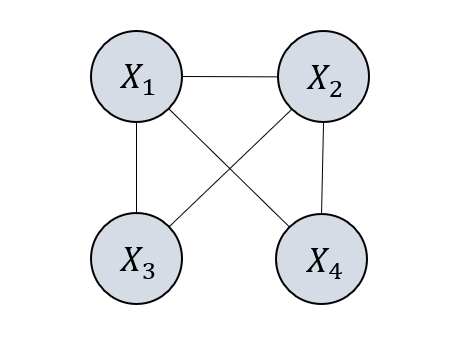}
   \caption{\centering\small    CI skeleton of $\graph_3$.}
 \end{subfigure}
 \caption{\small Examples that illustrate the basic intuition for a latent to be identifiable.}
   \label{fig:intuition of atomic cover}
 \end{figure}

\subsection{Rank: An Informative Graphical Indicator for Latent Variables}
In the presence of latent variables, CI is not enough: FCI  
and its variants 
make full use of CI
but only recover a representation which is not informative enough about the latent confounders.
 Fortunately, leveraging rank information
 can naturally make causal discovery results more informative.

An example highlighting the greater informativeness of rank compared to CI is as follows. 
Consider the graph $\graph_1$ in Fig.~\ref{fig:motiv rank},
where 
$\{\node{X}_1,\node{X}_2\}$ 
 and $\{\node{X}_3,\node{X}_4\}$ are d-separated by $\node{L}_1$,
 but we cannot infer that from a CI test 
 (i.e., whether $\{\node{X}_1,\node{X}_2\} \indep \{\node{X}_3,\node{X}_4\}|\node{L}_1$),
as $\node{L}_1$ is not observed.
In contrast, with rank information, we can 
infer that 
$\text{rank}(\Sigma_{\{\node{X}_1\node{X}_2\},\{\node{\node{X}_3\node{X}_4}\}})=1$, which 
implies $\{\node{X}_1,\node{X}_2\}$ and $\{\node{X}_3,\node{X}_4\}$ are t-separated 
by one latent variable. 
The rationale behind is that the t-sep of $\set{A}$, $\set{B}$ by $(\set{C}_\set{A},
\set{C}_\set{B})$ can be deduced through rank (as in Theorem~\ref{theorem:rank and t}) 
without observing any element in $(\set{C}_\set{A},\set{C}_\set{B})$.

%
%

With this intuition in mind, below we present three theorems that characterize the graphical implications of rank constraints, in scenarios where latent variables might exist: Theorem~\ref{theorem: condition for not adjacent} gives conditions for observed variables to be nonadjacent,  
illustrated by Example~\ref{example: nonadjacency};  Theorem~\ref{theorem: global for latent} gives
conditions for the existence of latent variables, illustrated by Example~\ref{example: exist latent}; Theorem~\ref{theorem:measurement as surrogate} implies how to utilize pure children as surrogates for calculating rank, illustrated with Example~\ref{example:measurement as surrogate}. All proofs are in Appendix.

\begin{theorem} [Condition for Nonadjacency]
   Consider a latent linear causal model.
   Two observed variables $\node{X_1}$,$\node{X_2}$ $\in \set{X}_\graph$ are not adjacent,
   if there exist two
   sets $\set{A}$,$\set{B}$ $\subseteq \set{X}_\graph \backslash \{\node{X_1},\node{X_2}\}$ that are not necessarily disjoint, such that
   $\text{rank}(\Sigma_{\set{A}\cup\{\node{X_1}\},\set{B}\cup\{\node{X_2}\}})=\text{rank}(\Sigma_{\set{A},\set{B}})$ and $\text{rank}(\Sigma_{\set{A}\cup\{\node{X_1},\node{X_2}\},\set{B}\cup\{\node{X_1},\node{X_2}\}})=\text{rank}(\Sigma_{\set{A},\set{B}})+2$. 
   \label{theorem: condition for not adjacent}
\end{theorem}

\begin{remark} 
   Theorem~\ref{theorem: condition for not adjacent} presents a sufficient condition for determining nonadjacency between two observed variables.  In the absence of latent variables, this condition transforms into a necessary and sufficient one. Note that $\set{A}$ and $\set{B}$ may have overlapping variables. 
   SGS and PC \citep{spirtes2000causation} also introduced a necessary and sufficient condition for determining two variables not being adjacent when
    latent variables do not exist:  there exist a set of observed variables $\set{C}\subseteq \set{X}_\graph$, $\node{X_1},\node{X_2} \notin \set{C}$, such that $\node{X_1} \indep \node{X_2} | \set{C}$.
   Interestingly, this condition can be expressed in the form of    Theorem~\ref{theorem: condition for not adjacent}, with $\set{C} = \set{A} = \set{B}$; 
   thus Theorem~\ref{theorem: condition for not adjacent} 
   generalizes PC's condition
   to scenarios where 
   latent variables may be present (this claim is detailed in Appx.~\ref{proof: condition for not adjacent}).
   %
\end{remark}

%

\begin{theorem} [Condition for Existence of Latent Variable]
   Suppose a latent linear causal model with graph $\graph$ and observed variables $\set{X}_\graph$.
   If there exist three disjoint sets of variables $\set{A}$,$\set{B},\set{C}
   \subseteq \set{X}_\graph$,
   such that
    (i)  $|\set{B}|\geq|\set{A}|\geq 2$,
     (ii) $\forall~\text{distinct}~ 
   \node{A_1},\node{A_2} \in\set{A}$,
   $\node{A_1},\node{A_2}$  are adjacent in the CI skeleton over $\set{X}_\graph$ 
 with CI skeleton defined in Appx.~\ref{appendix: ci skeleton}),
   (iii) $\forall \node{A}\in\set{A},\node{B}\in\set{B}$, 
   $\node{A},\node{B}$    are adjacent in the CI skeleton over $\set{X}_\graph$,
   (iv) $\set{C}\subseteq\{\node{X}: \exists \node{Y}\in\set{A}\cup\set{B}~\text{s.t.}~\node{X}, \node{Y}~\text{are adjacent in the CI skeleton}\}$ (i.e., all elements in $\set{C}$ are neighbours of an element in $\set{A}\cup\set{B}$ in the CI skeleton),
   (v) $\text{rank}(\Sigma_{\set{A}\cup\set{C},\set{B}\cup\set{C}}) < |\set{A}|+|\set{C}|$,
   then there must exist at least one latent variable in one of the treks between $\set{A},\set{B}$.
   \label{theorem: global for latent}
   \end{theorem}


\begin{remark}
Theorem~\ref{theorem: global for latent} provides a sufficient condition for 
determining the existence of latent variables. 
 The underlying intuition is that, in the absence of latent variables,
  rank information should align 
 with what CI skeleton
  provides;
 if not, then there must exist at least one 
 latent variable.
 Furthermore, we will show in Theorem~\ref{theorem:necessary and sufficient under cond} that, with further graphical constraints 
 the condition in Theorem~\ref{theorem: global for latent} becomes both necessary and sufficient.
\end{remark}

Moreover, it can be shown that observed children, or even descendants of latent variables can be used as surrogates to calculate  rank
as stated in Theorem \ref{theorem:measurement as surrogate} with the definition of \textit{pure children} below.

\begin{definition} [Pure Children]
  $\set{Y}$ are pure children of variables $\set{X}$ in graph $\graph$, iff $\parents(\set{Y}) = \cup_{\node{Y_i} \in \set{Y}} \parents(\node{Y_i}) = \set{X}$  and $\set{X} \cap \set{Y}=\emptyset$. We denote the pure children of $\set{X}$ in $\graph$ by $\purechildren(\set{X})$. 
\label{definition:pch}
\end{definition}

\begin{theorem} [Pure Children as Surrogate for Rank Estimation]
Let $\set{C} \subseteq \purechildren(\set{A})$ be a subset of pure children of $\set{A}$, and $\set{B}$ be a set of variables such that for all $\mathsf{B} \in \set{B}$, $\node{B} \notin \descendant(\set{C})$. We have $\text{rank}(\Sigma_{\set{A},{\set{B}}}) \geq \text{rank}(\Sigma_{\set{C},{\set{B}}})$. Moreover, if  $\text{rank}(\Sigma_{\set{A},{\set{C}}})=|\set{A}|$, then $\text{rank}(\Sigma_{\set{A},{\set{B}}})=\text{rank}(\Sigma_{\set{C},{\set{B}}})$.
\label{theorem:measurement as surrogate}
\end{theorem}

\begin{remark}
Theorem \ref{theorem:measurement as surrogate}  
informs us that under certain conditions, we can estimate the rank of 
covariance 
involving latent variables
by using their pure children as surrogates. 
Even when the pure children are not observed
one can recursively examine the children's children until reaching observed descendants (defined in Appx.~\ref{appendix: covariance}). 
This enables us to deduce graphical information associated with latent variables
 through the use of observed ones as surrogates. 
\end{remark}

\section{Discovering Latent Structure through Rank Constraints}
\label{sec:allalgorithm}
In this section, we begin with the concept of \textit{atomic cover} and explore
 its rank deficiency, 
 and then develop an efficient algorithm based on rank-deficiency for latent causal discovery, as in Alg.~\ref{alg:all}.

\subsection{Atomic Cover and Rank Deficiency}
\label{sec:atomic cover}

Below, we introduce \textit{atomic covers} and their associated rank deficiency properties,
which allows us to define the minimal identifiable substructure of a graph.
We start with an example in Figure \ref{fig:intuition of atomic cover} that motivates the conditions for a latent variable to be identifiable.
%
%

\begin{example}
\label{example: exist latent}
Consider $\graph_1$  in Fig.~\ref{fig:intuition of atomic cover} (a).
%
It can be shown that the latent variable $\node{L_1}$ in $\graph_1$
is not identifiable.
We can easily find a graph $\graph_1'$ with no latent variable, e.g., as in Fig.~\ref{fig:more intuition of atomic cover} (b), such that 
$\graph_1'$ shares the same 
 skeleton as Fig.~\ref{fig:intuition of atomic cover} (a),
 but all the observational rank information entailed by $\graph_1$
is the same as by $\graph_1'$ - they are 
indistinguishable.
However, if $\node{X_4}$ becomes the children of $\node{L_1}$,
as in $\graph_2$ given in  Fig.~\ref{fig:intuition of atomic cover} (c),
the whole structure becomes identifiable.  
Specifically, the conditions in Theorem~\ref{theorem: global for latent} holds when we take
 $\set{A}=\{\node{X_1},\node{X_2}\}$,
$\set{B}=\{\node{X_3},\node{X_4}\}$, and $\set{C}=\emptyset$, which informs the existence of latent variables.
The same conditions also hold for Fig.~\ref{fig:intuition of atomic cover} (e)
and thus $\node{L_1}$ in $\graph_3$ is also identifiable,
though  $\node{L_1}$ has only two children together with another two neighbors.
\end{example}

The intuition is that, for a latent variable to be identifiable, 
it should have enough children and enough neighbors.
We next formalize this intuition into the concept of \textit{atomic cover} 
as the minimal identifiable unit in a graph. 
%
%
The formal definition of an atomic cover is given in Definition~\ref{definition:ac}, where \textit{effective cardinality} of a set of covers $\setset{V}$ is defined
as $||\setset{V}||=|(\cup_{\set{V} \in \setset{V}} \set{V})|$. 

    %

\begin{definition} [Atomic Cover]
  Let $\set{V}$ be a set of variables in $\graph$ with $|\set{V}|=k$, where $t$ of the $k$ variables are observed, and the remaining $k-t$ are latent. 
  $\set{V}$ is an atomic cover if $\set{V}$ contains a single observed variable (i.e.,  $k=t=1$ ), or 
    if the following conditions hold:
  \begin{itemize}[leftmargin=20pt,itemsep=-3pt,topsep=-3pt]
      \item [(i)] There exists 
      a set of atomic covers $\setset{C}$, with $||\setset{C}||\geq k+1-t$, such that
      $\cup_{\set{C} \in \setset{C}} \set{C}\subseteq \purechildren(\set{V})$
      and 
      $\forall \set{C_1}, \set{C_2} \in \setset{C}, \set{C_1}\cap\set{C_2}=\emptyset$.

     \item [(ii)] There exists a  set  of covers $\setset{N}$, with $||\setset{N}||\geq k+1-t$,
     such that every element in $\cup_{\set{N} \in \setset{N}} \set{N}$ is a neighbour of $\set{V}$  and 
     $ (\cup_{\set{N} \in \setset{N}} \set{N}) \cap (\cup_{\set{C} \in \setset{C}} \set{C})=\emptyset$.

\item [(iii)] There does not exist a partition of $\set{V}= \set{V_1} \cup \set{V_2}$ such that both $\set{V_1}$ and  $\set{V_2}$ 
are atomic covers.
  \end{itemize}
\label{definition:ac}
\end{definition}
 
In the definition above, each observed variable is treated as an atomic cover, e.g., $\{\node{X_1}\}$ in Figure \ref{fig:example1}. 
We define the minimal identifiable unit as an atomic cover with the rationale that, 
when two or more latent variables share exactly the same set of neighbors (e.g., $\node{L_1}$ and $\node{L_2}$ in Fig.~\ref{fig: example for treks}),
they can never be distinguished from observational information. Hence, it is more convenient and unified to consider them together in an atomic cover. Examples of atomic covers can be found in Appx.~\ref{appendix: example atomic cover}.   

%
%

Based on the definition of 
atomic covers,
 we define a \textit{cluster} 
 as the set of pure children of an atomic cover,
 and refer \textit{k-cluster} to
 a cluster whose parents' cardinality is k. 
We further define an operator 
$\sep({\set{X}})=\cup_{\node{X} \in \set{X}} \{\{\node{X}\}\}$.  
We next show that every atomic cover possesses a  useful rank deficiency property, 
which is formally stated in the following theorem 
(proof is given in Appx.~\ref{proof: rank property of atomic cover}).

\begin{theorem} [Rank Deficiency of an Atomic Cover]
  Let $\set{V}=\set{X} \cup \set{L}$  be an atomic cover 
where $|\set{X}|=t$  variables are observed
  and  $|\set{L}|=k-t$ are latent.
  Let $\setset{X}=\sep(\set{X})$, $\setset{X}_\graph=\sep(\set{X}_\graph)$,
  and a set of atomic covers $\setset{C}$, satisfying (i) in Definition~\ref{definition:ac}.
  Then 
  $\text{rank}(\Sigma_{\setset{C}\cup \setset{X}, \setset{X}\cup\setset{X}_\graph \backslash \setset{C}\backslash\meassureddes(\setset{C})})=k$ and $k< \min(||\setset{C}\cup \setset{X}||,||\setset{X}\cup\setset{X}_\graph \backslash \setset{C}\backslash\meassureddes(\setset{C})||)$ 
  ($\meassureddes$ denotes measured descendants).
\label{theorem: rank property of atomic cover}
\end{theorem}

\begin{example} [Example for atomic cover and rank deficiency]
  Consider an atomic cover in Fig.~\ref{fig:example2} (d): $\set{V}=\{\node{X_2},\node{L_2}\}$. $\set{V}$ is an atomic cover, because $\set{V}$ has at least 2 pure children and has additional 3 neighbors, satisfying the conditions in Definition \ref{definition:ac}. If we take $\setset{C}=\{\{\node{X_4}\},\{\node{X_5}\}\}$, and $\setset{X}=\{\{\node{X_2}\}\}$,
  we have $\text{rank}(\Sigma_{\setset{C}\cup \setset{X}, 
  \setset{X}\cup\setset{X}_\graph \backslash \setset{C}\backslash\meassureddes(\setset{C})}
  )=
  \text{rank}(\Sigma_{\{\node{X_4},\node{X_5}, \node{X_2}\}, \{\node{X_1},\node{X_2},\node{X_3},\node{X_6},\node{X_7},\node{X_{8}}\}})=2=|\set{V}|$. Noted that both $||\setset{C}\cup \setset{X}||, ||\setset{X}\cup\setset{X}_\graph \backslash \setset{C}\backslash\meassureddes(\setset{C})||>2$, so here the rank is deficient.
\end{example}


Theorem~\ref{theorem: rank property of atomic cover} establishes the rank-deficiency property of an atomic cover. 
Furthermore, if we can build a unique connection between rank deficiency and atomic covers under certain conditions, then we can exploit rank deficiency to identify atomic covers in a graph. 
%
%
In the following, we present the 
graphical conditions to achieve identifiability and  Theorem~\ref{theorem:unique_rank_of_atomic_cover} 
delineates under what conditions 
the uniqueness of the rank-deficiency property can be ensured.

  \begin{condition}[Basic Graphical Conditions for Identifiability] A graph $\graph$ satisfies the basic graphical condition for identifiability, if 
     $\forall \node{L}\in\set{L}_\graph$,  $\node{L}$ belongs to at least one atomic cover in $\graph$ and no latent variable is involved in any triangle structure  (whose definition is in Appx.~\ref{def of triangle}).
    \label{cond:basic}
  \end{condition}

\begin{theorem} [Uniqueness of Rank Deficiency]
  Suppose a graph $\mathcal{G}$ satisfies Condition~\ref{cond:basic}.
  We further assume  
  (i) all the atomic covers with cardinality $k'<k$ have been discovered and recorded, and (ii) there is no collider
  in $\graph$. 
  %
  If there exists a
  set of observed variables $\set{X}$ and a set of atomic covers $\setset{C}$ satisfying
$\setset{X}=\sep(\set{X})$, $\setset{C}\cap\setset{X}=\emptyset$,  and 
   $||\setset{C}||+||\setset{X}||=k+1$,
   such that (i) For all recorded $k'$ cluster $\setset{C'}$, 
  $||\setset{C} \cap \setset{C'}||\leq |\parents(\setset{C'})|$,
   (ii) $\text{rank}(\Sigma_{\setset{C}\cup \setset{X}, \setset{X}\cup\setset{X}_\graph \backslash \setset{C}\backslash\meassureddes(\setset{C})})=k$,
  then there exists an atomic cover
  $\set{V}=\set{L}\cup\set{X}$ in $\graph$,
  with $\set{X}=\cup_{\set{X'} \in \setset{X}} \set{X'}$, $|\set{L}|=k-|\set{X}|$,
  and $\cup_{\set{C} \in \setset{C}} \set{C}\subseteq \purechildren(\set{V})$.
  \label{theorem:unique_rank_of_atomic_cover}
  \end{theorem}

  For a better understanding, we provide an illustrative example in Appx.~\ref{example: theorem unique rank}. 
Basically, Theorem~\ref{theorem:unique_rank_of_atomic_cover} 
says that under certain conditions, we can build a unique connection between rank deficiency and atomic covers, and thus we can identify atomic covers in a graph
by searching for combinations of $\setset{C}$ and $\setset{X}$ that induce rank deficient property. 
We further introduce Theorem~\ref{theorem:necessary and sufficient under cond}, 
which is useful in that it provides necessary and sufficient conditions for the existence of latent variables,
under Condition~\ref{cond:basic}.

\begin{theorem} [Necessary and Sufficient Condition for Existence of Latent Variable]
\label{theorem:necessary and sufficient under cond}
If a graph satisfies Condition~\ref{cond:basic}, then the sufficient condition for the existence of latent variables in
Theorem~\ref{theorem: global for latent} becomes both necessary and sufficient.
That is, the "if" in Theorem~\ref{theorem: global for latent} becomes "if and only if".
\end{theorem}

With the theoretical guarantees
of Theorem~\ref{theorem:unique_rank_of_atomic_cover} and Theorem~\ref{theorem:necessary and sufficient under cond},
the next question is how to design a search procedure that strives to fulfill these conditions, in order to cash out the theorems for existence of latent variable and the uniqueness of rank-deficiency, for
 identifying atomic covers in a graph, and 
 consequently the whole latent causal structure.


\begin{figure}[t]
   \begin{subfigure}[b]{0.42\textwidth}
     \centering
     \includegraphics[width=\textwidth]{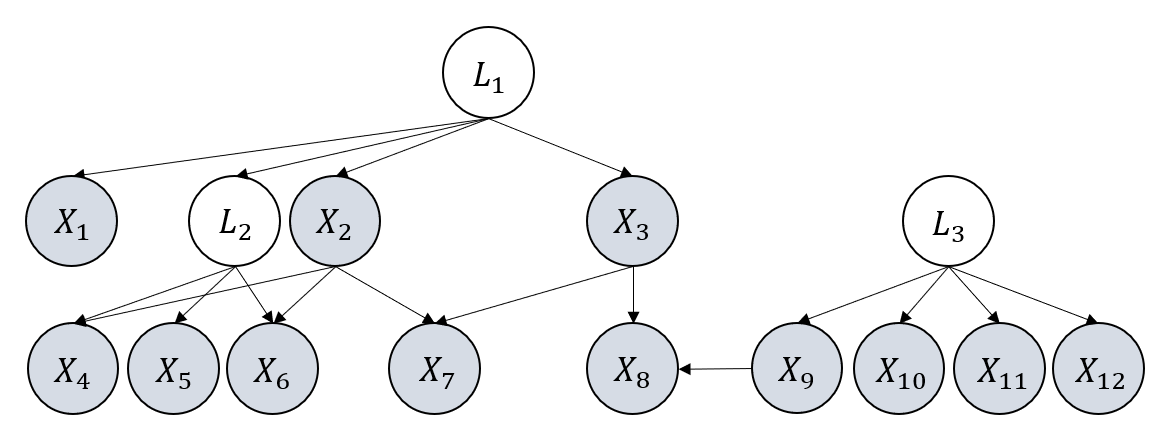}
     \caption{\scriptsize{The underlying graph $\graph$, all the observed variables of which
      are taken as input to Phase 1.\color{white}{place holder}}}
 \end{subfigure}
 \hfill
 \begin{subfigure}[b]{0.42\textwidth}
   \centering
   \includegraphics[width=\textwidth]{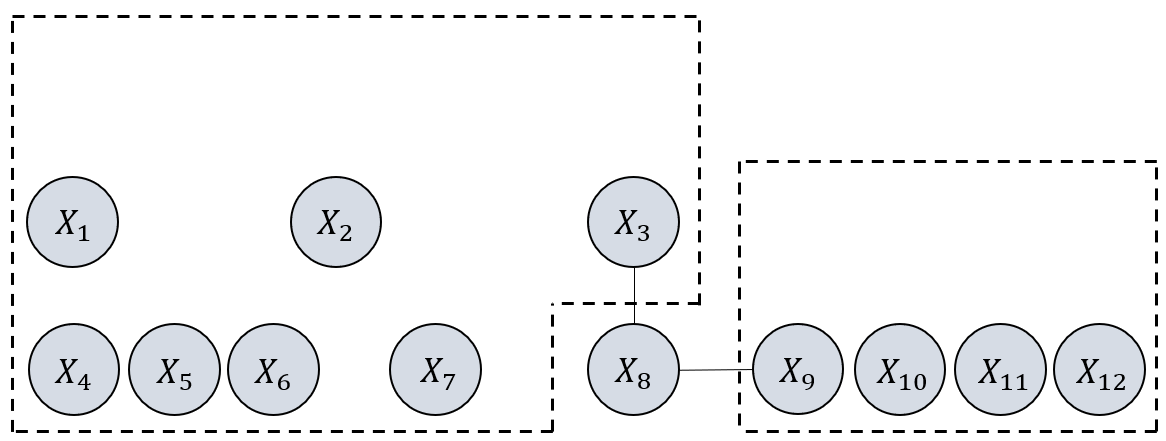}
   \caption{\scriptsize{Given the CI skeleton from Phase 1,
   variables are partitioned into two groups, shown in the two dashed areas.}}
 \end{subfigure}
  \hfill
 \begin{subfigure}[b]{0.42\textwidth}
   \centering
   \includegraphics[width=\textwidth]{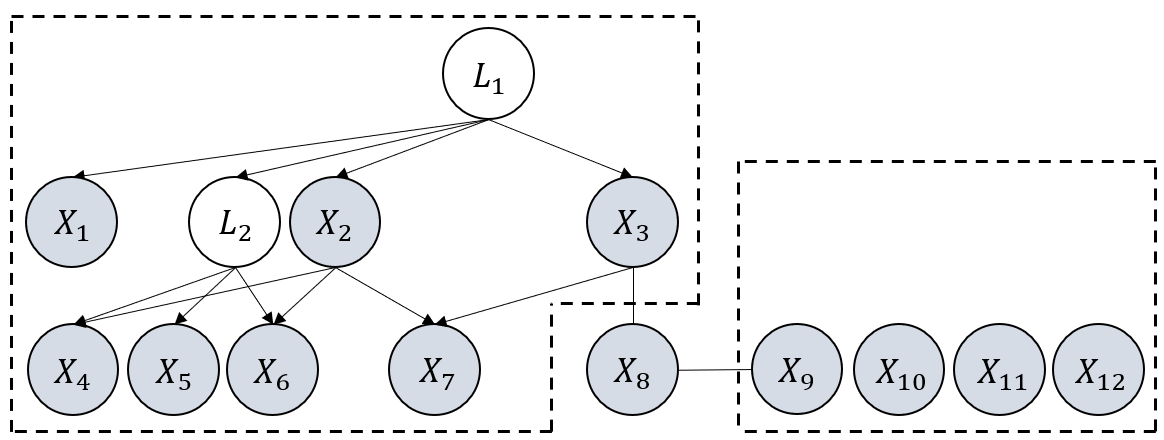}
    \caption{\scriptsize{Take variables from the first dashed area to Phase 2 and 3,  use the result to update the CI skeleton.}}
 \end{subfigure}
  \hfill
 \begin{subfigure}[b]{0.42\textwidth}
   \centering
   \includegraphics[width=\textwidth]{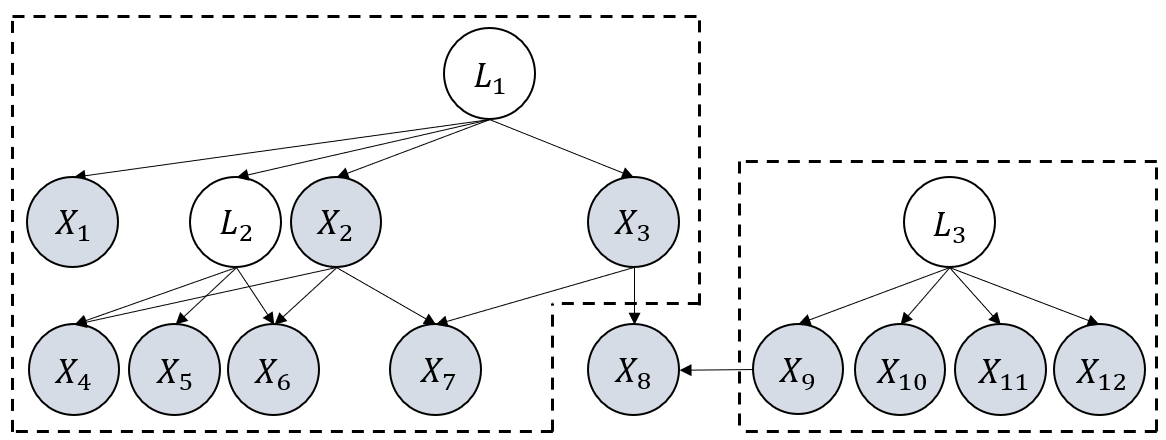}
    \caption{\scriptsize{Do the same thing as in (c) for the second dashed area and update all the remaining directions.}}
 \end{subfigure}
 \caption{\small An illustrative example of the overall search procedure in Alg.~\ref{alg:all}.}
   \label{fig:phase 1 example}
 \end{figure}

\begin{figure}[t]
   \begin{subfigure}[b]{0.22\textwidth}
     \includegraphics[width=\textwidth]{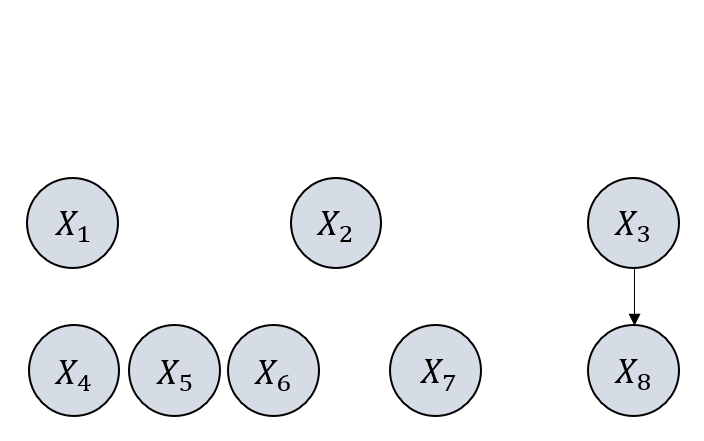}
     \caption{\scriptsize{Take $\setset{X}=\{\{\node{X_3}\}\}$ and \\$\setset{C}=\{\{\node{X_8}\}\}$.}}
 \end{subfigure}
 \hfill
 \begin{subfigure}[b]{0.22\textwidth}
   \includegraphics[width=\textwidth]{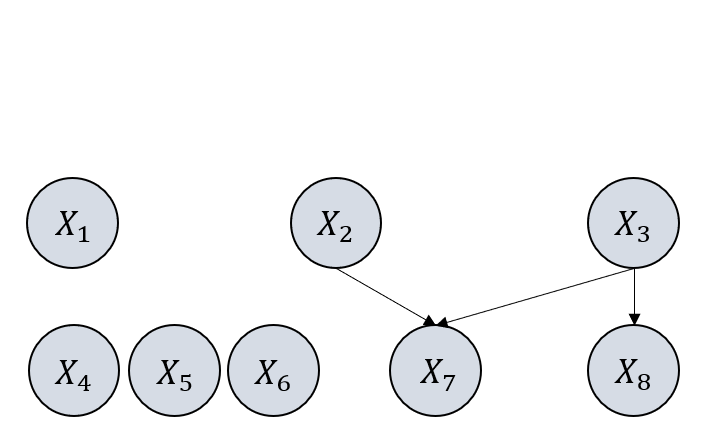}
   \caption{\scriptsize{Take $\setset{X}=\{\{\node{X_2},\node{X_3}\}\}$ and $\setset{C}=\{\{\node{X_7}\},\{\node{X_8}\}\}$.}}
 \end{subfigure}
 \hfill
 \begin{subfigure}[b]{0.22\textwidth}
   \includegraphics[width=\textwidth]{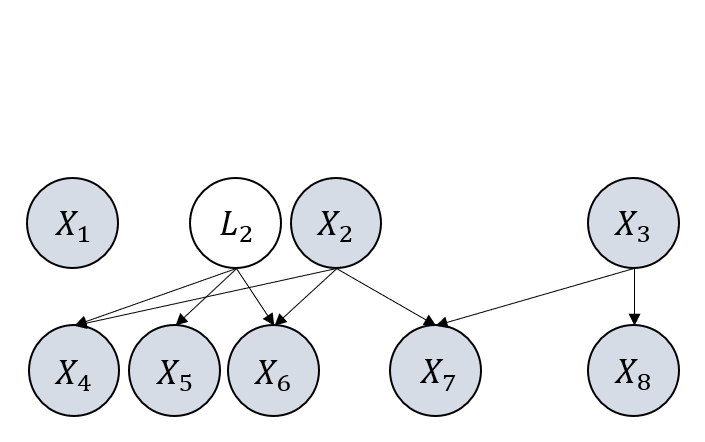}
   \caption{\scriptsize{Take $\setset{X}=\{\{\node{X_2}\}\}$ and \\$\setset{C}=\{\{\node{X_4}\},\{\node{X_5}\},\{\node{X_6}\}\}$.}}
 \end{subfigure}
 \hfill
 \begin{subfigure}[b]{0.22\textwidth}
   \includegraphics[width=\textwidth]{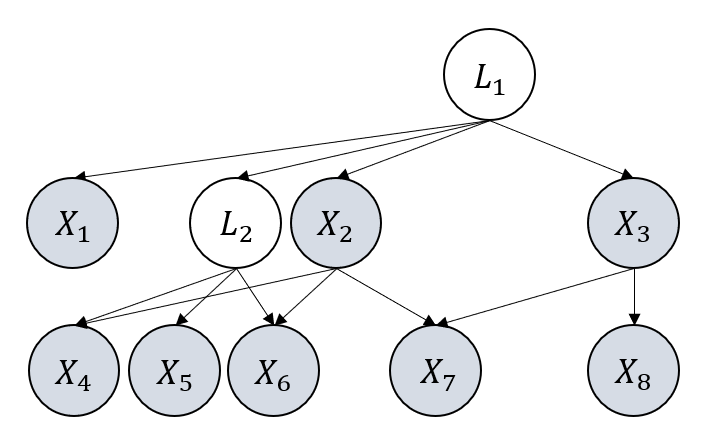}
   \caption{\scriptsize{Take $\setset{X}=\{\}$ and $\setset{C}=\{\{\node{X_1}\},\{\node{L_2},\node{X_2}\},\{\node{X_3}\}\}$}}
 \end{subfigure}
 \caption{\small An illustrative example of the process of Phase 2 in Alg.~\ref{alg:phase2}.}
   \label{fig:example2}
 \end{figure}

\subsection{Method - Rank-based Latent Causal Discovery}
In this section, we propose a computationally efficient and scalable search algorithm to 
identify the latent causal structure, referred to as  \textit{Rank-based Latent Causal Discovery (RLCD)}, that leverages the connection between graph structures and rank deficiency, as we discussed in previous sections. The search process mainly comprises three phases: (1) Phase 1: FindCISkeleton, (2) Phase 2: FindCausalClusters, and (3) Phase 3: RefineCausalClusters, as outlined in Alg.~\ref{alg:all}. 

%

Specifically, Phase 1 is to find the CI skeleton deduced by conditional independence tests over $\set{X}_\graph$,
and Phase 2 and Phase 3 are designed such that the conditions in Theorem~\ref{theorem:unique_rank_of_atomic_cover} are satisfied to the largest extent, in order to make use of the unique rank deficiency property to identify the latent structure.
%
 We initiate the process by finding the  CI skeleton first, for the reason as follows. According to Theorem~\ref{theorem:necessary and sufficient under cond}, latent variables exist iff conditions (i)-(v) in Theorem~\ref{theorem: global for latent} are satisfied, while
conditions (i)-(iv) can be directed inferred from the CI skeleton.
 Therefore, there is no need to consider all variables in $\set{X}_{\graph}$ as inputs for Phases 2 and 3. Instead, we make use of the CI skeleton over $\set{X}_\graph$ and select some groups of observed variables as inputs into Phases 2 and 3,
 where variables in each group together have the potential to satisfy conditions (i)-(iv).
 This also benefits the computational efficiency as 
 the Phase 2 and 3 for different groups can be done in parallel.
 
 An example of the entire search process is illustrated in Fig.~\ref{fig:phase 1 example}.
 After Phase 1, variables are partitioned into groups as shown in Fig.~\ref{fig:phase 1 example} (b). For each group (dashed area in Fig.~\ref{fig:phase 1 example} (b)),
 we conduct Phases 2 and 3, and have the final result shown in Fig.~\ref{fig:phase 1 example} (d). Details of each phase are given below.

\subsection{Phase 1: Finding CI Skeleton}
The objective of  Phase 1 is to find the CI skeleton over observed variables $\set{X}_\graph$ by utilizing conditional independence relations.
To this end, we employ 
the first stage of the PC algorithm \citep{spirtes2000causation},
with the difference that we replace all CI tests with rank tests, 
according to the following Lemma~\ref{lemma: rank and d-sep} (proof in Appx.~\ref{proof: rank and d-sep}).

\begin{lemma} [D-separation by Rank Test]
\label{lemma: rank and d-sep}
   Suppose a linear latent causal model with graph $\graph$. For disjoint $\mathbf{A}, \mathbf{B} ,\mathbf{C}\in\set{X}_\graph$,
   $\mathbf{C}$ d-separates
    $\mathbf{A}$ and $\mathbf{B}$ 
    in graph $\mathcal{G}$,
     if and only if $\text{rank}(\Sigma_{\set{A}\cup\set{C}, \set{B}\cup\set{C}})=|\set{C}|$.
\end{lemma}


\begin{figure}[t]
\begin{algorithm}[H]
\small
  \caption{The overall procedure for Rank-based Latent Causal Discovery (RLCD).}
  \label{alg:all}
  \SetAlgoLined
  \SetKwInOut{Input}{Input}
  \SetKwInOut{Output}{Output}
  \Input{Samples from all $n$ observed variables $\mathbf{X}_{\mathcal{G}}$}
  \Output{Markov equivalence class $\graphp$}
  \SetKwProg{Def}{def}{:}{}
  \Def{LatentVariableCausalDiscovery($\set{X}_{\graph}$)}
  {
    {Phase 1: $\graphp$ = FindCISkeleton($\set{X}_{\graph}$)} (Algorithm~\ref{alg:phase1})\;
    \For{ Each $\setset{Q}$, a group of overlapping maximal cliques, in $\graphp$} 
    {
      Set an empty graph $\graphpp$,      $\set{X}_\setset{Q}=\cup_{\set{Q}\in\setset{Q}}\set{Q}$,
      $\set{N}_\setset{Q}=\{\node{N}: \exists \node{X}\in \set{X}_\setset{Q}~\text{s.t.}~ 
      \node{N},\node{X}~\text{are adjacent in}~\graphp \}$\;
      {Phase 2: $\graphpp$ = FindCausalClusters($\graphpp$, $\set{X}_\setset{Q}\cup\set{N}_\setset{Q}$) (Algorithm~\ref{alg:phase2})\;}
      {Phase 3: $\graphpp$ = RefineCausalClusters($\graphpp$, $\set{X}_\setset{Q}\cup\set{N}_\setset{Q}$) (Algorithm~\ref{alg:phase3})\;}
      Transfer the estimated DAG $\graphpp$ to the Markov equivalence class and update $\graphp$ by $\graphpp$\;
    }
    Orient remaining causal directions that can be inferred from v structures\;
    \Return{$\graphp$}
  }
\end{algorithm}
\end{figure}

\begin{algorithm}[tb]
    \small
     \caption{Phase1: FindCISkeleton (Stage 1 of PC \citep{spirtes2000causation})}
     \label{alg:phase1}
     \SetAlgoLined
     \SetKwInOut{Input}{Input}
     \SetKwInOut{Output}{Output}
     \Input{Samples from $n$ observed variables $\set{X}_{\graph}$}
     \Output{CI skeleton $\graphp$}
     \SetKwProg{Def}{def}{:}{}
     \Def{\text{Stage1PC}($\set{X}_{\graph}$)}
     {
        Initialize a complete undirected graph $\graphp$ on $\set{X}_{\graph}$\;
        \Repeat{no adjacent $\node{X},\node{Y}$ s.t., $|\text{Adj}_{\graphp}(\node{X})\backslash\{\node{Y}\}|<n$}
        {
          \Repeat{all $\node{X},\node{Y}$ s.t., 
          $|\text{Adj}_{\graphp}(\node{X})\backslash\{\node{Y}\}|\geq n$ and all 
          $\set{S}\subseteq \text{Adj}_{\graphp}(\node{X})\backslash\{\node{Y}\}$, $|\set{S}|=n$, tested.}
          {
          Select an ordered pair $\node{X},\node{Y}$ that are adjacent in $\graphp$, s.t., 
          $|\text{Adj}_{\graphp}(\node{X})\backslash\{\node{Y}\}|\geq n$\;
          Select a subset $\set{S}\subseteq \text{Adj}_{\graphp}(\node{X})\backslash\{\node{Y}\}$
          s.t., $|\set{S}|=n$\;
          If $\text{rank}(\Sigma_{\{\node{X}\}\cup\set{C}, \{\node{Y}\}\cup\set{C}})=|\set{C}|$,
          delete the edge between $\node{X}$ and $\node{Y}$ from $\graphp$ 
          and record $\set{S}$ in
          $\text{Sepset}(\node{X},\node{Y})$ and $\text{Sepset}(\node{Y},\node{X})$.\;
          }
          n:=n+1\;
        }
       \Return $\graphp$
     }
   \end{algorithm}

We summarize the procedure of Phase 1 in Alg.~\ref{alg:phase1} in the appendix.
Although, asymptotically,
using CI and rank information will provide the same d-separation result over observed variables,
we use rank instead of CI in Phase 1  just for the purpose of having a unified causal discovery framework with rank constraints (as Phases 2 and 3 are also based on rank).

Given the CI skeleton $\graphp$ (result from Phase 1), the next step is to find the substructures in $\graphp$
that might contain latent variables.
Specifically, Theorem~\ref{theorem:necessary and sufficient under cond} informs us that latent variable exists iff
(i)-(v) in Theorem~\ref{theorem: global for latent} holds,
where (i)-(iv) can be directly inferred from the CI skeleton.
Specifically,
 we consider all the maximal cliques $\set{Q}$ in $\graphp$ (a clique is a set of variables that are fully connected and a maximal clique is a clique that cannot be extended),
   ~s.t., 
 $|\set{Q}|\geq3$.
We then partition these cliques into groups 
such that two cliques $\set{Q_1},\set{Q_2}$ are in the same group if
$|\set{Q_1}\cap\set{Q_2}|\geq 2$.
Finally, for each group of
cliques $\setset{Q}$ (as in line 3 in Alg.~\ref{alg:all}),
we combine them to form a set of variables
$\set{X}_{\setset{Q}}=
\cup_{\set{Q}\in\setset{Q}}\set{Q}$.
We further determine the neighbour set for each $\set{X}_\setset{Q}$, as 
$\set{N}_\setset{Q}=\{\node{N}: 
\exists \node{X}\in \set{X}_{\setset{Q}}~\text{s.t.}~ 
\node{N},\node{X}~\text{are adjacent in CI skeleton}~\graphp\}$.
It can be shown that variables that satisfy (i)-(iv) will be in the same set with $\set{X}_{\setset{Q}}\cup\set{N}_{\setset{Q}}$ (examples and proof in Appx.~\ref{appdneix: example phase1}).
Therefore 
our next step is to take each $\set{X}_{\setset{Q}}\cup\set{N}_{\setset{Q}}$  separately 
as input to our Phases 2 and 3, 
detailed as follows.

\begin{algorithm}[tb]
 \small
  \caption{Phase2: FindCausalClusters}
  \label{alg:phase2}
  \SetAlgoLined
  \SetKwInOut{Input}{Input}
  \SetKwInOut{Output}{Output}
  \Input{Samples from $n$ observed variables $\set{X}_{\graph}$}
  \Output{Graph $\graphp$}
  \SetKwProg{Def}{def}{:}{}
  \Def{\text{FindCausalClusters}($\graphp$, $\set{X}_{\graph}$)}
  {
    Active set $\setset{S} \gets \setset{X}_\graph=\{\{\node{X_1}\},...,\{\node{X_n}\}\}$, $k \gets 1$ \tcp*{$\setset{S}$ is a set of covers}\
    \Repeat{$k$ is sufficiently large}
    {
      $\graphp$, $\setset{S}$, $\text{found}$ = Search($\graphp$, $\setset{S}$, $\set{X}_{\graph}$, $k$) \tcp*{Only when nothing can be found }\
      If $\text{found}=1$ then $k\gets 1$ else $k\gets k+1$ \tcp*{udner current $k$ do we add $k$ by 1}
    }
    \Return $\graphp$\;
  }
  \SetKwProg{Def}{def}{:}{}
  \Def{\text{Search}($\graphp$, $\setset{S}$, $\set{X}_{\graph}$, $k$)}
  {
    Rank deficiency set $\setsetset{D}=\{\}$ \tcp*{To store rank deficient combinations}\
    \For{$\setset{T} \in \text{PowerSet}(\setset{S})$ (from $\setset{S}$ to $\emptyset$)} 
    {
      $\setset{S'} \gets (\setset{S} \backslash \setset{T}) \cup (\cup_{\set{T} \in \setset{T}} \purechildrenp(\set{T}))$ \tcp*{\footnotesize Unfold $\setset{S}$ to get $\setset{S'}$}
      \For{$t=k$ to $0$} 
      {
        \Repeat{all $\setset{X}$  exhausted}
        {
          Draw a set of $t$ observed covers $\setset{X} \subset \mathcal{S'}\cap \setset{X}_\graph$\;
          \Repeat{all $\setset{C}$ exhausted}
          {
            Draw a set of covers $\setset{C} \subset \setset{S'}\backslash \setset{X}$, s.t.,
            $||\setset{C}||=k-t+1$
            and get $\setset{N} \gets \setset{S'}\backslash (\setset{X} \cup \setset{C})$\;
            
            \lIf{$\texttt{rank}(\Sigma_{\setset{C}\cup\setset{X},\setset{N}\cup\setset{X}}) = k$ and  NoCollider($\setset{C}$, $\setset{X}$, $\setset{N})$}
            {
              Add $\setset{C}$ to $\setsetset{D}$
            }
          }
          \If{$\setsetset{D}\neq \emptyset$}
          {
            \For{$\setset{D}_i \in \setsetset{D}$}
            {
              \lIf {$|\parentsp(\setset{D}_i)\cup\mathbf{X}|=k$}
              {
                $\mathbf{P}\gets \parentsp(\setset{D}_i)\cup\mathbf{X}$
              }
              \lElse
              {
                Create new latent variables $\set{L}$, s.t., $\mathbf{P} \gets \mathbf{L}\cup\parentsp(\setset{D}_i)\cup\mathbf{X}$ 
                 and $|\set{L}|=k-|\parentsp(\setset{D}_i)\cup\mathbf{X}|$
              }
              Update $\graphp$ by taking elements of $\setset{D}_i$ as the pure children of $\set{P}$\; 
              \lIf {$\set{P}$ is atomic}
              {
                Update $\setset{S} \gets (\setset{S} \backslash \setset{D}_i) \cup \set{P}$
              }
            }
            \Return $\graphp$, $\setset{S}$, $\text{True}$ \tcp*{Return to search with $k=1$}
          }
        }
      }
    } 
    \Return $\graphp$, $\setset{S}$, $\text{False}$ \tcp*{Return to search with $k\gets k+1$}
  }
 \vspace{-1mm} 
\end{algorithm}
\setlength{\textfloatsep}{7pt}

\begin{algorithm}[t]    
\small
  \caption{Function: NoCollider}
  \SetAlgoLined
  \SetKwInOut{Input}{Input}
  \SetKwInOut{Output}{Output}
  \Input{$\setset{C}$, $\setset{X}$, $\setset{N}$}
  \Output{Whether  there exists $\set{O}\in\setset{C}$ s.t., $\set{O}$ is a collider of $\setset{C}\backslash\{\set{O}\}$ and $\setset{N}$}
  \SetKwProg{Def}{def}{:}{}
    \label{alg:checkcollider}
  \Def{\text{NoCollider}($\setset{C}$, $\setset{X}$, $\setset{N}$)}
  {
    \For{$c=1$ to $|\setset{C}|-1$} 
    {
      Draw $\setset{C'} \subset \setset{C}$ s.t., $|\setset{C'}|=c$\;
      \Repeat{all $\setset{C'}$ exhausted}
      {
        \lIf{$\texttt{rank}(\Sigma_{\setset{C'}\cup\setset{X},\setset{N}\cup\setset{X}}) <  ||\setset{C'}\cup\setset{X}||$}
        {
          \Return False
        }
      } 
    }
  \Return True
  }
  \vspace{-1mm}
\end{algorithm}
\setlength{\textfloatsep}{7pt}

\subsection{Phase 2: Finding Causal Clusters}
In this section, we introduce the second phase of our algorithm, 
FindCausalClusters, summarized in Alg.~\ref{alg:phase2} in the appendix and 
illustrated in Fig.~\ref{fig:example2}.
The objective here is to design an effective search procedure to find combinations of 
sets of covers $\setset{C}$ and $\setset{X}$ (as defined in
Theorem~\ref{theorem:unique_rank_of_atomic_cover}), such that rank deficiency holds and 
all the conditions required in 
Theorem~\ref{theorem:unique_rank_of_atomic_cover} are satisfied.
To be specific,  
given Condition~\ref{cond:basic},
Theorem~\ref{theorem:unique_rank_of_atomic_cover} further requires that (i) when we are searching for $k$-clusters, all the $k'$-clusters, $k'<k$ have been found and recorded, and that (ii) there is no collider in $\graph$.
We next introduce the key designs for that end, accompanied by examples.

As for the requirement (i), we design our search procedure such that it starts with $k=1$. If a k-cluster is found, we update the graph and reset $k$ to $1$; otherwise, we increase $k$ by $1$ (as in line 5 Alg~\ref{alg:phase2}).
This ensures to a large extent that when searching for a $k$-cluster, all $k'$ clusters such that $k'<k$, can be found, in order to fulfill the requirement (i).

Regarding the requirement (ii), we need to ensure that all the rank deficiencies are from atomic covers, rather than from colliders. One could directly assume the absence of colliders in the underlying $\graph$,
but this would impose rather strong structural constraints and thus limit the applicability of a discovery method.
Therefore, we add a collider check function \textit{NoCollider} defined in Alg.~\ref{alg:checkcollider}, together
with the designed search procedure to allow incorporating colliders timely such that they will not induce unexpected rank deficiency anymore.
With these two designs,
we only rely on a much weaker condition about colliders, as in Condition~\ref{cond:vstructure},
under
which
it can be guaranteed that our search procedure will not be affected by the existence of colliders (proof in Appx.~\ref{proof:identifiability}).

\begin{condition}[Graphical condition  on  colliders for identifiability] 
In a latent graph $\graph$, if (i) there exists a set of variables $\set{C}$
such that every variable in $\set{C}$ is a collider of two atomic covers $\set{V_1}$, $\set{V_2}$, and denote by $\set{A}$ the minimal set of variables that d-separates $\set{V_1}$ from $\set{V_2}$, (ii) there is a latent variable in $\set{V_1}, \set{V_2}, \set{C}$ or $\set{A}$, then we must have 
$|\set{C}| + |\set{A}| \geq |\set{V_1}|+|\set{V_2}|$.
\label{cond:vstructure}
\end{condition}

We summarize the whole process of Phase 2 in Alg.~\ref{alg:phase2}
and provide an illustration in Fig.~\ref{fig:example2}, where the input variables are from 
 the left dash area in Fig.~\ref{fig:phase 1 example} (b).
 For a better understanding, please refer to Appx.~\ref{appendix: example phase2} for a detailed description of key steps and illustrative examples.


\begin{algorithm}[t]
\small
  \caption{Phase3: RefineCausalClusters}
  \label{alg:phase3}
  \SetKwInOut{Input}{Input}
  \SetKwInOut{Output}{Output}
  \Input{Graph $\graph'$}
  \Output{Refined graph $\graph'$}
  \SetKwProg{Def}{def}{:}{}
  \Def{RefineCausalCLusters($\graphp$, $\set{X}_\graph$)}
  {
    \Repeat{No more $\set{V}$ found and all $\set{V}$ exhausted}
    {
      Draw an atomic cover $\set{V}$ from $\graphp$\;
      Delete $\set{V}$, neighbours of $\set{V}$ that are latent, and all relating edges from $\graphp$ to get $\hat{\graph}$\;
      $\graphp =\text{FindCausalClusters}(\hat{\graph},\set{X}_\graph)$\; 
    }
    \Return{$\graphp$}
  }
\end{algorithm}

\subsection{Phase 3: Refining Causal Clusters}
In Phase 2, we strive to fulfill all required conditions such that we can correctly identify causal clusters and related structures.
%
However,
 there still exist some rare cases where our search cannot ensure the requirement (i) in Theorem~\ref{theorem:unique_rank_of_atomic_cover}. 
In this situation, Phase 2 might produce a big cluster in the resulting $\graphp$ that should be split into smaller ones (see examples in  
Appx.~\ref{appendix: example phase3}).
Fortunately, the incorrect cluster will not do harm to the identification of other substructures in the graph, and thus we can employ Phase 3 to characterize and refine the incorrect ones,
by making use of the following Theorem~\ref{theorem:phase3} (proof in Appx.~\ref{proof:phase3}).
%


\begin{theorem}[Refining Clusters]
  Denote by $\graphp$ the output from FindCausalClusters and by $\graph$ the true graph. For an atomic cover $\set{V}$ in $\graphp$, if $\set{V}$  is not a correct cluster in $\graph$ but consists of some smaller clusters, then $\set{V}$ can be refined into correct ones by $\text{FindCausalClusters}(\hat{\graph},\set{X})$, where $\hat{\graph}$ is got by deleting $\set{V}$, all neighbors of $\set{V}$ that are latent, and all relating edges of them, from $\graphp$.
  \label{theorem:phase3}
  \end{theorem}

 To be specific, we search through all the atomic covers $\set{V}$ in $\graphp$,
  the output of Phase 2, and then perform $\text{FindCausalClusters}(\hat{\graph}, \set{X})$, 
where $\hat{\graph}$ is defined as in Theorem~\ref{theorem:phase3}.
 %
With this procedure, we can make sure that all the found clusters in $\graphp$ are correct as in $\graph$. 
We summarized Phase 3 in Alg.~\ref{alg:phase3}
and provide illustrative examples in  
Appx.~\ref{appendix: example phase3}. 


 
\section{Identifiability Theory of Causal Structure}
\label{sec:theory}

Here we show the identifiability of the proposed RLCD algorithm. Specifically,
RLCD asymptotically produces the correct Markov equivalence class of the causal graph over both observed and latent variables under certain graphical conditions, up to the \textit{minimal-graph operator} $\mathcal{O}_{\text{min}}(\cdot)$ and \textit{skeleton operator} $\mathcal{O}_{\text{s}}(\cdot)$ (defined in Appx.~\ref{appendix: graph operator} following \citet{huang2022latent}). 
$\mathcal{O}_{\text{min}}(\cdot)$ is to absorb redundant latent variables under certain conditions and $\mathcal{O}_{\text{s}}(\cdot)$ is to introduce edges involving latent variables if certain conditions hold,
and we note that the observational rank information is invariant to these two graph operators (examples in Appx.~\ref{examples:operator}). We summarize the identifiability result in
Theorem~\ref{theorem:identifiability}, along with  Corollary~\ref{cor:pcandrank} (proof in Appx.~\ref{proof:identifiability}).

\begin{theorem}[Identifiability of the Proposed Alg.~\ref{alg:all}]
Suppose $\graph$ is a DAG associated with a Linear Latent Causal Model
that satisfies Condition~\ref{cond:basic}
and Condition~\ref{cond:vstructure}.
Algorithm \ref{alg:all} can asymptotically identify the Markov equivalence class of $\mathcal{O}_{\text{min}}(\mathcal{O}_s(\graph))$.
\label{theorem:identifiability}
\end{theorem}

\begin{corollary}
    \label{cor:pcandrank}
    Assume linear causal models. Asymptotically, when no latent variable exists, Alg. \ref{alg:all}'s output is the same as that of PC; as another special case, when there is no edge between observed variables, the output of Alg. \ref{alg:all} is the same as that of Hier. rank \citep{huang2022latent}.
\end{corollary}


\begin{center}
\begin{table}[tb]
  \caption{\small{F1 scores (mean (standard deviation)) of compared methods on different types of latent graphs.} }
   \vspace{-2mm}
   \label{tab:f1 v}
  \footnotesize
  \center 
\begin{center}
\begin{tabular}{|c|c|c|c|c|c|c|c|}
  \hline  \multicolumn{2}{|c|}{} &\multicolumn{6}{|c|}{\textbf{F1 score for skeleton among all variables $\set{V}_\graph$ (both $\set{X}_\graph$ and $\set{L}_\graph$)}}\\
  \hline 
  \multicolumn{2}{|c|}{Algorithm}  & Ours & Hier. rank  & PC & FCI & GIN &RCD\\
  \hline 
    & 2k 
    & \textbf{0.84}{\color{white}-}(0.11) & 0.58{\color{white}-}(0.01) & 0.36{\color{white}-}(0.01) & 0.37{\color{white}-}(0.01) &   0.37{\color{white}-}(0.03) & 0.24{\color{white}-}(0.04)\\
  \cline{2-8}
  {\emph{Latent+tree}}
    &5k 
    & \textbf{0.92}{\color{white}-}(0.05) & 0.60{\color{white}-}(0.01) & 0.37{\color{white}-}(0.00) & 0.37{\color{white}-}(0.01) &   0.41{\color{white}-}(0.03) & 0.33{\color{white}-}(0.00)\\
  \cline{2-8}
    &10k
    & \textbf{0.98}{\color{white}-}(0.02) &0.60{\color{white}-}(0.01)  &0.37{\color{white}-}(0.00)  & 0.38{\color{white}-}(0.02) &    0.41{\color{white}-}(0.03)& 0.33{\color{white}-}(0.01)\\
  \hline 
    & 2k 
    & \textbf{0.81}{\color{white}-}(0.12)  & 0.52{\color{white}-}(0.05) &0.44{\color{white}-}(0.01)  &  0.38{\color{white}-}(0.02) & 0.40{\color{white}-}(0.02)&0.26{\color{white}-}(0.03) \\
    \cline{2-8}
    {\emph{Latent+measm}}
    &5k 
    & \textbf{0.88}{\color{white}-}(0.11) & 0.52{\color{white}-}(0.05)  & 0.49{\color{white}-}(0.01) &0.40{\color{white}-}(0.01)  &  0.46{\color{white}-}(0.03)&0.29{\color{white}-}(0.01) \\
  \cline{2-8}
    &10k
    & \textbf{0.91}{\color{white}-}(0.09) & 0.53{\color{white}-}(0.05)  & 0.49{\color{white}-}(0.01) & 0.40{\color{white}-}(0.01) &  0.47{\color{white}-}(0.05)& 0.34{\color{white}-}(0.04) \\
  \hline 
    & 2k 
    & \textbf{0.66}{\color{white}-}(0.01) & 0.44{\color{white}-}(0.02) & 0.31{\color{white}-}(0.01) & 0.25{\color{white}-}(0.02) &   0.30{\color{white}-}(0.04)&0.32{\color{white}-}(0.03)  \\
  \cline{2-8}
  {\emph{Latent general}} &5k 
  & \textbf{0.72}{\color{white}-}(0.03) & 0.45{\color{white}-}(0.03) & 0.32{\color{white}-}(0.01) & 0.28{\color{white}-}(0.02) &   0.38{\color{white}-}(0.04)&0.34{\color{white}-}(0.02) \\
  \cline{2-8}&10k 
  & \textbf{0.80}{\color{white}-}(0.05)& 0.45{\color{white}-}(0.04) &0.32{\color{white}-}(0.01)  & 0.28{\color{white}-}(0.02) &  0.35{\color{white}-}(0.01)&0.36{\color{white}-}(0.01) \\
  \hline 
\end{tabular}
\end{center}
\end{table}
\end{center}

\begin{center}
\begin{table}[h]
\vspace{-2mm}
  \caption{\small F1 scores (mean (standard deviation)) of compared methods on different types of latent graphs.
  F1 score is calculated only for edges between observed variables $\set{X}_\graph$ in this graph. Hier.rank and GIN assume that observed variables are not directly adjacent so their performance is reported as -.}
   \label{tab:f1 x}
  \footnotesize
  \center 
  \begin{center}
  \begin{tabular}{|c|c|c|c|c|c|c|c|}
    \hline  \multicolumn{2}{|c|}{} &\multicolumn{6}{|c|}{\textbf{F1 score for skeleton among $\set{X}_\graph$}}\\
    \hline 
     \multicolumn{2}{|c|}{Algorithm} & Ours & Hier. rank  & PC & FCI & GIN & RCD \\
    \hline
    & 2k 
    & \textbf{0.79} (0.16) & - &  0.46{\color{white}-}(0.02)& 0.00{\color{white}-}(0.00) & - & 0.30{\color{white}-}(0.03)\\
    \cline{2-8}
    {\emph{Latent+tree}}
    &5k 
    & \textbf{0.86} (0.10) & - & 0.44{\color{white}-}(0.00) & 0.03{\color{white}-}(0.04) & - & 0.38{\color{white}-}(0.01)\\
    \cline{2-8}
    &10k
    &  \textbf{0.97} (0.04) &- & 0.44{\color{white}-}(0.00) & 0.18{\color{white}-}(0.07) &  -& 0.39{\color{white}-}(0.02)\\
    \hline 
    & 2k 
    & \textbf{0.84} (0.11) &- &0.50{\color{white}-}(0.02)  & 0.00{\color{white}-}(0.00) & -& 0.30{\color{white}-}(0.02)\\
    \cline{2-8}
    {\emph{Latent+measm}}
    &5k 
    & \textbf{0.93} (0.08) &- & 0.49{\color{white}-}(0.01) & 0.05{\color{white}-}(0.03) & - &0.32{\color{white}-}(0.02) \\
    \cline{2-8}
    &10k
    & \textbf{0.95} (0.05)  & - & 0.48{\color{white}-}(0.02) &0.03{\color{white}-}(0.05)  & - &0.42{\color{white}-}(0.09)\\
    \hline 
    & 2k 
    & \textbf{0.68} (0.02) &-  & 0.44{\color{white}-}(0.01) &0.27{\color{white}-}(0.09)  &  -& 0.39{\color{white}-}(0.06)\\
    \cline{2-8}
    {\emph{Latent general }} &5k 
    & \textbf{0.71} (0.03) &-  &  0.45{\color{white}-}(0.01)& 0.31{\color{white}-}(0.10) &-&0.44{\color{white}-}(0.05)\\
    \cline{2-8}&10k 
    & \textbf{0.78} (0.06)  & - & 0.45{\color{white}-}(0.01) &0.32{\color{white}-}(0.05)  & - & 0.44{\color{white}-}(0.01)\\
    \hline 
  \end{tabular}
  \end{center}
  \hfill
\end{table}
\end{center}

  \begin{figure}[t]
  \centering 
  \includegraphics[width=1\linewidth]{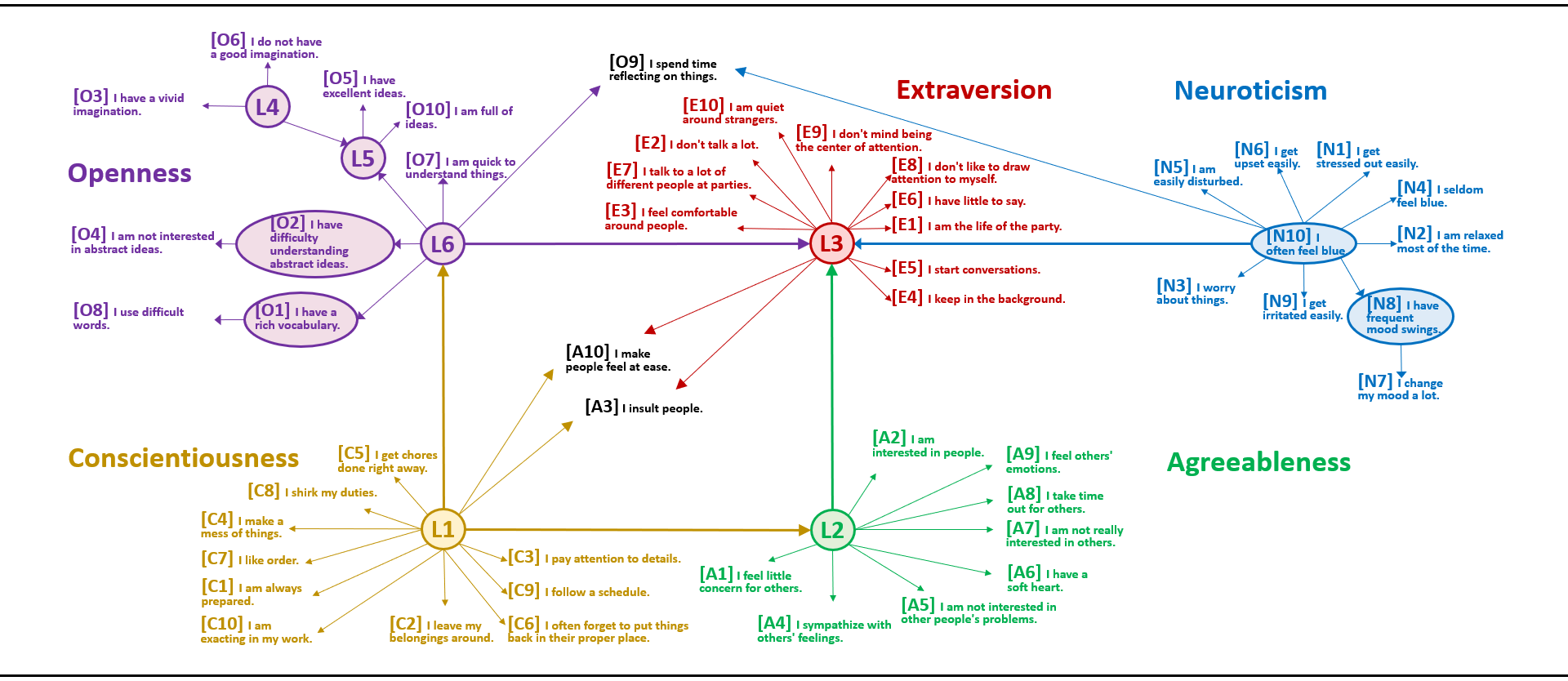}
  \caption{\small Identified latent graph on Big Five personality dataset.}
  \label{fig:big5}
\end{figure}

\section{Experiments}
\label{sec:exp}
We validate our method using both synthetic and real-life data.
In finite sample cases, 
we employ canonical correlations \citep{ranktest} 
to estimate the rank (detailed in Appx.~\ref{ranktest}).

\paragraph{Synthetic Data}
%
Specifically, we considered different types of latent graphs:
%
(i) latent tree models
 (Appx.~\ref{sec apendix: graphs for latent tree}), (ii) latent measurement models (Appx.~\ref{sec apendix: graphs for latent mm}),
 and 
 (iii) general latent models (Fig.~\ref{fig:example1} and Appx. \ref{sec apendix: graphs for latent general}).
 The causal strength is uniformly from $[-10, 10]$, 
 and the noise is either Gaussian or uniform (which is for RCD and GIN).
%
%
We propose to use the following two metrics for comparisons:
  \textbf{(1)} F1 score of skeleton among observed variables $\set{X}_\graph$,
  \textbf{(2)} F1 score of skeleton among all variables $\set{V}_\graph$.
  %
  We consider combinations and permutations of latent variables during evaluation (detailed in Appx.~\ref{appendix:comb and perm} together with specific definition of F1 score).
  %
  %

  We compared with many competitive baselines, including \textbf{(i)} Hier. rank \citep{huang2022latent},
  \textbf{(ii)} PC \citep{spirtes2000causation}, 
  \textbf{(iii)} FCI \citep{spirtes2013causal},
  \textbf{(iv)} RCD \citep{maeda2020rcd},
   and \textbf{(v)} GIN \citep{xie2020generalized}.
%
The results are reported in Tab.~\ref{tab:f1 v} and Tab.~\ref{tab:f1 x}, where we run experiments with different random seeds and sample sizes $2k$, $5k$, and $10k$.
 Our proposed RLCD gives the best results on all types of graphs,
 in terms of both metrics, with a clear margin. This result serves as strong empirical support for the identifiability of latent linear causal graphs by our proposed method.

 \paragraph{Real-World Data}
To further verify our proposed method, we employed a real-world Big Five Personality dataset \url{https://openpsychometrics.org/}. It consists of 50 personality indicators and close to 20,000 data points. Each Big Five personality dimension, namely, Openness, Conscientiousness, Extraversion, Agreeableness, and Neuroticism (O-C-E-A-N), are measured with their own 10 indicators. Data is processed to have zero mean and unit variance. We employ the proposed method to determine the Markov equivalence class and employ GIN~\citep{xie2020generalized} to further decide other directions between latent variables (more details in Appendix \ref{sec appendix: info of big5}).

We analyzed the data using RLCD, producing a causal graph
in Fig~\ref{fig:big5} that exhibits interesting psychological properties. 
First,
 most of the variables 
related to the same Big Five dimension are in the same cluster.
Strikingly, our result reconciles two currently deemed distinct theories of personality: latent personality dimensions and network theory \citep{cramer2012dimensions,wright2017factor}. We see groups of closely connected items that are predictable from latent dimensions (L1, L2, L3),  interactions among latents (L1$\rightarrow$L6$\rightarrow$L3, L1$\rightarrow$L2$\rightarrow$L3), and latents influencing the same indicators (L1, L3).
We also observe plausible causal links between indicators (e.g., O2$\rightarrow$O4 and O1$\rightarrow$O8). We argue that our findings are consistent with pertinent personality literature, but more importantly, offer new, plausible explanations as to the nature of human personality (detailed analysis in Appx.~\ref{appendix:bigfive analysis}).


\section{Discussion and Conclusion}
We developed a versatile causal discovery approach that allows latent variables to be causally-related in a flexible way, by making use of rank information.
We showed the proposed method can asymptotically identify the Markov equivalence class of the underlying graph under mild conditions.
One limitation of our method is that it cannot directly handle cyclic graphs, and as future work, we will extend 
this line of thought of using rank information to cyclic graphs. Another line of future research is to extend the idea to handle nonlinear causal relations.

\bibliography{iclr2024_conference}
\bibliographystyle{iclr2024_conference}

\newpage
\renewcommand\appendixpagename{\Large Appendix}
\begin{appendices}
Organization of Appendices:
\begin{itemize}
    \item Section A: Definitions, Examples, and Proofs
      \begin{itemize}
              \item Section~\ref{appendix: covariance}: Subcovariance Matrix and Definition of Descendants.
                \item 
          Section~\ref{example: trek}: Treks, T-separations, and Examples.
                    \item 
          Section~\ref{appendix: ci skeleton}: Definition of CI Skeleton
                    \item 
          Section~\ref{appendix: graph operator}: Definition of Rank-invariant Graph Operator.
                    \item 
          Section~\ref{appendix: maxflowmincut lemma}: Max-Flow-Min-Cut Lemma~\ref{lemma: maxflowmincut} for Treks
                    \item 
          Section~\ref{appendix: nonadjacency example}: Example for Theorem~\ref{theorem: condition for not adjacent}.
                    \item 
          Section~\ref{appendix: example surrogate}: Example for Theorem~\ref{theorem:measurement as surrogate}.
                    \item 
          Section~\ref{proof: condition for not adjacent}: Proof of Theorem~\ref{theorem: condition for not adjacent}.
          \item 
          Section~\ref{proof:theorem: global for latent}: Proof of Theorem~\ref{theorem: global for latent}.
          \item Section~\ref{proof:measurement as surrogate}: Proof of Theorem~\ref{theorem:measurement as surrogate}.
          \item Section~\ref{proof: rank property of atomic cover}: Proof of Theorem~\ref{theorem: rank property of atomic cover}.
          \item Section~\ref{example: theorem unique rank}: Example for Theorem~\ref{theorem:unique_rank_of_atomic_cover}.
          \item Section~\ref{proof:unique_rank_of_atomic_cover}: Proof of Theorem~\ref{theorem:unique_rank_of_atomic_cover}.
          \item Section~\ref{proof:necessary and sufficient under cond}: Proof of Theorem~\ref{theorem:necessary and sufficient under cond}.
          \item Section~\ref{proof: rank and d-sep}: Proof of Lemma~\ref{lemma: rank and d-sep}.
          \item Section~\ref{proof:phase3}:Proof of Theorem~\ref{theorem:phase3}.
          \item Section~\ref{proof:identifiability}: Proof  of Theorem~\ref{theorem:identifiability}.
          \item Section~\ref{ranktest}: Description of the rank test that we employ.
      \end{itemize}   
    \item Section B: Graphs, Illustrations of algorithms, and more information on datasets
    \begin{itemize}
        \item Section~\ref{compare graphs with each method}: Examples of Grpahs that Each Method Can Handle.
        \item Section~\ref{violation of faithfulness}: Example of Violation of faithfulness
        \item Section~\ref{appendix: example atomic cover}: Example of Atomic Cover.
        \item Section~\ref{appdneix: example phase1}: Example for Phase 1
        \item Section~\ref{appendix: example phase2}: Detailed Description and Example for Phase 2.
        \item Section~\ref{appendix: example phase3}: Example for Phase 3.
        \item Section~\ref{sec apendix: graphs for observed only}: Graph examples for model that has only observed variables.
        \item Section~\ref{sec apendix: graphs for latent tree}: Graph examples for latent tree model.
        \item Section~\ref{sec apendix: graphs for latent mm}: Graph examples for latent measurement model.
        \item Section~\ref{sec apendix: graphs for latent general}: Graph examples for general latent model.
        \item Section~\ref{example:colliderinc}: Example for considering colliders in Phase 1.
        \item Section~\ref{examples:operator}: Examples for graph operators.
        \item Section~\ref{appendix: relations between covers}: Examples for graphical relations between covers.
        \item Section~\ref{appendix:checkcollider}: Discussions on checking colliders completely.
        \item Section~\ref{appendix:comb and perm}: Evaluation metric details.
        \item Section~\ref{appendix:Computational cost}: More details of experiments on synthetic data.
        \item Section~\ref{sec appendix: info of big5}: Detailed information of the Big Five personality dataset.
        \item Section~\ref{appendix:bigfive analysis}: Detailed analysis of the result from the Big Five dataset.
    \end{itemize}
    \item Section C: Related work
    \begin{itemize}
        \item Section~\ref{sec appendix: related work}: Related work.
    \end{itemize}
    \item {Section D: Additional Information (added during rebuttal)}
    
\end{itemize}

\section{Definitions, Examples, and Proofs}

\subsection{Subcovariance Matrix and Definition of Descendants}
\label{appendix: covariance}
$\Sigma_{\set{A},\set{B}}$ refers to subcovariance over $\set{A}$ and $\set{B}$. E.g., $\Sigma_{\{\node{X_1},\node{X_2}\},\{\node{X_3}\}}$
is a $2\times1$ matrix whose  entry at $(1,1)$ is $\Cov(\node{X_1},\node{X_3})$ 
and  entry at $(2,1)$ is $\Cov(\node{X_2},\node{X_3})$.
$\Sigma_{\setset{A},\setset{B}}$ refers to subcovariance over $\setset{A}$ and $\setset{B}$.
Specifically,$\Sigma_{\setset{A},\setset{B}}=\Sigma_{\cup_{\set{A} \in \setset{A}} \set{A},\cup_{\set{B} \in \setset{B}} \set{B}}$.
E.g., $\Sigma_{\{\{\node{X_1}\},\{\node{X_2}\}\},\{\{\node{X_3}\}\}}$
is also a $2\times1$ matrix whose  entry at $(1,1)$ is $\Cov(\node{X_1},\node{X_3})$ 
and  entry at $(2,1)$ is $\Cov(\node{X_2},\node{X_3})$.

As for the definition of descendants,
a descendant is a node that can be reached by following one or more directed edges from a given node, and this definition excludes the node itself.

\subsection{Treks, T-separations, and Examples}
\label{example: trek}

\begin{figure}[h]
    \centering 
    \vspace{0mm} 
    \includegraphics[width=0.35\textwidth]{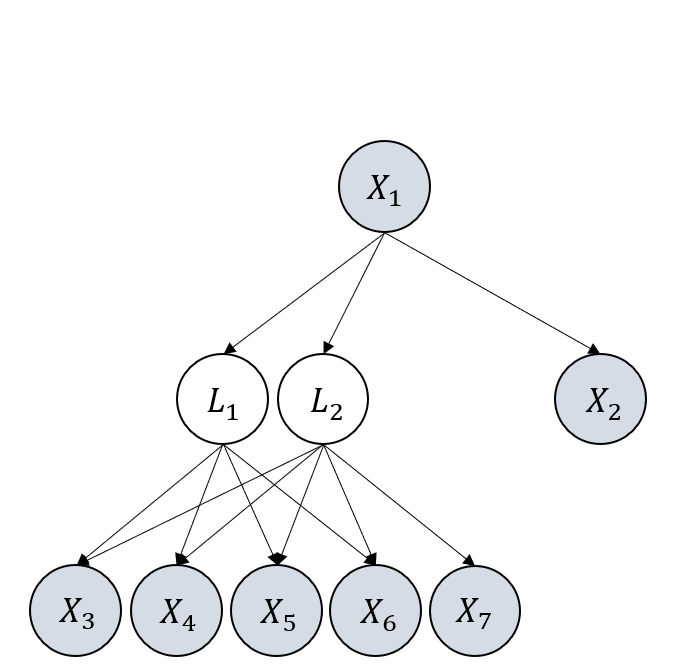}
    \caption{\small An example graph to show treks, t-separations, and rank of subcovariance matrix.} 
     \label{fig: example for treks}
  \end{figure}

\begin{example} [Example of Treks]
    \label{example:treks}
    In Figure~\ref{fig: example for treks},
    there are four treks between $\node{X_3}$ and $\node{X_4}$:
    (i) $(P_1,P_2)$=($\node{X_3}\leftarrow\node{L_1}$,~$\node{L_1}\rightarrow\node{X_4}$),
    (ii) $(P_1,P_2)$=($\node{X_3}\leftarrow\node{L_2}$,~$\node{L_2}\rightarrow\node{X_4}$),
    (iii) $(P_1,P_2)$=($\node{X_3}\leftarrow\node{L_1}\leftarrow\node{X_1}$,~
    $\node{X_1}\rightarrow\node{L_2}\rightarrow\node{X_4}$),
    (iv) $(P_1,P_2)$=($\node{X_3}\leftarrow\node{L_2}\leftarrow\node{X_1}$,~
    $\node{X_1}\rightarrow\node{L_1}\rightarrow\node{X_4}$).
    For adjacent variables such as $\node{X_1}$ and $\node{X_2}$,
    there must exist at least one trek $(P_1,P_2)$=($\node{X_1}$,~$\node{X_1}\rightarrow\node{X_2}$)
     between them.
    \end{example}

\begin{example} [Example of Trek-separations]
\label{example:t-sep}
In Figure~\ref{fig: example for treks}, 
$\node{X_3}$ and $\node{X_4}$ can be t-separated by $(\{\node{L_1},\node{L_2}\},\emptyset)$,
as for all the treks (i)-(iv) in Example~\ref{example:treks}, either $P_1$ contains a vertex in $\{\node{L_1},\node{L_2}\}$ or
$P_2$ contains a vertex in $\emptyset$.
Similarly, $\node{X_3}$ and $\node{X_4}$ can also be t-separated 
by $(\emptyset,\{\node{L_1},\node{L_2}\})$.
However, the most simple way to t-separate $\node{X_3}$ and $\node{X_4}$ is by 
$(\{\node{X_3}\},\emptyset)$ or $(\emptyset,\{\node{X_4}\})$. 
\end{example}

\begin{example} [Example of calculating rank]
    As shown in Example~\ref{example:t-sep},
    the minimal way to t-separate
    $\node{X_3}$ and $\node{X_4}$ 
    is by $(\set{C}_{\set{A}},\set{C}_{\set{B}})=(\{\node{X_3}\},\emptyset)$ or $(\emptyset,\{\node{X_4}\})$,
    and thus  $\min |\set{C}_{\set{A}}|+|\set{C}_{\set{B}}|=1$.
    Therefore, $\text{rank}(\Sigma_{\{\node{X_3}\},\{\node{X_4}\}})=1$.
    Now suppose we want to calculate $\text{rank}(\Sigma_{\{\node{X_3},\node{X_4},\node{X_5}\},
    \{\node{X_1},\node{X_6},\node{X_7}\}})$.
    As the minimal way to t-separate 
    $\{\node{X_3},\node{X_4},\node{X_5}\}$ and $\{\node{X_1},\node{X_6},\node{X_7}\}$
    is by $(\{\node{L_1},\node{L_2}\},\emptyset)$ (or $(\emptyset,\{\node{L_1},\node{L_2}\})$),
    the rank is $|\{\node{L_1},\node{L_2}\}|+|\emptyset|$=2.
    \end{example}

 \begin{figure}[t]
   \centering 
   \vspace{-1mm} 
    \begin{minipage}[c]{0.26\textwidth}
      \centering 
   \includegraphics[width=1.3\linewidth]{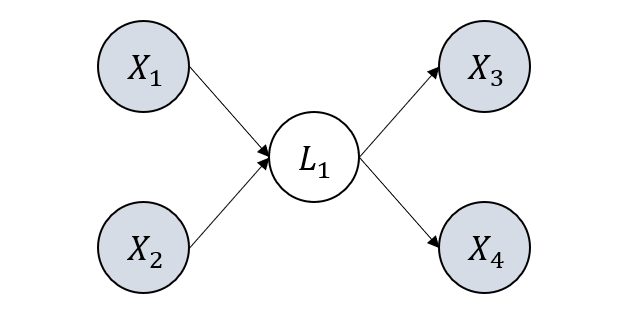}
   \end{minipage}
   \hspace{6em}
    \begin{minipage}[c]{0.35\textwidth}
      \centering 
   \caption{\small An illustrative example that highlights the motivation for using rank. When using CI, we cannot deduce that $\{\node{X}_1,\node{X}_2\}$ and 
    $\{\node{X}_3,\node{X}_4\}$ are d-separated by $\node{L_1}$ as 
    $\node{L_1}$ is latent, 
    while by using rank we can.}
   \label{fig:motiv rank}
   \end{minipage}
   \vspace{-3mm}
 \end{figure}
 
\subsection{Definition of CI Skeleton}
\label{appendix: ci skeleton}
\begin{definition}
A CI skeleton of $\set{X}_\graph$ is an undirected graph where
the edge between $\node{X_1}$ and $\node{X_2}$ exists 
iff there does not exist a set of observed variables $\set{C}$ such that $\node{X_1},\node{X_2}\notin \set{C}$ and $\node{X_1}\indep \node{X_2}|\set{C}$.
\end{definition}

Examples of CI skeleton can be found in Example~\ref{fig:intuition of atomic cover}

\subsection{Definition of Rank-invariant Graph Operator}
\label{appendix: graph operator}

The definitions are as follows with examples in Appx.~\ref{examples:operator}.

  \begin{definition}[Minimal-Graph Operator \citep{huang2022latent}]
    \label{def:minimal operator}
  Given two atomic covers $\set{L},\set{P}$ in $\graph$, we can merge $\set{L}$ to $\set{P}$ if the following conditions hold: (i) $\set{L}$ is the pure children of $\set{P}$, (ii) all elements of $\set{L}$ and $\set{P}$ are latent and  $|\set{L}| = |\set{P}|$, and (iii) the pure children of $\set{L}$ form a single atomic cover, or the siblings of $\set{L}$ form a single atomic cover. 
  We denote such an operator as minimal-graph operator $\mathcal{O}_{\text{min}}(\graph)$.
  \end{definition}

  \begin{definition}[Skeleton Operator \citep{huang2022latent}]
    \label{def:skeleton operator}
    Given an atomic covers $\set{V}$ in a graph $\graph$, for all  $\node{V} \in \set{V}$, $\node{V}$ is latent, and all $\node{C} \in \purechildren(\set{V})$, such that $\node{V}$ and $\node{C}$ are not adjacent in $\graph$, we can draw an edge from $\node{V}$ to $\node{C}$. We denote such an operator as skeleton operator $\mathcal{O}_s(\graph)$.
  \end{definition}
  
\subsection{Max-Flow-Min-Cut Lemma for Treks}
\label{appendix: maxflowmincut lemma}

\begin{lemma} [Max-Flow-Min-Cut for Treks \citep{sullivant2010trek}]
\label{lemma: maxflowmincut}
The minimal $|\set{C}_\set{A}|+|\set{C}_\set{B}|$ s.t., $(\set{C}_\set{A}|,|\set{C}_\set{B})$ t-separates $\set{A}$ from $\set{B}$, equals to the maximum number of non-overlapping (no sided intersection) treks from $\set{A}$ to $\set{B}$.
\end{lemma}

\subsection{Example for Theorem~\ref{theorem: condition for not adjacent}}
\label{appendix: nonadjacency example}

\begin{figure}[th]
    \vspace{-2mm}
  \setlength{\belowcaptionskip}{-2.5mm}
    \begin{subfigure}[b]{0.3\textwidth}
      \includegraphics[width=\textwidth]{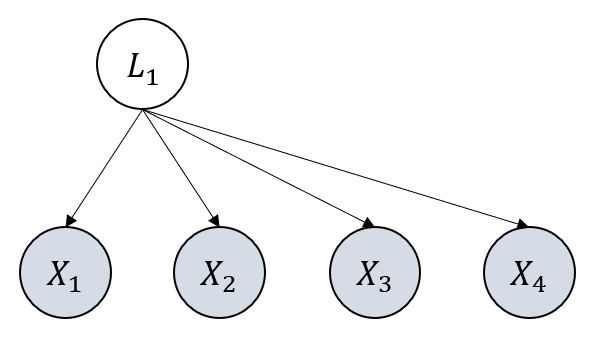}
      \caption{Illustrative graph $\graph_1$ for Theorem~\ref{theorem: condition for not adjacent}.}
  \end{subfigure}
  \hfill
  \begin{subfigure}[b]{0.3\textwidth}
    \includegraphics[width=\textwidth]{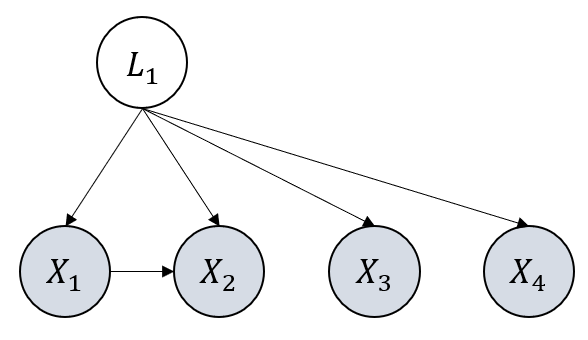}
    \caption{Illustrative graph $\graph_2$ for Theorem~\ref{theorem: condition for not adjacent}.}
  \end{subfigure}
  \hfill
  \begin{subfigure}[b]{0.3\textwidth}
    \includegraphics[width=\textwidth]{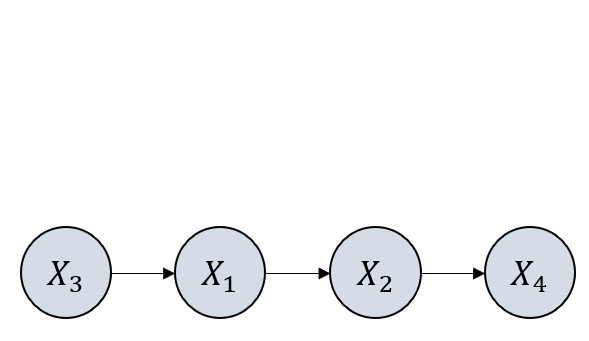}
    \caption{Illustrative graph $\graph_3$ for Theorem~\ref{theorem: condition for not adjacent}.}
  \end{subfigure}
  \caption{Illustrative figures for Theorem~\ref{theorem: condition for not adjacent}.}
    \label{fig:example for nonadjacency}
    \vspace{-0mm}
  \end{figure}
\begin{example}
\label{example: nonadjacency}
In Figure~\ref{fig:example for nonadjacency} (a),
we can employ Theorem~\ref{theorem: condition for not adjacent}
to check whether $\node{X_1}$ and $\node{X_2}$
are adjacent.
Specifically, let $\set{A}=\{\node{X_3}\}$ and 
$\set{B}=\{\node{X_4}\}$, 
then the condition in Theorem~\ref{theorem: condition for not adjacent} is satisfied,
i.e., $\text{rank}(\Sigma_{\set{A}\cup\{\node{X_1}\},
\set{B}\cup\{\node{X_2}\}})=\text{rank}(\Sigma_{\set{A},\set{B}})$ 
and 
$\text{rank}(\Sigma_{\set{A}\cup\{\node{X_1},\node{X_2}\},
\set{B}\cup\{\node{X_1},\node{X_2}\}})=\text{rank}(\Sigma_{\set{A},\set{B}})+2$. 
Therefore  $\node{X_1}$ and $\node{X_2}$ are not adjacent.

When it comes to Figure~\ref{fig:example for nonadjacency} (b),
the condition does not hold, and thus we cannot conclude that 
$\node{X_1}$ and $\node{X_2}$ are not adjacent.

Figure~\ref{fig:example for nonadjacency} (c) is an example to show that 
$\text{rank}(\Sigma_{\set{A}\cup\{\node{X_1},\node{X_2}\},
\set{B}\cup\{\node{X_1},\node{X_2}\}})=\text{rank}(\Sigma_{\set{A},\set{B}})+2$ 
in the condition is important in the theorem.
We need this condition to ensure than treks between $\set{A}$ and $\set{B}$ do not rely on $\node{X_1}$ or $\node{X_2}$. Otherwise, as in
Figure~\ref{fig:example for nonadjacency} (c),
if we only check $\text{rank}(\Sigma_{\set{A}\cup\{\node{X_1}\},
\set{B}\cup\{\node{X_2}\}})=\text{rank}(\Sigma_{\set{A},\set{B}})$, we will found that it holds if we take $\set{A}=\{\node{X_3}\}$ and 
$\set{B}=\{\node{X_4}\}$.
Therefore, we will mistakenly conclude that 
$\node{X_1}$ and $\node{X_2}$ are not adjacent.

\end{example}



\begin{figure}[h]
    \vspace{-2mm}
  \setlength{\belowcaptionskip}{-2.5mm}
    \begin{subfigure}[b]{0.4\textwidth}
      \includegraphics[width=\textwidth]{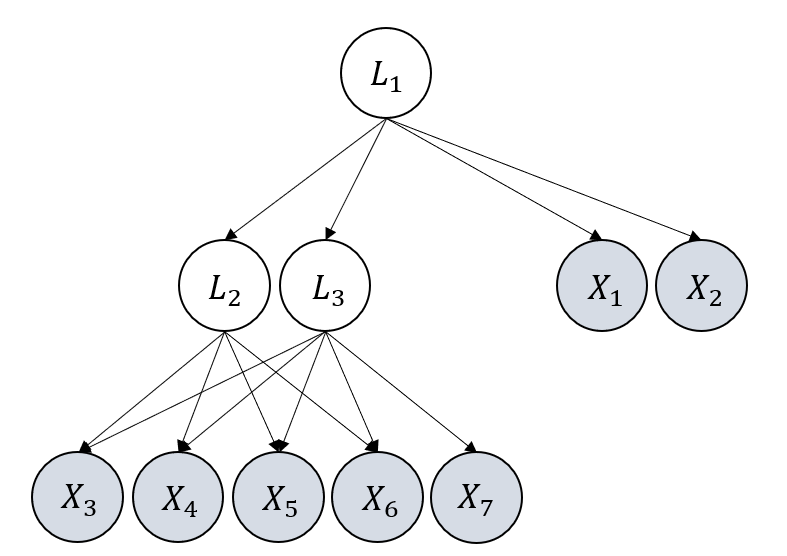}
      \caption{$\graph_1$.}
  \end{subfigure}
  \hfill
  \begin{subfigure}[b]{0.4\textwidth}
    \includegraphics[width=\textwidth]{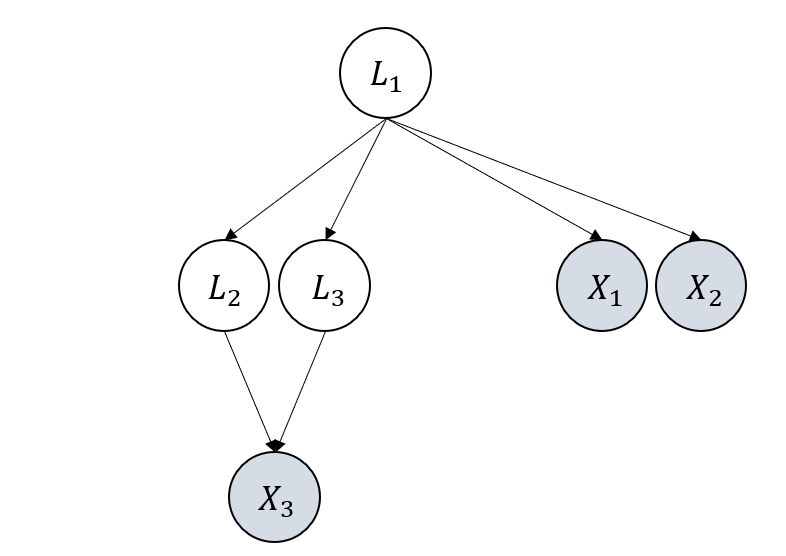}
    \caption{$\graph_2$.}
  \end{subfigure}
  \caption{Illustrative figures for Theorem~\ref{theorem:measurement as surrogate}.}
    \label{fig:example for surrogate}
    \vspace{-0mm}
  \end{figure}
\subsection{Example for Theorem~\ref{theorem:measurement as surrogate}}
\label{appendix: example surrogate}

\begin{figure}[t]
   \vspace{0mm}
 \setlength{\belowcaptionskip}{-2.5mm}
 \begin{subfigure}[b]{0.25\textwidth}
   \centering
   \includegraphics[width=\textwidth]{figures/intuition_atomic_cover/1.png}
   \caption{\small $\graph_1$}
 \end{subfigure}
 \hfill
 \begin{subfigure}[b]{0.25\textwidth}
   \centering
   \includegraphics[width=\textwidth]{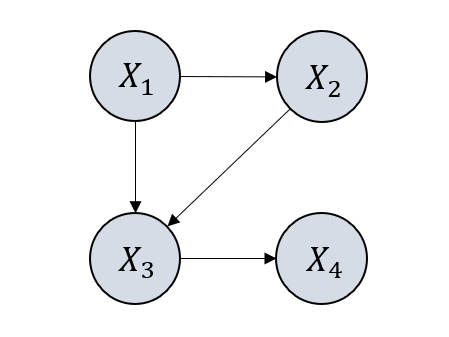}
   \caption{\small $\graph_1'$.}
 \end{subfigure}
  \hfill
 \begin{subfigure}[b]{0.25\textwidth}
   \centering
   \includegraphics[width=\textwidth]{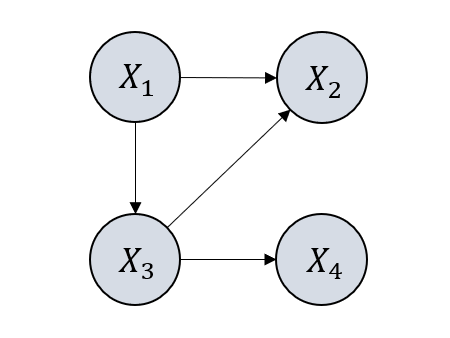}
   \caption{\small $\graph_1''$.}
 \end{subfigure}
 \caption{Examples to show that $\graph_1$ cannot be identified as $\graph_1$ $\graph_1'$,$\graph_1''$ have the same observational rank information.}
   \label{fig:more intuition of atomic cover}
 \end{figure}
 
\begin{example}
\label{example:measurement as surrogate}
Take Figure~\ref{fig:example for surrogate} (a) as an example.
Suppose we aim to calculate $\text{rank}(\Sigma_{\set{A},\set{B}})$, where
$\set{A}=\{\node{L_2},\node{L_3}\}$ and $\set{B}=\{\node{X_1},\node{X_2}\}$.
We can employ the pure children of $\set{A}$, $\set{C}=\{\node{X_3},...,\node{X_7}\}$ to do so.
Specifically, we have $\text{rank}(\Sigma_{\set{A},\set{B}})$=$\text{rank}(\Sigma_{\set{C},\set{B}})$=
$\text{rank}(\Sigma_{\{\node{X_3},...,\node{X_7}\},\{\node{X_1},\node{X_2}\}})=1$.

However, in Figure~\ref{fig:example for surrogate} (b), it is not the case.
this is because the number of pure children of $\{\node{L_2},\node{L_3}\}$ is not enough and thus 
$\text{rank}(\Sigma_{\set{A},\set{C}})\neq|\set{A}|$. In this case 
$\set{C}=\{\node{X_3}\}$ cannot work as a surrogate for calculating 
$\text{rank}(\Sigma_{\set{A},\set{B}})$.
\end{example}

\subsection{Proof of Theorem~\ref{theorem: condition for not adjacent}}
\label{proof: condition for not adjacent}
\begin{proof}
    Assume that the condition holds and suppose $\text{rank}(\Sigma_{\set{A},\set{B}})=t$, which means that there exist t non-overlapping treks between $\set{A}$ and $\set{B}$.
    By $\text{rank}(\Sigma_{\set{A}\cup\{\node{X_1},\node{X_2}\},\set{B}\cup\{\node{X_1},\node{X_2}\}})=\text{rank}(\Sigma_{\set{A},\set{B}})+2$, we further have that these t treks do not necessarily travel across $\node{X_1}$ or $\node{X_2}$.
    As such, if $\node{X_1}$ and $\node{X_2}$ are adjacent, then we must have
    $\text{rank}(\Sigma_{\set{A}\cup\{\node{X_1}\},\set{B}\cup\{\node{X_2}\}})=t+1\neq \text{rank}(\Sigma_{\set{A},\set{B}})$, which contradicts with the condition. Therefore $\node{X_1}$ and $\node{X_2}$ cannot be adjacent.
\end{proof}

Here we show that the condition generalizes PC's condition.
 SGS and PC \citep{spirtes2000causation} proposed: there exist a set of observed variables $\set{C}\subseteq \set{X}_\graph$, $\node{X_1},\node{X_2} \notin \set{C}$, such that $\node{X_1} \indep \node{X_2} | \set{C}$.
   This condition can be expressed in the form of    Theorem~\ref{theorem: condition for not adjacent}, with $\set{C} = \set{A} = \set{B}$, because we have  
   $\node{X_1} \indep \node{X_2} | \set{C}$ iff $\text{rank}(\Sigma_{\set{A}\cup\{\node{X_1}\},\set{B}\cup\{\node{X_2}\}})=\text{rank}(\Sigma_{\set{A},\set{B}})$.
   Plus, $\text{rank}(\Sigma_{\set{A}\cup\{\node{X_1},\node{X_2}\},\set{B}\cup\{\node{X_1},\node{X_2}\}})=\text{rank}(\Sigma_{\set{A},\set{B}})+2$ is always true when $\set{A}=\set{B}$.

\subsection{Proof of Theorem~\ref{theorem: global for latent}}
\label{proof:theorem: global for latent}

We first introduce Lemma~\ref{lemma: rank and ci when no latent} as follows to show that when there is no latent variable, the rank information should be aligned with what CI skeleton provides.

\begin{lemma}
    When there is no latent variable in the graph,
     rank and CI are equally informative about the underlying structure,
      i.e., the rank-equivalence class and the Markov equivalence class are 
      the same when there is no latent variable.
 \label{lemma: rank and ci when no latent}
 \end{lemma}

\begin{proof}
 (i) As all d-sep can be stated by rank according to Lemma~\ref{lemma: rank and d-sep},
 using rank information is able to arrive at the markov equivalence class.
 (ii) Every element in the markov equivalence class are distributionally equivalent in terms of second order statistics.
 Therefore, using information from the rank of the covariance matrix cannot differentiate
 elements in the markove equivalence class.
 Taking (i) and (ii) together, we have that 
 the rank-equivalence class and the Markov equivalence class are 
 the same when there is no latent variable.
\end{proof}

Bellow is the proof of Theorem~\ref{theorem: global for latent}.

\begin{proof}
First we assume that there is no latent variable 
and we have that the CI skeleton 
among observed variables is also the skeleton 
of the true underlying graph $\graph$.
By (ii) and (iii), we have
(ii') $\forall~\text{distinct}~ 
\node{A_1},\node{A_2} \in\set{A}$, $\node{A_1},\node{A_2}$ 
are adjacent in $\graph$,
(iii') $\forall \node{A}\in\set{A},\node{B}\in\set{B}$, 
$\node{A},\node{B}$ are adjacent in $\graph$.
By (ii') and (iii'),
it must hold that (iv')
$\text{rank}(\Sigma_{\set{A}\cup\set{C},\set{B}\cup\set{C}}) 
 = |\set{A}|+|\set{C}|$.
However, (iv') contradicts with (iv).
Therefore, there must exist at least one latent variable.
\end{proof}

\subsection{Proof of Theorem~\ref{theorem:measurement as surrogate}}
\label{proof:measurement as surrogate}

\begin{proof}
By Theorem~\ref{theorem:rank and t}
and  Lemma~\ref{lemma: maxflowmincut},
we have $\text{rank}(\Sigma_{\set{A},{\set{B}}})$
equals the maximum number of non-overlapping trek paths from $\set{A}$ 
to $\set{B}$.
As $\set{C}$ are pure children of $\set{A}$ and no element in $\set{B}$
are descendants of $\set{C}$,
every trek from $\set{C}$ to $\set{B}$ must travel across $\set{A}$.
Therefore,  $\text{rank}(\Sigma_{\set{A},{\set{B}}})\geq\text{rank}(\Sigma_{\set{C},{\set{B}}})$.
If  we further have $\text{rank}(\Sigma_{\set{A},{\set{C}}})=|\set{A}|$,
then the   maximum number of non-overlapping trek paths from $\set{A}$ to $\set{B}$
equals  the maximum number of non-overlapping trek paths from $\set{C}$ to $\set{B}$, and thus  $\text{rank}(\Sigma_{\set{A},{\set{B}}})=\text{rank}(\Sigma_{\set{C},{\set{B}}})$.
\end{proof}

\subsection{Proof of Theorem~\ref{theorem: rank property of atomic cover}}
\label{proof: rank property of atomic cover}

\begin{proof}
As none of elements in $\setset{X}_\graph\setminus \setset{C}$ are descendants of $\setset{C}$, by Theorem~\ref{theorem:rank and t}, 
the minimal way to block every trek path between 
$\setset{C}\cup \setset{X}$ and $(\setset{X}_\graph\setminus \setset{C})\cup\setset{X}$ is by blocking all the elements of the atomic
cover $\set{V}$ and all the elements in $\set{X}$.
As $\set{X}$ is a subset of $\set{V}$, we have 
$\text{rank}(\Sigma_{\setset{C}\cup \setset{X}, (\setset{X}_\graph\setminus \setset{C})\cup\setset{X}})=|\set{V}|=k$.
\end{proof}

\subsection{Example for Theorem~\ref{theorem:unique_rank_of_atomic_cover}}
\label{example: theorem unique rank}

\begin{example} [Example for the uniqueness of rank deficiency in Theorem\ref{theorem:unique_rank_of_atomic_cover}]
\label{example:uniqueness}
In Figure~\ref{fig:example1}, 
if the current $k=2$, and we assume that all $k=1$ clusters are found,
and no v-structure exists,
the rank deficiency would uniquely map to a $k=2$ cluster. 
E.g., if we take 
$\setset{C}=\{\{\node{X_6}\},\{\node{X_7}\}\}$ and $\setset{X}=\{\{\node{X_2}\}\}$, we have 
$\text{rank}(\Sigma_{\setset{C}\cup \setset{X}, \setset{X}\cup\setset{X}_\graph \backslash \setset{C}\backslash\meassureddes(\setset{C})})=
\text{rank}(\Sigma_{\{\node{X_6},\node{X_7}, \node{X_2}\}, \{\node{X_1},...,\node{X_5},\node{X_8},...,\node{X_{16}}\}})=2$.
In this case, this rank deficiency uniquely relates to 
a cover $\set{V}=\set{X}\cup\set{L}$, where $\set{X}=\cup_{\set{X'} \in \setset{X}} \set{X'}
=\{\node{X_2}\}$, $\set{L}$ is latent variable to be added with $|\set{L}|=k-|\set{X}|=1$, and $\setset{C}=\{\{\node{X_6}\},\{\node{X_7}\}\}$ are the pure children of $\set{V}$.

In contrast, 
if we are searching for $k=2$ and a $1$-cluster $\node{L_4}\rightarrow\{\node{X_{14}},
\node{X_5}\}$
has  not been identified, then 
the condition for 
Theorem~\ref{theorem:unique_rank_of_atomic_cover} 
is not satisfied 
and thus the uniqueness of rank deficiency does not hold:
e.g., by  taking $\setset{C}=\{\node{X_{13}},\node{X_{14}},\node{X_{15}}\}$ and $\setset{X}=\{\}$, we have $\text{rank}(\Sigma_{\{\node{X_{13}},\node{X_{14}},\node{X_{15}}\},
 \{\node{X_1},...,\node{X_{12}},\node{X_{16}}\}})=2$, 
which is deficient, and yet 
$\{\node{X_{13}},\node{X_{14}},\node{X_{15}}\}$ are not from a $k=2$ cluster.
This is because this rank deficiency is not from a $k=2$ cluster, rather,
it is from the $1$-cluster  $\{\node{X_{14}},
\node{X_{15}}\}$ with parent $\node{L_4}$ that has not been found yet.
\end{example}

\subsection{Proof of Theorem~\ref{theorem:unique_rank_of_atomic_cover}}
\label{proof:unique_rank_of_atomic_cover}

\begin{proof}
We first  show that (a)
if $||\setset{X}||=t=0$
and, elements from $\setset{C}$  are pure children of two or more atomic covers, then we must have 
$\text{rank}(\Sigma_{\setset{C}\cup \setset{X}, \setset{X}\cup\setset{X}_\graph \setminus \setset{C}\setminus\meassureddes(\setset{C})})= k+1$.

Proof of (a).
As there is no collider between atomic covers,
we have that all the elements from $\setset{C}$ are pure children of
atomic covers. 
Suppose $\setset{C}$ can be partitioned into $\setset{C}_1,...,\setset{C}_N$, where each $\setset{C}_i$ are the pure children of a distinct atomic cover in $\graph$.
If $\setset{C}_i$ are the pure children of an atomic cover with cardinality $k'<k$,
then we have $||\setset{C}_i||\leq k'$. If
$\setset{C}_i$ are the pure children of an atomic cover with cardinality $k'\geq k \geq ||\setset{C}_i||$ ($k\geq ||\setset{C}_i||$ because if $k<||\setset{C}_i||$ then all elements of $\setset{C}$ are from the same cluster),
so we also have $||\setset{C}_i||\leq k'$. Therefore,
 by Lemma~\ref{lemma: maxflowmincut} and the fact that each atomic cover has $k+1-t$ pure children and $k+1$ additional neighbors,
 we have
 $\text{rank}(\Sigma_{\setset{C}\cup \setset{X}, \setset{X}\cup\setset{X}_\graph \setminus \setset{C}\setminus\meassureddes(\setset{C})})= 
 \text{rank}(\Sigma_{\setset{C}, \setset{X}_\graph \setminus \setset{C}\setminus\meassureddes(\setset{C})})=\sum_{1}^{N} ||\setset{C}_i||=k+1$. Therefore, when $||\setset{X}||=t=0$,
 the rank deficiency property does not hold
 when elements of $\setset{C}$ are from different clusters.

(b) When $||\setset{X}||=t\neq0$,
we consider a new graph $\graph''$,
where all variables from $||\setset{X}||$
are removed (as well as related edges).
Assume elements of $\setset{C}$ are from different clusters and elements of $\setset{X}$ are from the same atomic cover.
Thus by (a), we have that 
the maximum number of non-overlapping treks in $\graph''$
between 
$\setset{C}$ and $\setset{X}_{\graph''} \setminus \setset{C}\setminus\text{MDe}_{\graph''}(\setset{C})$ is $k+1-t$.
Then we add $\setset{X}$ with relating edges back to the graph and 
thus we will have $||\setset{X}||=t$ additional non-overlapping treks between 
$\setset{C}\cup \setset{X}$ and $\setset{X}\cup\setset{X}_\graph \setminus \setset{C}\setminus\meassureddes(\setset{C})$.
Therefore, we also have
$\text{rank}(\Sigma_{\setset{C}\cup \setset{X}, \setset{X}\cup\setset{X}_\graph \setminus \setset{C}\setminus\meassureddes(\setset{C})})=k+1$, when $||\setset{X}||=t\neq0$ and elements of $\setset{C}$ are from different clusters. Similarly, if elements of $\setset{C}$ are from the same cluster but not all elements of $\setset{X}$ are the parents of that cluster, we also have  $\text{rank}(\Sigma_{\setset{C}\cup \setset{X}, \setset{X}\cup\setset{X}_\graph \setminus \setset{C}\setminus\meassureddes(\setset{C})})=k+1$.

Taking (a) and (b) together, we have that rank deficiency holds only if all elements of $\setset{C}$ are from the same cluster and all elements of $\setset{X}$  are the parents of that cluster.
\end{proof}

\subsection{Proof of Theorem~\ref{theorem:necessary and sufficient under cond}}
\label{proof:necessary and sufficient under cond}
In the proof of Theorem~\ref{theorem: global for latent}, we have already shown the 'if' direction. Now we are going to show the sketch of the proof for the 'only if' direction.
\begin{proof}
Suppose there is a latent variable in $\graph$.
According to Condition~\ref{cond:basic}, it must belong to an atomic cover, say $\set{V}$ and we suppose that $\set{V}$ 
contains $n$ latent variables in total, rest of which are observed $\set{X}$.
According to the definition of atomic cover,
$\set{V}$ has at least $n+1$ pure children and $n+1$ neighbours that are distinct with the $n+1$ pure children.
Assume that all of them are observed.
Then we can simply take $\set{A}$ as the pure children, $\set{B}$ as the neighbours, $\set{C}$ as $\set{X}$,
and thus conditions (i)-(v) will be all satisfied.
If some of the pure children or neighbors of $\set{V}$ are latent, we can simply use their pure children instead (if the pure children are still latent, use the pure children of the pure children and finally we will find enough observed pure children/descendants, as latent variables cannot be leaf nodes).
Thus the conditions (i)-(v) can also be satisfied.
Therefore, if there is at least a latent variable, then there must exist disjoint $\set{A}$,$\set{B}$, and $\set{C}$, such that (i)-(v) hold.
\end{proof}

\subsection{Proof of Lemma~\ref{lemma: rank and d-sep}}
\label{proof: rank and d-sep}
\begin{proof}
By Theorem~\ref{theorem:tsep and dsep} 
we have that for disjoint $\set{A}$,$\set{B}$ and $\set{C}$,
$\set{C}$ d-separates $\set{A}$ from $\set{B}$, iff
 there is a partition $\mathbf{C} = \mathbf{C}_{\mathbf{A}} \cup \mathbf{C}_{\mathbf{B}}$ such that
  ($\mathbf{C}_{\mathbf{A}}$,$\mathbf{C}_{\mathbf{B}}$)
   t-separates $\mathbf{A}\cup\mathbf{C}$ from $\mathbf{B}\cup\mathbf{C}$.
   By Theorem~\ref{theorem:rank and t},
    we have that $\text{rank}(\Sigma_{\set{A}\cup\set{C},\set{B}\cup\set{C}})\leq|\set{C}|$.
    Plus, by the definition of treks,
    $\text{rank}(\Sigma_{\set{A}\cup\set{C},\set{B}\cup\set{C}})\geq|\set{C}|$.
    Therefore, $\set{C}$ d-separates $\set{A}$ from $\set{B}$, iff $\text{rank}(\Sigma_{\set{A}\cup\set{C},\set{B}\cup\set{C}})=|\set{C}|$.

\end{proof}

\subsection{Proof of Theorem~\ref{theorem:phase3}}
\label{proof:phase3}

\begin{proof}
The sketch of the proof is as follows.

We first show that a fake cover will not influence all other found structures except itself and its neighbors in the result $\graphp$.
By Lemma 11 in \citep{huang2022latent},
we have that a fake cover with observed  descendants $\set{X}$ in $\graphp$ implies a bond set in $\graph$ (whose definition can be found in \citet{huang2022latent}),
and there is a partition of the rest of the observed variables $\set{X}_\graph \backslash \set{X}$ into two groups $\set{A}$ and $\set{B}$ such that $\set{A}$ and $\set{B}$ are d-separated by the bond set.
Suppose this faker cover is $\set{V}$ that corresponds to a set of $n$ latent covers $\{\set{L_1},...,\set{L_n}\}$ in $\graph$.
Then $\{\set{L_1},...,\set{L_n}\}$ d-separates $\set{A}$ and $\set{B}$, and thus during the search we will generate two dummy covers that interact with $\set{A}$ and $\set{B}$ respectively,
during which the rank information will not be mistaken.
Then we show that in Phase 3, the fake cover can be corrected. When we refine the fake cover $\set{V}$, we will delete it together with its neighbours and thus the two dummy covers will be deleted. Suppose now we are searching for $k$ clusters, as this time all the remaining $k'$-clusters s.t. $k'<k$ have been found, the FindCausalCluster function will not generate fake cluster anymore. Therefore, the output would be corrected.

\end{proof}

\subsection{Proof of Theorem~\ref{theorem:identifiability}}
\label{proof:identifiability}

First, we show an extension of Theorem~\ref{theorem:measurement as surrogate} to atomic covers.
\begin{lemma} [Pure Children as Surrogate for Atomic Covers]
Let $\set{A}=\set{L}\cup\set{X}$ be an atomic cover, $\setset{C} \subseteq \purechildren(\set{A})$ be a subset of pure children of $\set{A}$, and $\set{B}_1$,$\set{B}_2$ be two sets of variables such that for all $\mathsf{B} \in \set{B}_1\cup\set{B}_2$, $\node{B} \notin \descendant(\setset{C})$. We have $\text{rank}(\Sigma_{\{\set{A}\}\cup\{\set{B}_1\},\set{B}_2}) = \text{rank}(\Sigma_{\setset{C}\cup\{\set{X}\}\cup\{\set{B}_1\},{\set{B}_2}})$,
if $||\setset{C}||+|\set{X}|\geq |\set{A}|$.
\label{lemma:measurement as surrogate for atomic covers}
\end{lemma}
\begin{proof}
    By Lemma~\ref{lemma: maxflowmincut}, $\text{rank}(\Sigma_{\{\set{A}\}\cup\{\set{B}_1\},\set{B}_2})$ is the max number of non-overlapping treks between $\{\set{A}\}\cup\{\set{B}_1\}$ and $\set{B}_2$.
    If $||\setset{C}||+|\set{X}|\geq |\set{A}|$ holds,
   all the treks starting from 
   $\{\set{A}\}\cup\{\set{B}_1\}$
   can be extended to treks that start from 
    $\setset{C}\cup\{\set{X}\}\cup\{\set{B}_1\}$,
    and   the max number of non-overlapping treks is the same, which means 
    $\text{rank}(\Sigma_{\{\set{A}\}\cup\{\set{B}_1\},\set{B}_2}) = \text{rank}(\Sigma_{\setset{C}\cup\{\set{X}\}\cup\{\set{B}_1\},{\set{B}_2}})$.
\end{proof}

This lemma informs  us  that if we correctly found a cover $\set{A}=\set{L}\cup\set{X}$ by
our rules in Algorithm~\ref{alg:phase1}, we can calculate the rank relating to $\set{A}$ by using its pure children $\setset{C}$ together with part of the observation of $\set{A}$, i.e., $\set{X}$ as surrogates, even though part of $\set{A}$ , i.e., $\set{L}$, cannot be observed. Note that
$||\setset{C}||+|\set{X}|\geq |\set{A}|$ always holds as it is required when we are searching for clusters in Algorithm~\ref{alg:phase1}.

Next, we show that 
when we are searching for combinations of $\setset{C}$ and $\setset{X}$
in Algorithm~\ref{alg:phase1},
by leveraging the checking function \text{NoCollider} defined in Algorithm~\ref{alg:checkcollider}, the correctness of Algorithm~\ref{alg:phase1} will not be influenced even though there might exist
colliders in $\setset{C}$.

\begin{lemma}[Colliders in $\setset{C}$ do not harm] 
Suppose there exist some collider structures in graph $\graph$, e.g., there exist two atomic covers $\set{V}_1$ and $\set{V}_2$,
with $\set{A}$ the minimal set of variables that d-separates $\set{V_1}$ from $\set{V_2}$, and  $\set{C}$
as a collider of $\set{V_1}$, $\set{V_2}$.
The correctness of Algorithm~\ref{alg:phase1} will not be influenced by the existence of colliders in $\setset{C}$.
\label{lemma:handle colliders in C}
\end{lemma}

\begin{proof}
In Algorithm~\ref{alg:phase1},
we check whether different combinations of $\setset{C}$, $\setset{X}$, and $\setset{N}$ induce rank deficiency.
Suppose we take $\setset{C}=\{\{\set{V'}_1\},\{\set{V'}_2\},\{\set{C'}\}\}$, where $\set{V'}_1 \subseteq \set{V}_1$,
$\set{V'}_2 \subseteq \set{V}_2$, and
$\set{C'} \subseteq \set{C}$ and  let $\set{R}$ be $\set{V}_1\cup\set{V}_2\backslash\set{V'}_1\backslash\set{V'}_2$.\\
(i) If $|\set{C'}|\leq |\set{R}|$,
and rank deficiency holds,
we have $|\set{V'}_1\cup\set{V'}_2\cup\set{C'}|=1+|\set{C'}|+|\set{A}|$.
Therefore, we can detect $\set{C'}$ in $\setset{C}$ by  the checking function \text{NoCollider},
because by removing $\set{C'}$ we have
$|\set{V'}_1\cup\set{V'}_2|=1+|\set{A}|$.\\
(ii) If $|\set{C'}| > |\set{R}|$,
and rank deficiency holds when checking $k$,
we have 
$|\set{V'}_1\cup\set{V'}_2\cup\set{C'}|=1+|\set{R}|+|\set{A}|=k+1$,
which means 
$|\set{V}_1\cup\set{V}_2|=|\set{V'}_1|+|\set{V'}_2|+|\set{R}|=1+|\set{R}|+|\set{A}|+|\set{R}|-|\set{C'}|\leq k$.
Therefore, by the unfolding order in Algorithm~\ref{alg:phase1}, $\set{C}$ will be taken as the children of $\set{V}_1$ and $\set{V}_2$ first and thus will not induce incorrect rank deficiency.
\end{proof}

Further, under Condition~\ref{cond:vstructure},
we can  show that 
the correctness of Algorithm~\ref{alg:phase1} will not be influenced even though there might exist
colliders in $\setset{N}$, which is summarized in the following Lemma.

\begin{lemma}[Under Condition~\ref{cond:vstructure}, colliders in $\setset{N}$ do not harm] 
If Condition~\ref{cond:vstructure} holds, i.e., for every collider structures $\set{V}_1$, $\set{V}_2$, $\set{C}$, and $\set{A}$,
we have
$|\set{C}| + |\set{A}| \geq |\set{V_1}|+|\set{V_2}|$,
then
the correctness of Algorithm~\ref{alg:phase1} will not be influenced by the existence of colliders in $\setset{N}$.
\label{lemma:based on condition2}
\end{lemma}
\begin{proof}
Consider a collider structure
$\set{V}_1$, $\set{V}_2$, $\set{C}$, and $\set{A}$.
The potential existence of $\set{C}$ in $\setset{N}$ will not cause rank deficiency unless it is when $k=|\set{C}|+|\set{A}|$. But under Condition~\ref{cond:vstructure}, we have $|\set{C}| + |\set{A}| \geq |\set{V_1}|+|\set{V_2}|$, so in Algorithm~\ref{alg:phase1}, the colliders $\set{C}$ will be taken as the pure children of $\set{V}_1$ and $\set{V}_2$ first.
This holds for every collider structure and thus the correctness of Algorithm~\ref{alg:phase1} will not be influenced.
\end{proof}

Now we are ready to prove Theorem~\ref{theorem:identifiability}, as follows.
\begin{proof}
As the existence of colliders in $\setset{X}$ will
enable more non-overlapping treks, the existence of colliders in $\setset{X}$ will not induce incorrect rank deficiency. 
Taking this and the above two Lemmas~\ref{lemma:handle colliders in C},\ref{lemma:based on condition2} into consideration, under Condition~\ref{cond:basic} and \ref{cond:vstructure},
the existence of colliders between atomic covers will not influence the correctness of Algorithm~\ref{alg:phase1}.
During the search process of Phase 2, 
Lemma~\ref{lemma:measurement as surrogate for atomic covers} allows us to test the rank involving partially observed atomic covers and thus we are able to iteratively find all clusters in $\graph$ by making use of Theorem~\ref{theorem:unique_rank_of_atomic_cover}, with the direction of some edges undetermined.
Further, Theorem~\ref{theorem:phase3} allows us to correct clusters induced by the violation of the assumption that when searching $k$ clusters all $k'<k$ clusters have been found, by Phase 3 in Algorithm~\ref{alg:phase3}.
Therefore, our Algorithm~\ref{alg:all} including Phase 1, 2, 3 can identify the Markov equivalence class of $\graph$, up to rank invariant operations $\mathcal{O}_{\text{min}}$ and $\mathcal{O}_s$.
\end{proof}

Corollary~\ref{cor:pcandrank} is directly from
Theorem~\ref{theorem:identifiability}.
When there is no latent, 
the Markov equivalence of $\mathcal{O}_{\text{min}}(\mathcal{O}_s(\graph))$ is the Markov equivalence of $\graph$ so asymptotically the output of RLCD is the same as that of PC.

\subsection{Rank Test}
\label{ranktest}
We employ canonical correlations \citep{ranktest} to calculate the rank of covariance matrices.
 Denote by $\alpha_i$ the $i$-th canonical correlation coefficient between two sets of variables $\set{A}$ and $\set{B}$,
 under  the null hypothesis $\texttt{rank}(\Sigma_{\set{A},\set{B}}) \leq r$ with $N$ sample size,
 the statistics $-(N\!-\!(p+q+3)/2) \sum_{i=r+1}^{\max(|\set{A}|,|\set{B}|)} \log(1-\alpha_i^2)$ is approximately
  $\chi^2$ distributed with $(|\set{A}|-r)(|\set{B}|-r)$ degrees of freedom. 


\section{Illustrations of Algorithms and More Details about Datasets}


\begin{figure}[t]
    \vspace{-0mm}
    \begin{subfigure}[b]{0.32\textwidth}
      \includegraphics[width=\textwidth]{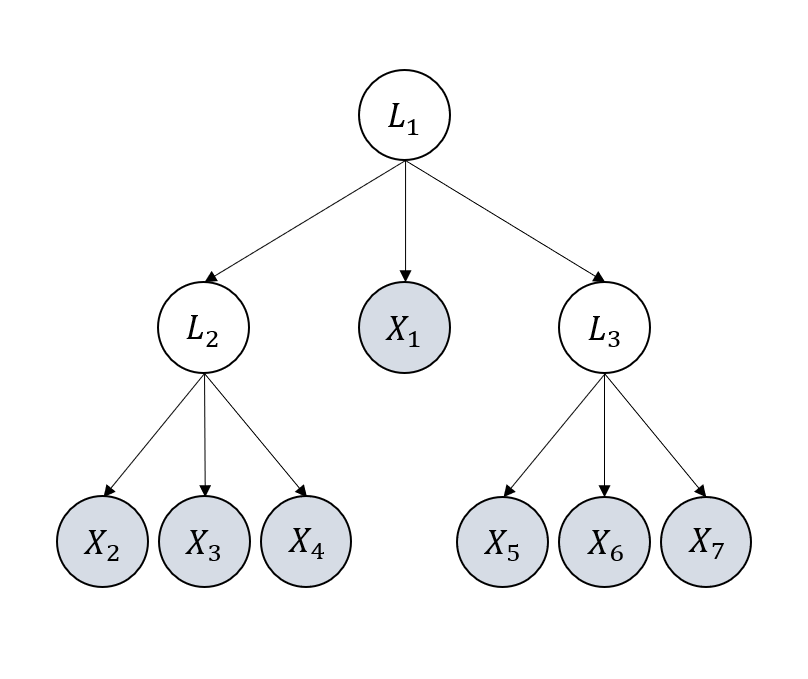}
      \caption{Illustrative graph allowed by \citet{Pearl88,zhang2004hierarchical}.}
    \end{subfigure}
    \hfill
    \begin{subfigure}[b]{0.60\textwidth}
    \includegraphics[width=\textwidth]{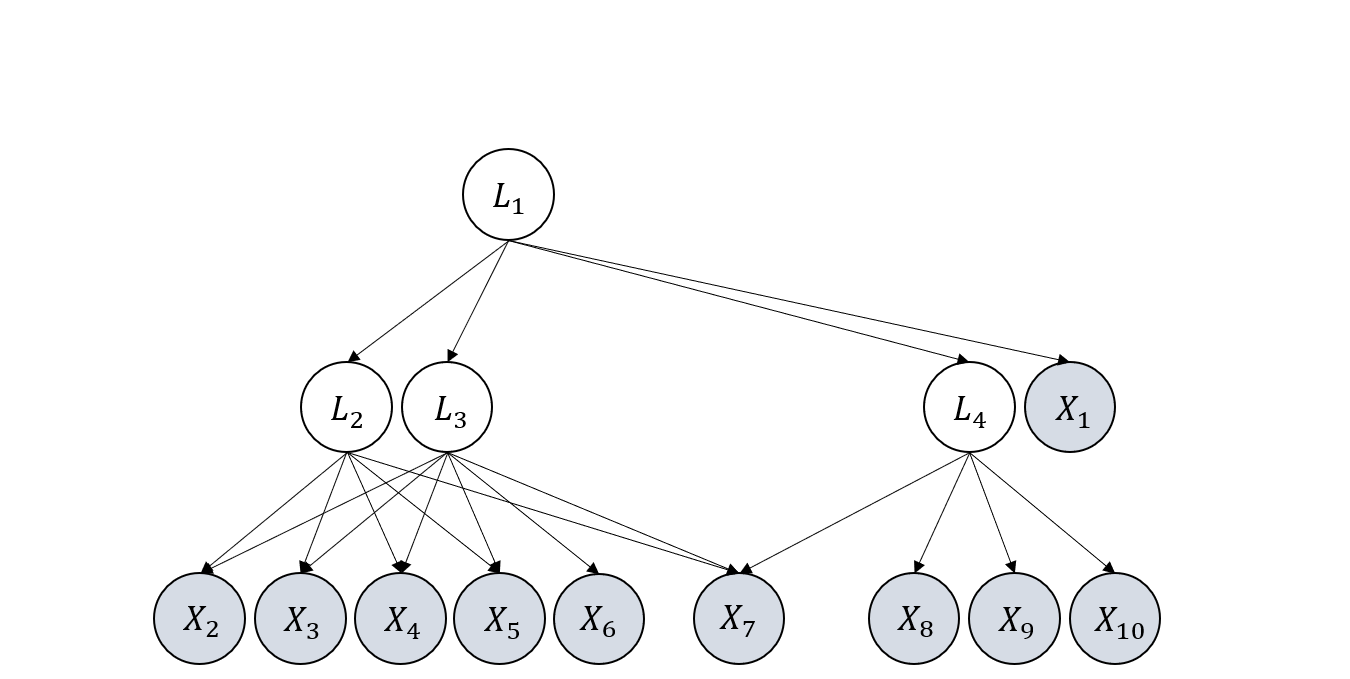}
    \caption{Illustrative graph allowed by \citet{huang2022latent}.\color{white}{place holder  place holder place holder}}
    \end{subfigure}
    \hfill
    \begin{subfigure}[b]{0.32\textwidth}
    \includegraphics[width=\textwidth]{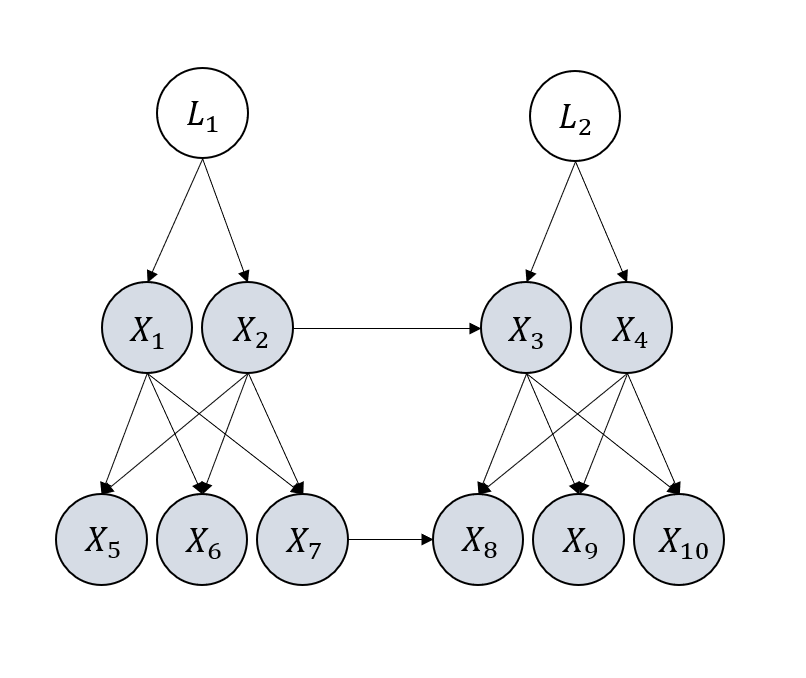}
    \caption{Illustrative graph allowed by \citet{maeda2020rcd}.}
    \end{subfigure}
    \hfill
    \begin{subfigure}[b]{0.60\textwidth}
    \includegraphics[width=\textwidth]{figures/example1.png}
    \caption{Illustrative graph allowed by the proposed method.\color{white}{place holder  place holder place holder place holder}}
    \end{subfigure}
  \caption{Examples of graphs that are allowed by each method.}
    \label{fig:compare graphs with each method}
    \vspace{-0mm}
  \end{figure}

\subsection{Examples of Grpahs that Each Method Can Handle}
\label{compare graphs with each method}

For causal discovery in the presence of latent variables, traditional wisdom often relies on strong 
graphical constraints for achieving the identifiability of the structure.
E.g., \citet{Pearl88,zhang2004hierarchical} assume that the underlying graph only follows a tree structure;
\citet{huang2022latent} assumes a more general latent hierarchical structures 
but edges among observed variables are not allowed;
\citet{maeda2020rcd} allows observed variables to be adjacent but requires that all latent variables are mutually independent.  

Illustrative graphs allowed by each method are shown in 
Figure~\ref{fig:compare graphs with each method}.
To be specific,
(a) is the illustrative graph allowed by 
\citet{Pearl88,zhang2004hierarchical} where
each cluster has only one latent variable and observed variables are not allowed to
 be directly related to each other.
 (b) is the graph allowed by \citet{huang2022latent} 
 where each cluster can have multiple latent variables. 
 However, observed variable cannot be adjacent to each other and observed variables cannot be cause of latent variables.
 (c) is the graph allowed by \citet{maeda2020rcd}, where all latent variables are required to be mutually independent.
 (d) is the graph allowed by the proposed method, where all variables are allowed to be very flexibly related to each other.

\subsection{Example of Violation of faithfulness}
\label{violation of faithfulness}
Below, we provide a special example where faithfulness does not hold. 
Suppose the true underlying graph $\mathcal{G}$ 
is $\mathsf{X}\xrightarrow{a}\mathsf{Y}\xrightarrow{b}\mathsf{Z}$ 
and 
$\mathsf{X}\xrightarrow{c}\mathsf{Z}$. 
If the corresponding SCM is parameterized with $ab+c=0$, 
then the faithfulness assumption is violated.
 This is because, when $ab+c=0$, 
 there exists another SCM with a graph $\mathcal{G}':$ 
  $\mathsf{X}\xrightarrow{a'}\mathsf{Y}\xleftarrow{b'}\mathsf{Z}$
that can generate exactly the same observational distribution 
as that of $\mathcal{G}$, 
and thus from observational data it is impossible to
 differentiate $\mathcal{G}$  and $\mathcal{G}'$
  (which results in $\text{rank}(\Sigma_{\mathsf{X},\mathsf{Z}})=0$).
   We note that such scenarios are very rare 
   and classical methods like PC \citep{spirtes2000causation} 
   cannot handle these situations either.

\

\

\

\

\

\subsection{Example of Atomic Cover}
\label{appendix: example atomic cover}

\begin{figure}[h]
    \centering 
    \vspace{0mm} 
    \includegraphics[width=0.36\textwidth]{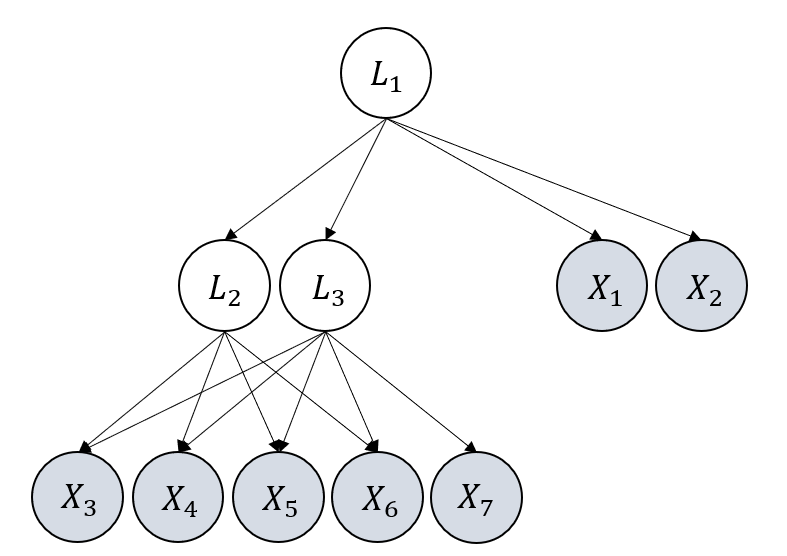}
    \caption{\small An example graph for showing atomic covers.} 
     \label{fig: example for atomic cover}
  \end{figure}
\begin{example}
  Take Figure~\ref{fig: example for atomic cover} as an example.
  Here $\{\node{X_1}\}$, $\{\node{X_2}\}$, $\{\node{X_3}\}$, $\{\node{X_4}\}$, $\{\node{X_5}\}$, $\{\node{X_6}\}$
  are all atomic covers as they only contain a single observed variable.
  We also have  
  $\set{V}=\{\node{L_2},\node{L_3}\}$ as an atomic cover.
  To show this, lets check whether
    conditions (i)-(iii) in Definition~\ref{definition:ac}
    are satisfied.
    For (i), we can let $\setset{C}=\{\{\node{X_3}\},\{\node{X_4}\},\{\node{X_5}\}\}$, and $||\setset{C}||=3\geq k+1-t=3$.
    For (ii), we can let $\setset{N}=\{\{\node{L_1}\},\{\node{X_6}\},\{\node{X_7}\}\}$, and $||\setset{N}||=3\geq k+1-t=3$.
    For (iii), it can be shown that both $\{\node{L_2}\}$ and $\{\node{L_3}\}$ cannot be an atomic cover.
    Therefore, $\{\node{L_2},\node{L_3}\}$ is an atomic cover.

    Next, let's show $\{\node{L_1}\}$ is also an atomic cover. As for (i)
    we can take $\setset{C}=\{\{\node{L_2},\node{L_3}\}\}$ with $||\setset{C}||=2\geq k-1+t=2$,
    for (ii) we can take $\setset{N}=\{\{\node{X_1}\},\{\node{X_2}\}\}$ with $||\setset{N}||=2\geq k-1+t=2$,
    and (iii) naturally holds as $\{\node{L_1}\}$ has a single element. Therefore $\{\node{L_1}\}$ is an atomic cover.

    From the example it is natural to see that when we take an atomic cover (a set) as the unit, we need to use a set of covers, i.e., $\setset{C}$, to capture the pure children of a cover.

\end{example}

\subsection{Example for Phase 1}
\label{appdneix: example phase1}

We take Figure~\ref{fig:phase 1 example} (a) as an example.
After Phase 1, we will find the CI skeleton $\graphp$.
In $\graphp$,
 we have three maximal cliques that have cardinality$\geq3$. They are $\{\node{X_1},\node{X_2},\node{X_3},\node{X_4},\node{X_5},\node{X_6}\}$, $\{\node{X_2},\node{X_3},\node{X_7}\}$, and $\{\node{X_9},\node{X_{10}},\node{X_{11}},\node{X_{12}}\}$.
 Then we partition them into groups such that two cliques $\set{Q_1},\set{Q_2}$ are in the same group if
$|\set{Q_1}\cap\set{Q_2}|\geq 2$.
Thus we have two groups of cliques $\setset{Q}_1=\{\{\node{X_1},\node{X_2},\node{X_3},\node{X_4},\node{X_5},\node{X_6}\},\{\node{X_2},\node{X_3},\node{X_7}\}\}$
and $\setset{Q}_2=\{\{\node{X_9},\node{X_{10}},\node{X_{11}},\node{X_{12}}\}\}$.
Given $\setset{Q}_1$ and $\setset{Q}_2$,
we get $\set{X}_{\setset{Q}_1}=
\cup_{\set{Q}\in\setset{Q}_1}\set{Q}=\{\node{X_1},\node{X_2},\node{X_3},\node{X_4},\node{X_5},\node{X_6},\node{X_7}\}$ and 
$\set{X}_{\setset{Q}_2}=
\cup_{\set{Q}\in\setset{Q}_2}\set{Q}=\{\node{X_9},\node{X_{10}},\node{X_{11}},\node{X_{12}}\}$,
the corresponding input to Phase 2 and 3 will be  
$\set{X}_{\setset{Q}_1}\cup\set{N}_{\setset{Q}_1}=\{\node{X_1},\node{X_2},\node{X_{3}},\node{X_{4}},\node{X_{5}},\node{X_{6}},\node{X_{7}},\node{X_{8}}\}$ and $\set{X}_{\setset{Q}_2}\cup\set{N}_{\setset{Q}_2}=\{\node{X_9},\node{X_{10}},\node{X_{11}},\node{X_{12}},\node{X_{8}}\}$ respectively.

We next show that if there exist disjoint $\set{A}$, $\set{B}$, $\set{C}$,
s.t., (i)-(iv) in Theorem~\ref{theorem: global for latent} hold, then 
$\set{A}\cup\set{B}\cup\set{C}\subseteq\set{X}_{\setset{Q}}$.

If there exist disjoint $\set{A}$, $\set{B}$, $\set{C}$,
s.t., (i)-(iv) in Theorem~\ref{theorem: global for latent} hold. Then $\set{A}$ itself is a clique, and for all $\node{B}\in\set{B}$, $\set{A}\cup\{\node{B}\}$ is also a clique. Plus $\set{A}$ and $\set{A}\cup\{\node{B}\}$ have at least two common elements as $|\set{A}|\geq 2$. Thus, after our processing,
$\set{A},\set{B}$ will both be subsets of a same $\set{X}_{\setset{Q}}$. Plus, by (iv), $\set{C}$ will be a subset of $\set{N}_{\setset{Q}}$. Therefore, we have 
$\set{A}\cup\set{B}\cup\set{C}\subseteq\set{X}_{\setset{Q}}$.

\subsection{Detailed Description and Example for Phase 2}
\label{appendix: example phase2}

Our search starts with $k=1$ and an input graph $\graphp$ (could be empty) over observed variables. Every time we successfully found rank deficiency with $\texttt{rank}=k$, we update the graph and reset $k$ to $1$. On the other hand, if no rank deficiency can be found with current $k$, we add $k$ by $1$ (line 5, Alg~\ref{alg:phase2}),
as we want to ensure all $k'$ clusters, $k'<k$, have been found
when searching for $k$, as in Theorem~\ref{theorem:unique_rank_of_atomic_cover}.


During the search procedure, we maintain an active set $\setset{S}$, which is a set of covers and is initialized as the set of observed covers $\{\{\node{X_1}\},...,\{\node{X_n}\}\}$ (line 2,  Alg.~\ref{alg:phase2}).
We test the rank deficiency over different combinations of $\setset{C}$ and $\setset{X}$, drawn from $\setset{S'}$, where $\setset{S'}$ is generated from $\setset{S}$ by unfolding some existing clusters (lines 10-11, Alg.~\ref{alg:phase1}).
Introducing $\setset{S}$ and $\setset{S'}$ has several merits: 
(i) When we found a new atomic cover, we want to explore its relation with existing ones. 
This can be achieved by adding the newly found atomic cover to the active set $\setset{S}$ (illustrated in Example~\ref{example:phase2}).
(ii) According to Theorem~\ref{theorem:unique_rank_of_atomic_cover}, we do not want any descendants of $\setset{C}$ to be on the right side of the cross-covariance matrix when testing the rank. 
This can be achieved by removing all the children
of a newly found atomic cover from the active set $\setset{S}$.
%
Taking  (i) and (ii) together, we update the active set by $\setset{S} \gets (\setset{S} \backslash \setset{C}) \cup \mathbf{P}$ (line 24, Alg.~\ref{alg:phase2}).
(iii) We unfold $\setset{S}$ to get $\setset{S'}$,
and draw combinations $\setset{C}$ and $\setset{X}$
from $\setset{S'}$ instead of $\setset{S}$.
This allows the children of existing clusters to 
re-appear in $\setset{C}$ and $\setset{X}$,
to facilitate finding new clusters that share parents with existing ones, which will be discussed in detail later.

In addition, to establish a unique connection between rank deficiency and atomic covers, we need to avoid colliders and their descendants to appear in
$\setset{S'}$, as the existence of colliders in $\setset{C}$ or $\setset{N}$
($\setset{N} \gets \setset{S'}\backslash (\setset{X} \cup \setset{C})$) might induce rank deficiency that does not indicate a correct cluster (example in Appx.~\ref{example:colliderinc}).
%
To this end, we take two steps:
(i) Every time we found rank deficiency, we further check whether there is a collider in $\setset{C}$ (by the \textit{NoCollider} function described in Alg.~
\ref{alg:checkcollider}),
i.e., check whether there exists $\set{O}\in\setset{C}$ s.t., $\set{O}$ is a collider of $\setset{C}\backslash\{\set{O}\}$ and $\setset{N}$ (line 5 in Alg.~\ref{alg:checkcollider}).
If there is a collider, then we ignore the corresponding combination of $\setset{C}$  and $\setset{X}$ (line 17 in Alg.~\ref{alg:phase2}).
(ii) We perform unfolding on $\setset{S}$ to get $\setset{S'}$.
That is, every time we consider $\setset{T}$, a subset of $\setset{S}$, and get $\setset{S'}\gets(\setset{S} \backslash \setset{T}) \cup (\cup_{\set{T} \in \setset{T}} \purechildrenp(\set{T}))$ (lines 10-11 in Alg.~\ref{alg:phase2}).
This allows us to reconsider the pure children of existing atomic covers when choosing combinations of $\setset{C}$ and $\setset{X}$, and thus,  colliders can be identified (illustrated in Example~\ref{example:phase2}).
Taking these two steps together,
under the Condition~\ref{cond:vstructure},
it can be guaranteed that our search procedure will not be affected by the existence of colliders (proof in Appx.~\ref{proof:identifiability}).

\begin{example} [Example for Phase 2]
\label{example:phase2}
Consider the graph in Figure~\ref{fig:example2}.
  We start with finding atomic covers with $k=1$,
  and we can find that $\{\node{X_3}\}$ is a parent of $\{\node{X_8}\}$,
  as in Figure~\ref{fig:example2}(a).
  At this point, no more $k=1$ clusters can be found, so next we search for $k=2$ clusters.
  Then to identify collider $\{\node{X_7}\}$, we only need to consider $\{\node{X_7}\}$ and $\{\node{X_8}\}$ together as the children of $\{\node{X_2},\node{X_3}\}$.
  %
  %
  After finding such a relationship, we arrive at Figure~\ref{fig:example2}(b),
  and from now on the collider $\{\node{X_7}\}$ will not induce unfavorable rank deficiency anymore (as it is recorded).
  The next step is to find the relation of  $\{\node{X_4}\}$,$\{\node{X_5}\}$,$\{\node{X_6}\}$ with $\{\node{L_2},\node{X_2}\}$, 
   by
  taking $\setset{X}=\{\{\node{X_2}\}\}$ and $\setset{C}=\{\{\node{X_4}\},\{\node{X_5}\}\}$
  or $\{\{\node{X_4}\},\{\node{X_6}\}\}$ or $\{\{\node{X_5}\},\{\node{X_6}\}\}$,
  and thus conclude $\{\{\node{X_4}\},\{\node{X_5}\},\{\node{X_6}\}\}$ as the pure children of $\{\node{L_2},\node{X_2}\}$,
   as in Fig~\ref{fig:example2}(c).
   Finally we 
    are able to find the relationship of $\{\node{X_1}\}$,$\{\node{L_2},\node{X_2}\}$, $\{\node{X_3}\}$ with $\{\node{L_1}\}$,
  by taking $\setset{X}=\{\}$ and $\setset{C}$ as $\{\{\node{X_1}\},\{\node{L_2},\node{X_2}\}\}$,
  $\{\{\node{X_1}\},\{\node{X_3}\}\}$, or $\{\{\node{L_2},\node{X_2}\},\{\node{X_3}\}\}$, as in Figure~\ref{fig:example2}(d).
\end{example}

We here give a more detailed example to show the procedure of Phase 2,
with the underlying graph $\graph$ showed in Figure~\ref{fig:example2}(d).

\textbf{Step 1.}  Initialize active set $\setset{S}$ as $\{\{\node{X_1}\},...,\{\node{X_8}\}\}$, $k$ as $1$.

\textbf{Step 2.} Get $\setset{S'}$ by unfolding $\setset{S}$. Currently, 
$\setset{S'}=\{\{\node{X_1}\},...,\{\node{X_8}\}\}$.
Now let $k=1$ and $t=1$. 
 Draw a set of $t$ observed covers $\setset{X} \subset \mathcal{S'}\cap \setset{X}_\graph$, and  draw a set of covers $\setset{C} \subset \setset{S'}\backslash \setset{X}$, s.t., $||\setset{C}||=k-t+1$,
 and check whether rank deficiency holds.
 We will find that when $\setset{X}=\{\{\node{X}_3\}\}$ and $\setset{C}=\{\{\node{X}_8\}\}$, rank deficiency holds and there is no collider detected.
 Therefore  we draw a link from $\node{X}_3$ to $\node{X}_8$ in $\graphp$, as shown in Figure~\ref{fig:example2}(a).
 Now update the active set $\setset{S}$ as $\{\{\node{X_1}\},...,\{\node{X_7}\}\}$.

 \textbf{Step 3.} 
 Continue searching with $k=1$, but no more rank deficiency can be found.
 Therefore, we add $k$ by 1.

 \textbf{Step 4.} 
 Unfold  $\setset{S}$ and  get 
$\setset{S'}=\{\{\node{X_1}\},...,\{\node{X_8}\}\}$.
Now, $k=2$ and $t=2$.
 By drawing $\setset{X}$ and $\setset{C}$,
 we will find that when $\setset{X}=\{\{\node{X}_2\},\{\node{X}_3\}\}$ and $\setset{C}=\{\{\node{X}_7\}\}$, rank deficiency holds and there is no collider detected.
Therefore, we draw links from $\node{X}_2\node{X}_3$ to $\node{X}_7$ in $\graphp$, as shown in Figure~\ref{fig:example2}(b).
 Now update the active set $\setset{S}$ as $\{\{\node{X_1}\},...,\{\node{X_6}\}\}$.

\textbf{Step 5.}
 Reset $k=1$ and search for rank deficiency. No more rank deficiency can be found with $k=1$, and thus we add $k$ by 1.

 \textbf{Step 6.}
 Get $\setset{S'}$ by unfolding $\setset{S}$, and  
$\setset{S'}=\{\{\node{X_1}\},...,\{\node{X_6}\}\}$.
When $k=2$ and $t=2$, no more rank deficiency can be found.
Therefore, we try $k=2$ and $t=1$.
 By drawing $\setset{X}$ and $\setset{C}$,
 we will find that when $\setset{X}=\{\{\node{X}_2\}\}$ and $\setset{C}=\{\{\node{X}_4\},\{\node{X}_5\}\}$ or $\{\{\node{X}_4\},\{\node{X}_6\}\}$ or $\{\{\node{X}_5\},\{\node{X}_6\}\}$, rank deficiency holds and there is no collider detected.
 Therefore, we conclude that there is an atomic cover. 
 As $k=2$ but $||\setset{X}||=1$, we need one additional latent variable to explain this atomic cover.
 Thus, we add a new node $\node{L}_2$ to $\graphp$ (the subscript index for $\node{L}$ can be rather arbitrary as long as it is not the same as an existing one), and draw links from $\node{L}_2\node{X}_2$ to $\node{X}_4$$\node{X}_5$$\node{X}_6$ in $\graphp$, as shown in Figure~\ref{fig:example2}(c).
 Now, update the active set $\setset{S}$ as $\{\{\node{X_1}\},\{\node{X_2}\},\{\node{X_3}\},\{\node{L}_2,\node{X}_2\}\}$.

 \textbf{Step 7.} Reset $k=1$ and search for rank deficiency. 
 Unfold $\setset{S}$ and get  $\setset{S'}=\{\{\node{X_1}\},\{\node{X_2}\},\{\node{X_3}\},\{\node{L}_2,\node{X}_2\}\}$.
When $k=1$ and $t=1$, no more rank deficiency can be found.
Therefore we try $k=1$ and $t=0$.
 By drawing $\setset{X}$ and $\setset{C}$,
 we will find that when $\setset{X}=\{\}$ and $\setset{C}=\{\{\node{X}_1\},\{\node{X}_2\}\}$ or $\{\{\node{X}_2\},\{\node{X}_3\}\}$ or $\{\{\node{X}_1\},\{\node{X}_3\}\}$ or $\{\{\node{L}_2,\node{X}_2\}\}$, rank deficiency holds and there is no collider detected.
  Therefore, we conclude that there is an atomic cover. 
  All the possible $\setset{C}$ will be merged.
 As $k=1$ but $||\setset{X}||=0$, we need one additional latent variable to explain this atomic cover.
  Thus, we add a new node $\node{L}_1$ to $\graphp$, and draw links from $\node{L}_1$ to $\node{X}_1$$\node{L}_2$$\node{X}_2$$\node{X}_3$ in $\graphp$, as shown in Figure~\ref{fig:example2}(d),
  and update the active set $\setset{S}$ as $\{\{\node{L}_1\}\}$.

\textbf{Step 8.} From now on, no more rank deficiency can be found, and  when $k$ is sufficiently large the procedure ends.
Output $\graphp$, as in Figure~\ref{fig:example2}(d).

\begin{figure}[t]
  \vspace{-0mm}
  \centering
  \begin{subfigure}[b]{0.45\textwidth}
    \centering
    \includegraphics[width=\textwidth]{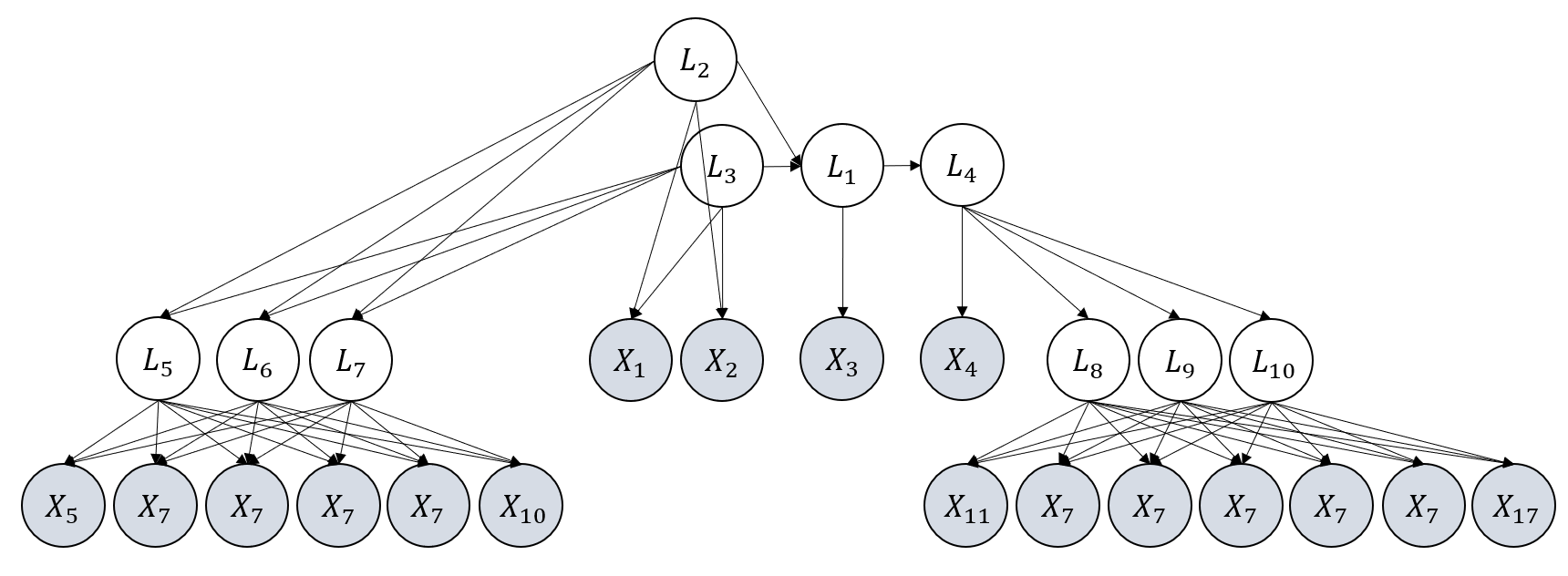}
    \caption{The ground truth graph $\graph$.\color{white}{place holder place holder}}
  \end{subfigure}
  \begin{subfigure}[b]{0.45\textwidth}
    \centering
    \includegraphics[width=\textwidth]{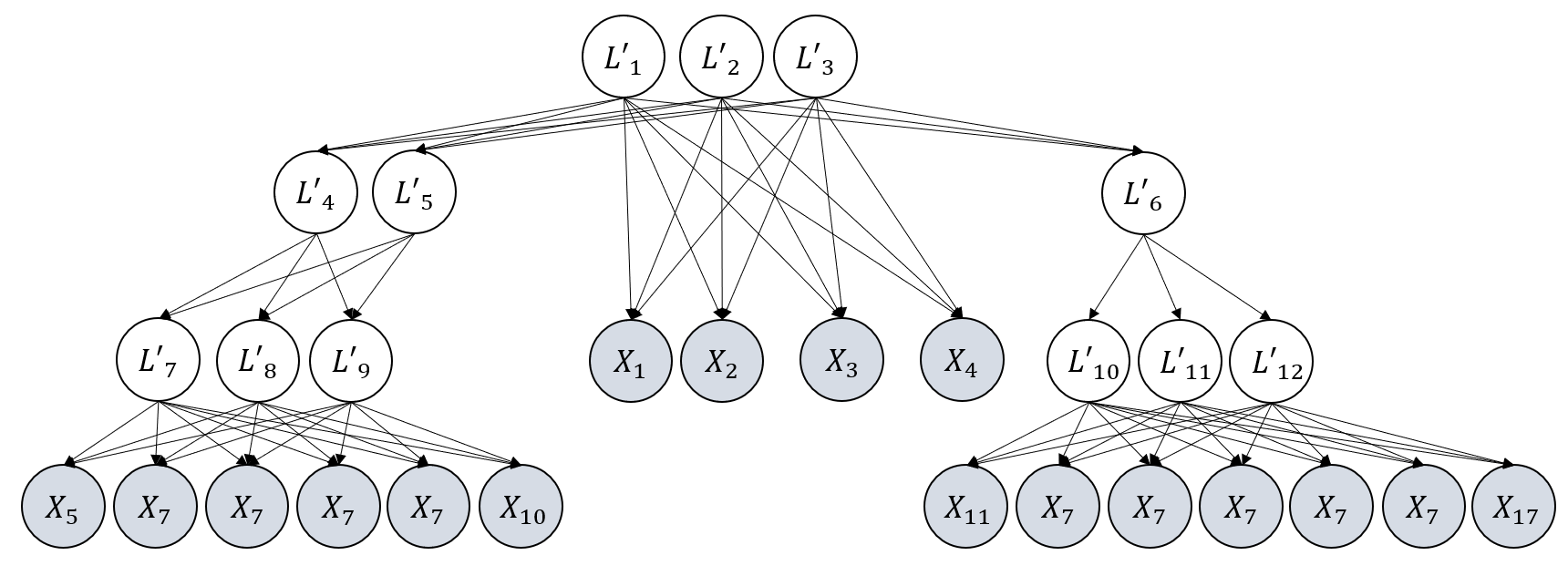}
    \caption{Algorithm output after phase 2, taken as input of RefineCausalClusters.}
  \end{subfigure}
  \begin{subfigure}[b]{0.45\textwidth}
    \centering
    \includegraphics[width=\textwidth]{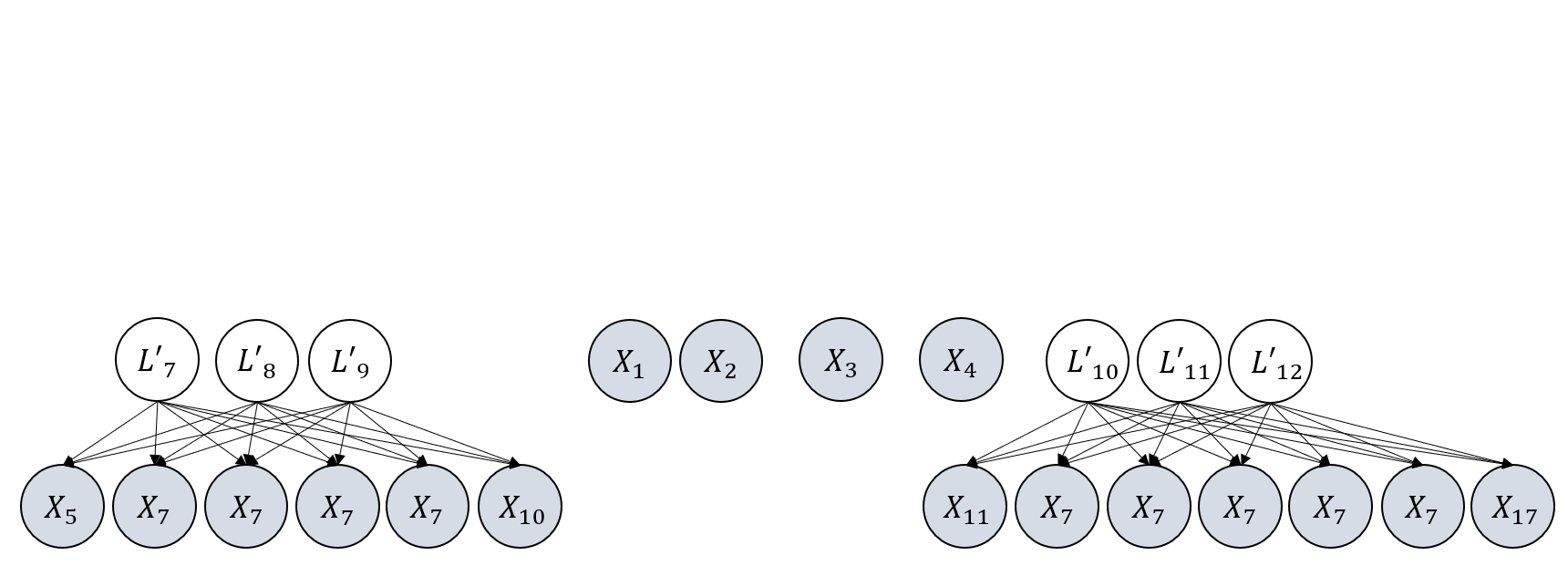}
    \caption{Remove $\{\node{L_1'},\node{L_2'},\node{L_3'}\}$ and its neighbours that are latent. Then perform FindCausalClusters.}
  \end{subfigure}
  \begin{subfigure}[b]{0.45\textwidth}
    \centering
    \includegraphics[width=\textwidth]{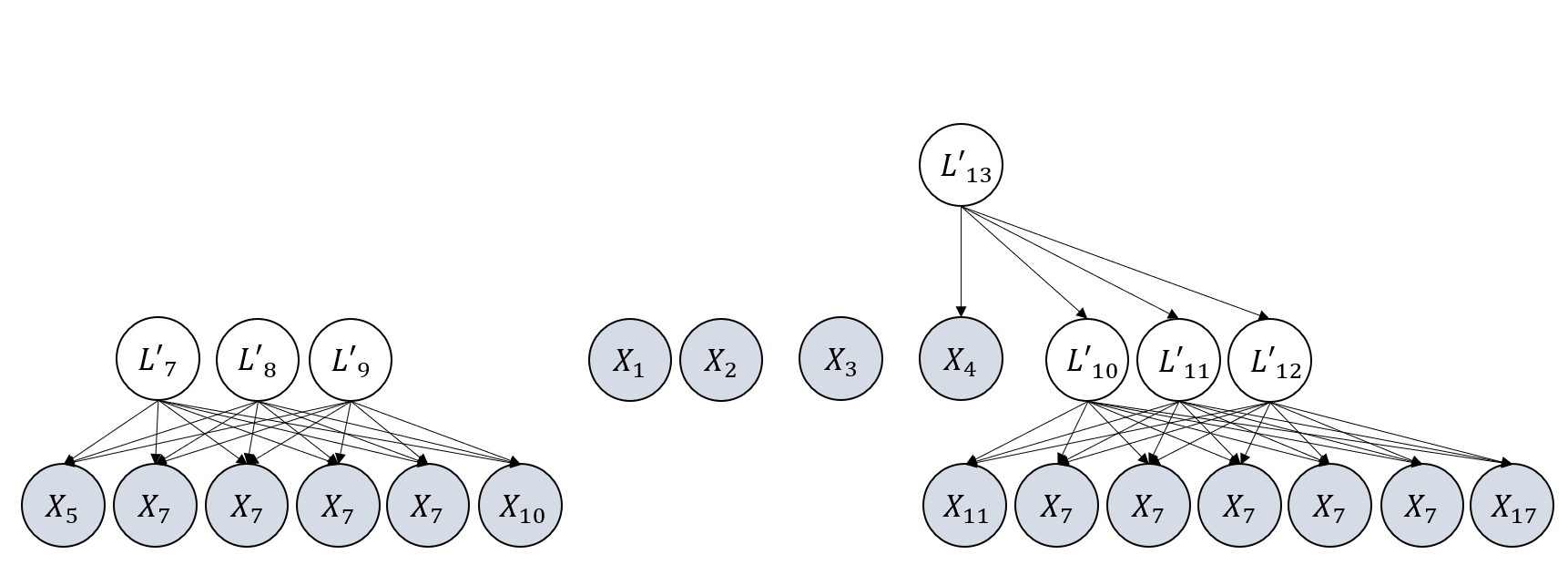}
    \caption{During FindCausalClusters performed at (c) we first  find $\{\node{L_{13}'}\}$.\color{white}{place holder place holder}}
  \end{subfigure}
  \begin{subfigure}[b]{0.45\textwidth}
    \centering
    \includegraphics[width=\textwidth]{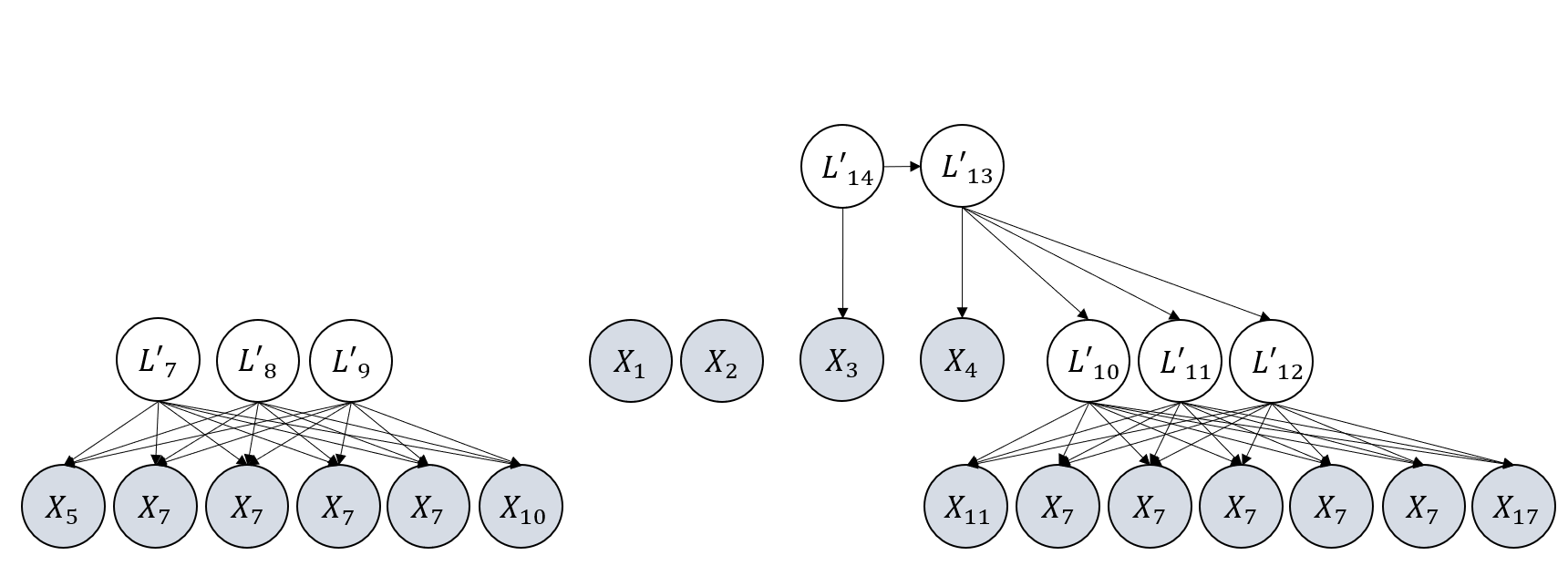}
    \caption{During FindCausalClusters performed at (c), we then find $\{\node{L_{14}'}\}$.\color{white}{place holder place holder}}
  \end{subfigure}
  \begin{subfigure}[b]{0.48\textwidth}
    \centering
    \includegraphics[width=\textwidth]{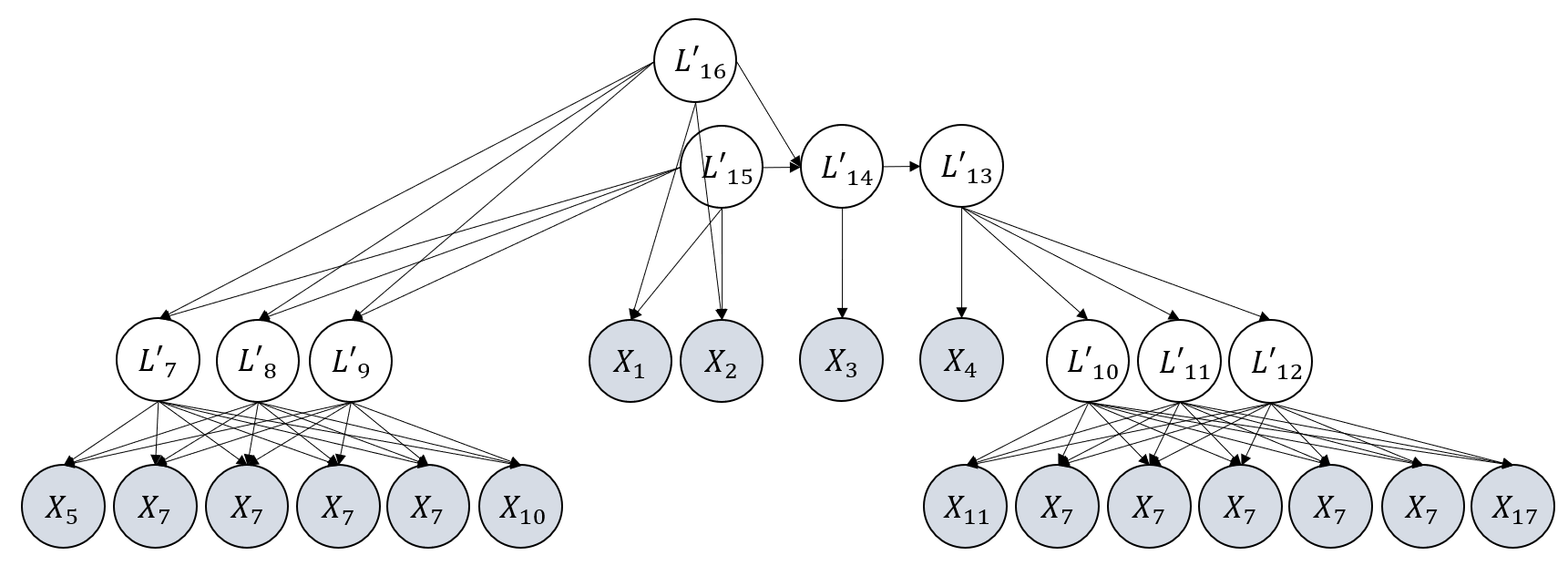}
    \caption{During FindCausalClusters performed at (c), we finally find $\{\node{L_{15}'},\node{L_{16}'}\}$.\color{white}{place holder place holder}}
  \end{subfigure}
  \caption{Example in subfigure (a) is the real graph $\graph$. 
  After phase 2 the output graph still contains a fake cluster, as in (b).
   After phase 3, the output graph will be correct, as in (f).}
  \label{fig:example3 for phase2}
\end{figure}

\subsection{Example for Phase 3}
\label{appendix: example phase3}
Here, we give an example (see Figure~\ref{fig:example3 for phase2}) where Phase 2 may result in incorrect latent covers, and thus we need Phase 3 to characterize and refine these incorrect latent covers.
Specifically, as in Figure~\ref{fig:example3 for phase2}),
when we look for $k=3$ clusters, none of the atomic  covers $\{\node{L_1}\}$,$\{\node{L_4}\}$,$\{\node{L_2},\node{L_3}\}$ has been discovered.
Therefore, in Phase 2, when looking for $k=3$ clusters,
we will find a combination of  $\setset{C}=\{\{\node{X_1}\},\{\node{X_2}\},\{\node{X_3}\},\{\node{X_4}\}\}$ and $\setset{X}=\{\}$ that causes rank deficiency, and thus we will mistakenly create an atomic cover $\{\node{L'_1},\node{L'_2},\node{L'_3}\}$
with their pure children $\{\node{X_1}\},\{\node{X_2}\},\{\node{X_3}\},\{\node{X_4}\}$, as in  Figure~\ref{fig:example3 for phase2}(b).

Fortunately, this incorrect cluster will not affect the identification of other clusters in the graph: e.g., in Figure~\ref{fig:example3 for phase2}(b),
the covers  $\{\node{L'_7},\node{L'_8},\node{L'_9}\}$,$\{\node{L'_{10}},\node{L'_{11}},\node{L'_{12}}\}$ are correctly found,
except that the neighbors of the wrong atomic cover $\{\node{L'_1},\node{L'_2},\node{L'_3}\}$ could be incorrect.
This allows us to take a further look into the incorrect cluster and refine it based on Theorem~\ref{theorem:phase3} (the proof of which is in Appendix~\ref{proof:phase3}).

As shown in Figure~\ref{fig:example3 for phase2},
the subfigure (a) is the underlying graph $\graph$. After phase 2 the output graph $\graphp$ in (b) contains incorrect cover $\set{V}=\{\node{L'_1},\node{L'_2},\node{L'_3}\}$.
In (c), we first calculate $\hat{\graph}$,
 which is got by deleting $\set{V}$, all neighbours of $\set{V}$ that are latent,
 and all relating edges of them from $\graphp$.
 The resulting $\hat{\graph}$  is shown in Figure~\ref{fig:example3 for phase2} (c).
 After that,
 we perform $\text{FindCausalClusters}(\hat{\graph}, \set{X})$,
 and then the clusters 
 $\{\node{X_1},\node{X_2}\}$,$\{\node{X_3}\}$,and $\{\node{X_4}\}$ can be correctly found,
 as shown in Figure~\ref{fig:example3 for phase2} (d)(e)(f).

\subsection{Graph Examples with Variables all Observed}
\label{sec apendix: graphs for observed only}
Please refer to Figure~\ref{fig:example all observed}.

\subsection{Graph Examples for Latent Tree Models}
\label{sec apendix: graphs for latent tree}
Please refer to Figure~\ref{fig:example latent tree}.

\subsection{Graph Examples for Latent Measurement Models}
\label{sec apendix: graphs for latent mm}
Please refer to Figure~\ref{fig:example latent measurement}.

\subsection{Graph Examples for General Latent Models}
\label{sec apendix: graphs for latent general}
Please refer to Figure~\ref{fig:example latent general}.

\begin{figure}[t]
  \vspace{-0mm}
  \centering
  \begin{subfigure}[b]{0.3\textwidth}
    \centering
    \includegraphics[width=\textwidth]{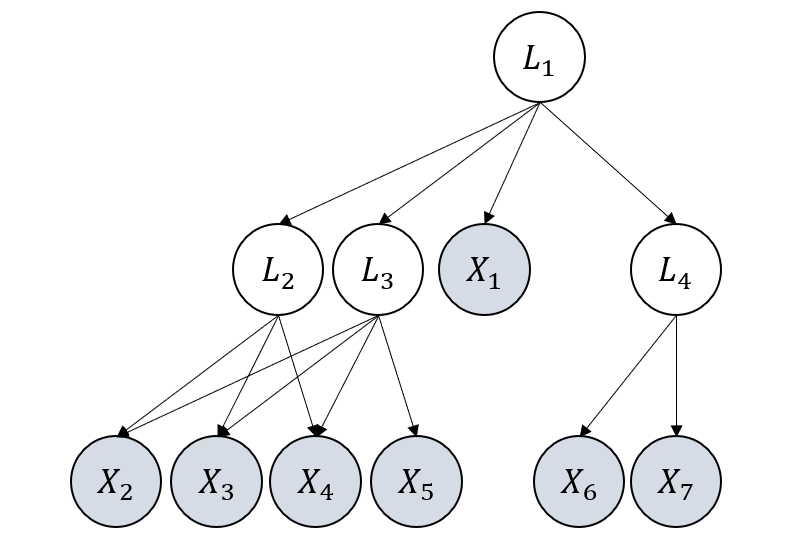}
    \caption{The original graph $\graph$.}
  \end{subfigure}
  \begin{subfigure}[b]{0.3\textwidth}
    \centering
    \includegraphics[width=\textwidth]{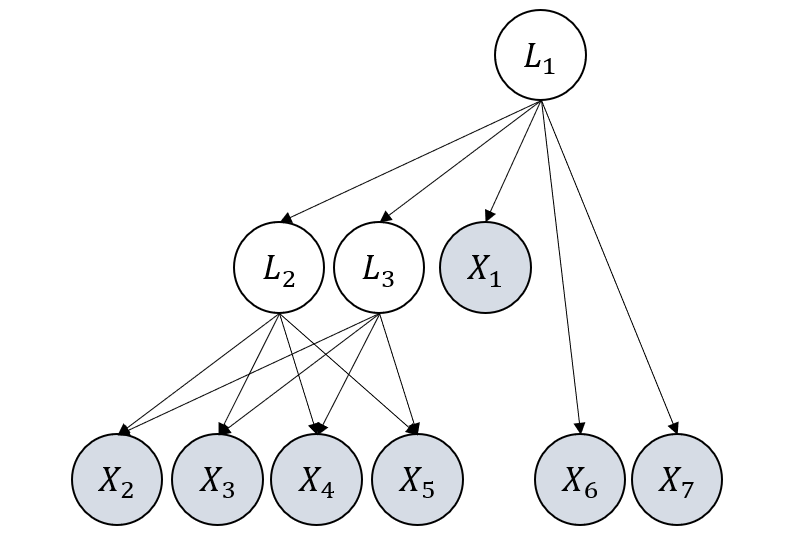}
    \caption{After operators $\mathcal{O}_{\text{min}}$ and $\mathcal{O}_s$.}
  \end{subfigure}
  \caption{Example to show graph operators $\mathcal{O}_{\text{min}}$ and $\mathcal{O}_s$.}
  \label{fig:example for operator}
\end{figure}

\begin{figure}[t]
  \vspace{-0mm}
  \begin{subfigure}[b]{0.32\textwidth}
    \includegraphics[width=\textwidth]{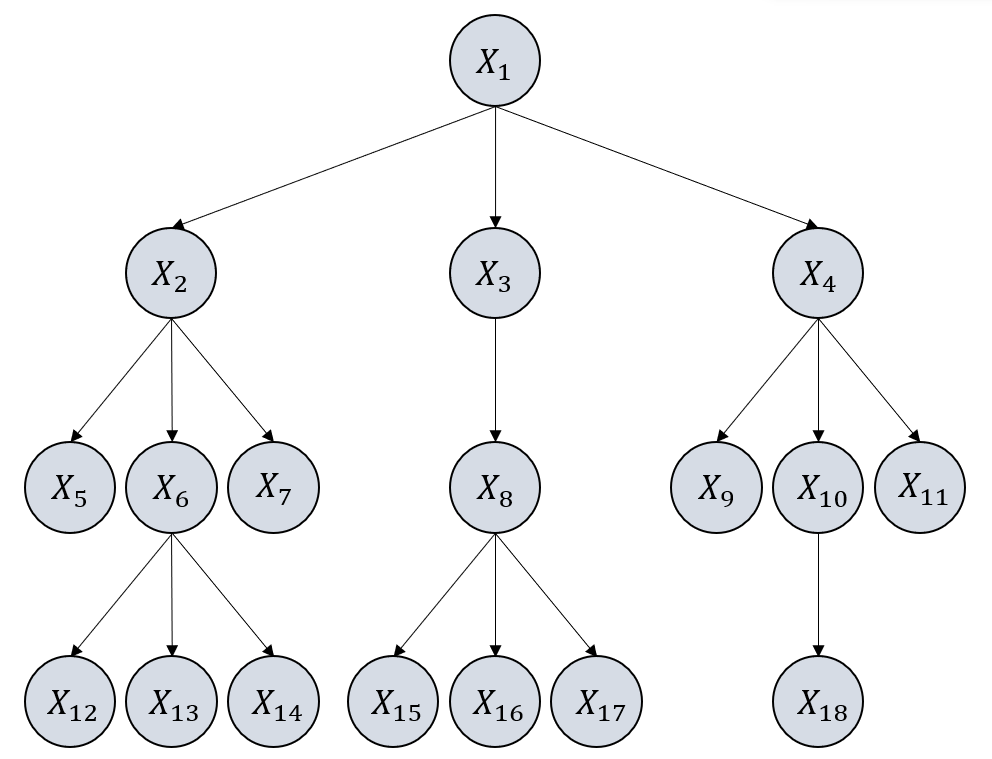}
    \caption{Example 1.}
\end{subfigure}
\hfill
\begin{subfigure}[b]{0.32\textwidth}
  \includegraphics[width=\textwidth]{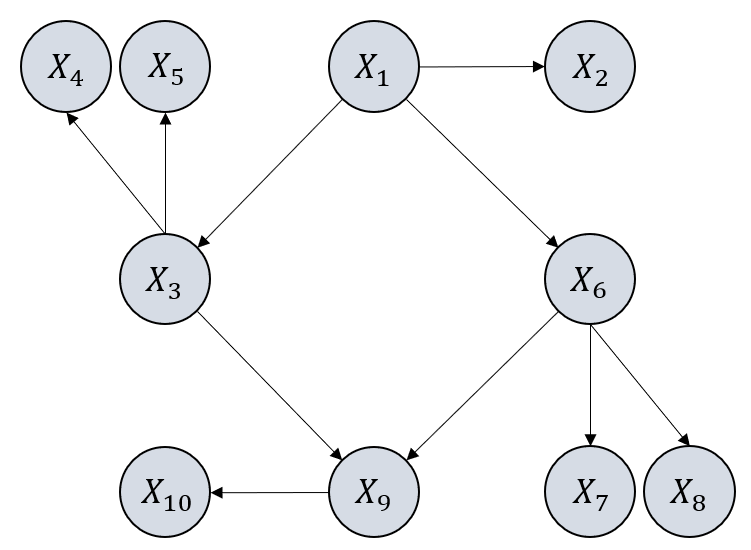}
  \caption{Example 2.}
\end{subfigure}
\hfill
\begin{subfigure}[b]{0.32\textwidth}
  \includegraphics[width=\textwidth]{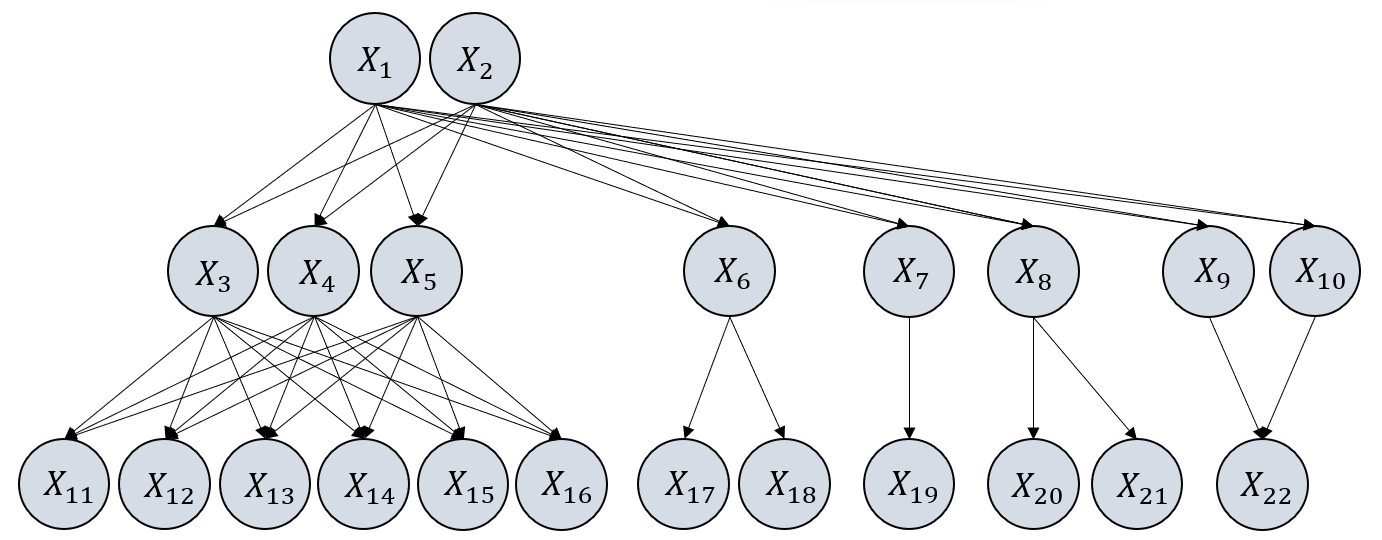}
  \caption{Example 3.}
\end{subfigure}
\caption{Examples of graphs that have only observed variables.}
  \label{fig:example all observed}
  \vspace{-0mm}
\end{figure}

\subsection{Illustrative Example of Considering Colliders in Phase 2}
\label{example:colliderinc}
For example, in Figure~\ref{fig:example for checking colliders in N},
suppose that we have already found the cover $\node{L}_1$
as the parent of cluster $\node{X}_1\node{X}_2$,
and 
$\node{L}_2$
as the parent of cluster $\node{X}_6\node{X}_7$.
Next, we search for $k=2$ clusters and take 
$\setset{C}=\{\{\node{X}_3\},\{\node{X}_4\},\{\node{L}_2\}\}$ and $\setset{X}=\{\}$,
and then we have rank deficiency $\texttt{rank}(\Sigma_{\setset{C}\cup\setset{X},\setset{N}\cup\setset{X}}) = 2$.
However, this rank deficiency does not imply a correct cluster as
there is a set of collider $\{\node{L}_2\}$ inside $\setset{C}$.
Fortunately, it can be detected by Algorithm~\ref{alg:checkcollider}.
Specifically, if we take $\setset{C'}=\{\{\node{X}_3\},\{\node{X}_4\}\}$,
we can find that   $\texttt{rank}(\Sigma_{\setset{C'}\cup\setset{X},\setset{N}\cup\setset{X}}) = \texttt{rank}(\Sigma_{\node{X}_3\node{X}_4,\node{X}_1\node{X}_2\node{X}_5\node{X}_6})=1$ (line 5 in Algorithm~\ref{alg:checkcollider}),
 which means there exists a smaller group of rank deficiency caused by removing the collider in $\setset{C}$.
 Thus, we conclude that $\setset{C}=\{\{\node{X}_3\},\{\node{X}_4\},\{\node{L}_2\}\}$ and $\setset{X}=\{\}$
 is not a correct combination and will not consider them for forming a cluster (as in line 1 in Algorithm~\ref{alg:phase1}).


\subsection{Examples for Graph Operators}
\label{examples:operator}

Suppose a graph $\graph$ of a latent linear model in  Figure~\ref{fig:example for operator}(a) is $\graph$.
After applying 
$\mathcal{O}_{\text{min}}(\mathcal{O}_s(\graph))$,
we have the graph in Figure~\ref{fig:example for operator}(b).
Specifically,
 the $\mathcal{O}_s$ operator 
 adds an edge from $\node{L_2}$ to $\node{X_5}$ and the $\mathcal{O}_{\text{min}}$
operator delete $\node{L_4}$
and add an edge from  $\node{L_1}$ directly to $\node{X_6}$ and $\node{X_7}$.
For  $\graph$, 
such 
 two operators will not change the rank in the infinite sample case.

\subsection{Graphical Relations between Covers and Set of Covers}
\label{appendix: relations between covers}
The relation between covers naturally follows the relation between a set of variables. For example,
in Figure~\ref{fig:example3 for phase2}(a),
the pure children of $\{\node{L_4},\node{L_5}\}$ is 
$\{\node{X_5},\node{X_6},\node{X_7},\node{X_8}\}$.
For the relation between  sets of covers,
it also follows the relationship between variables.
E.g., in Figure~\ref{fig:example3 for phase2}(a),
the parents of $\{\{\node{L_4}\},\{\node{L_5}\}\}$ is a set of nodes 
$\{\node{L_1}\}$.

\subsection{Discussions on Checking Colliders Completely}
\label{appendix:checkcollider}
With our search procedure that checks colliders in Algorithm~\ref{alg:checkcollider},
we can make sure that the existence of colliders between atomic covers in $\setset{C}$ will not induce incorrect clustering results.
However,
we note that there are still chances that colliders are in $\setset{N}$.
If Condition~\ref{cond:vstructure} holds,
then we can make sure that the existence of colliders in $\setset{N}$
will not induce fake clusters.
In fact, there is a way to further check whether 
there exist colliders in $\setset{N}$.
Specifically,
in the line $17$ of Algorithm~\ref{alg:phase1},
if we have
$\texttt{rank}(\Sigma_{\setset{C}\cup\setset{X},\setset{N}\cup\setset{X}}) = k$, and NoCollider($\setset{C}$, $\setset{X}$, $\setset{N})$ returns True, we can further check whether there exist
a set of covers $\setset{N'} \subseteq \setset{N}$ such that 
$\setset{N'}$ consists of all the colliders between $\setset{C}$ and $\setset{N}\backslash \setset{N'}$.
To this end, we just enumerate all the possible subsets $\setset{N'}$ of $\setset{N}$.
If $\setset{N'}$ is the set of all the colliders,
then it must be that (i) 
$\texttt{rank}(\Sigma_{\setset{C}\cup\setset{X},(\setset{N}\backslash\setset{N'})\cup\setset{X}}) = k'<k$,
and (ii)
$\texttt{rank}(\Sigma_{\setset{C}\cup\setset{X}\cup\setset{N'},(\setset{N})\cup\setset{X}})> k'+||\setset{N'}||$.

Take Figure~\ref{fig:example for checking colliders in N} as an example. First, we check whether Condition~\ref{cond:vstructure} holds.
As $|\set{C}| + |\set{A}| =|\{\node{L}_1\}| + |\{\node{L}_2\}| = 2 < |\set{V_1}|+|\set{V_2}| = |\{\node{X}_3,\node{X}_4\}| + |\{\node{X}_5,\node{X}_6\}|=4$,
Condition~\ref{cond:vstructure} does not hold.
Therefore, when checking $k=2$,
if we take $\setset{X}=\{\}$,  $\setset{C}=\{\{\node{X}_3\},\{\node{X}_4\},\{\node{X}_5\}\}$, and
$\setset{N}=\{\{\node{X}_1\},\{\node{X}_2\},\{\node{X}_6\},\{\node{X}_7\},\{\node{X}_8\}\}$ in line 17 of Algorithm~\ref{alg:phase1},
we will find that 
$\texttt{rank}(\Sigma_{\setset{C}\cup\setset{X},\setset{N}\cup\setset{X}}) = k=2$,
 which implies an incorrect cluster as the cardinality of parents of $\{\{\node{X}_3\},\{\node{X}_4\},\{\node{X}_5\}\}$ should be only 1.
Fortunately, in this scenario,
we can detect that $\setset{N'}=\{\{\node{X}_7\},\{\node{X}_8\}\} \subseteq \setset{N}$ is the set of all the colliders, by finding that (i) 
$\texttt{rank}(\Sigma_{\setset{C}\cup\setset{X},(\setset{N}\backslash\setset{N'})\cup\setset{X}}) = 1<k=2$,
and (ii)
$\texttt{rank}(\Sigma_{\setset{C}\cup\setset{X}\cup\setset{N'},(\setset{N})\cup\setset{X}})=4> 1+||\setset{N'}||=3$.

As mentioned in Section~\ref{sec:theory},
by adding this check function to our algorithm (specifically to line 17 in Algorithm~\ref{alg:phase1} before adding $\setset{C}$ to $\setsetset{D}$),
we can achieve better identifiability that relies on Condition~\ref{cond:basic} only. However, that additional checking function is computationally inefficient.

\begin{figure}[t]
  \vspace{-0mm}
  \centering
    \includegraphics[width=0.35\textwidth]{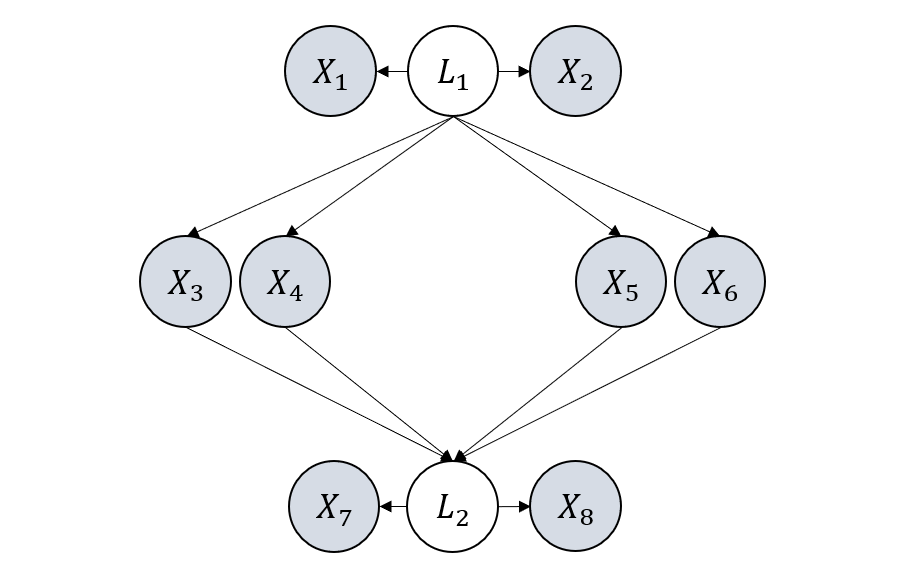}
    \caption{Example of  checking colliders in $\setset{N}$.}
  \label{fig:example for checking colliders in N}
\end{figure}

\begin{figure}[t]
  \vspace{-1mm}
  \begin{subfigure}[b]{0.24\textwidth}
    \includegraphics[width=\textwidth]{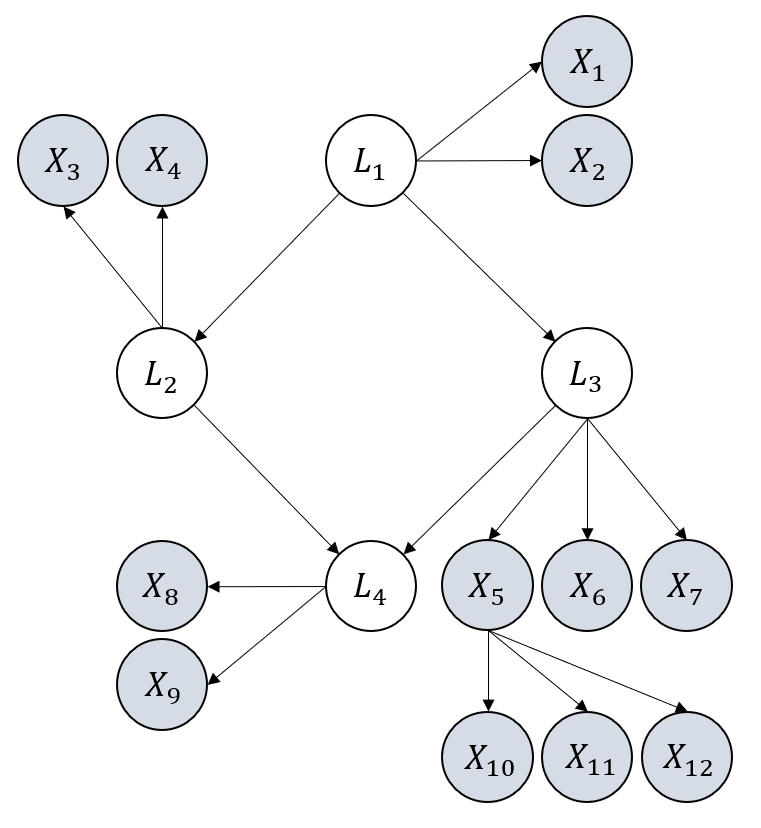}
    \caption{Example 1.}
\end{subfigure}
\hfill
\begin{subfigure}[b]{0.24\textwidth}
  \includegraphics[width=\textwidth]{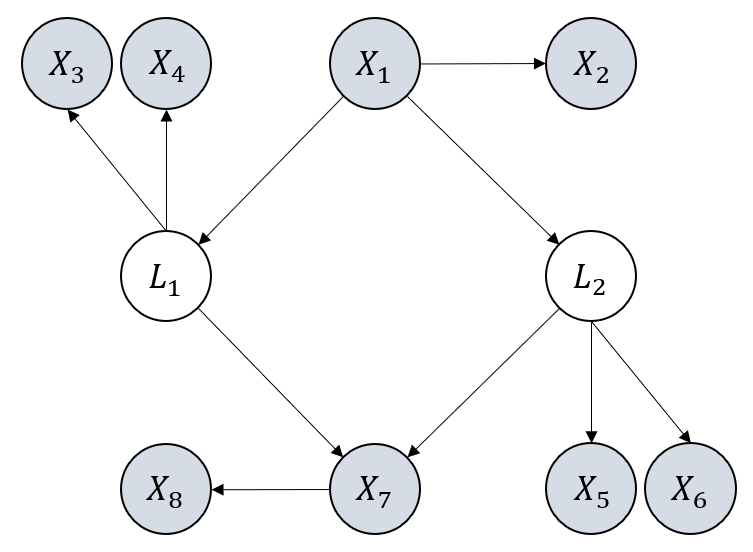}
  \caption{Example 2.}
\end{subfigure}
\hfill
\begin{subfigure}[b]{0.24\textwidth}
  \includegraphics[width=\textwidth]{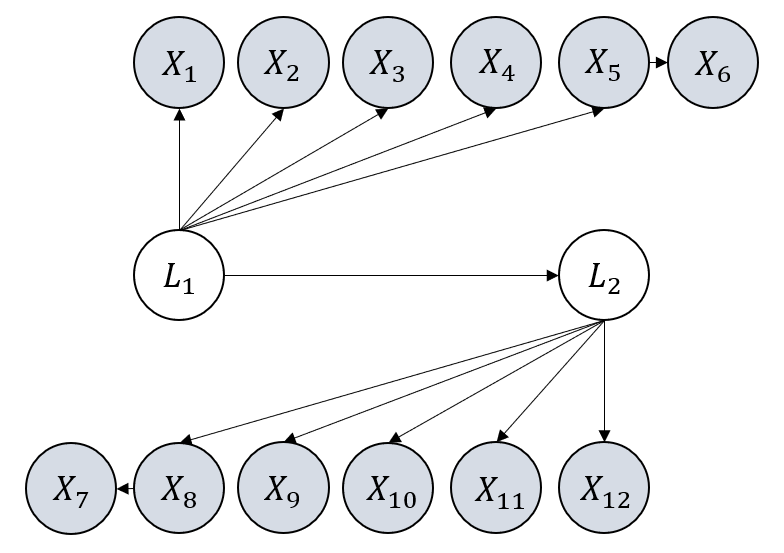}
  \caption{Example 3.}
\end{subfigure}
\hfill
\begin{subfigure}[b]{0.24\textwidth}
  \includegraphics[width=\textwidth]{figures/MM_latent4.png}
  \caption{Example 4.}
\end{subfigure}
\caption{Examples of  Latent Measurement Graphs.}
  \label{fig:example latent measurement}
  \vspace{-0mm}
\end{figure}

\begin{figure}[t]
  \vspace{-0mm}
  \begin{subfigure}[b]{0.245\textwidth}
    \includegraphics[width=\textwidth]{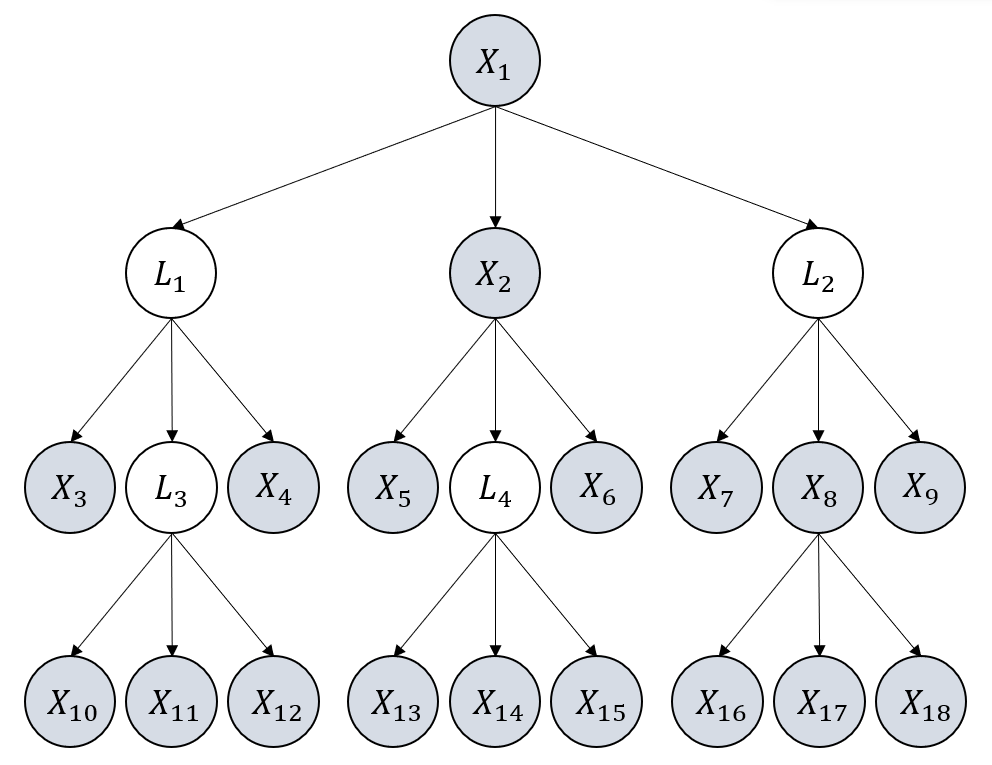}
    \caption{Example 1.}
\end{subfigure}
\hfill
\begin{subfigure}[b]{0.245\textwidth}
  \includegraphics[width=\textwidth]{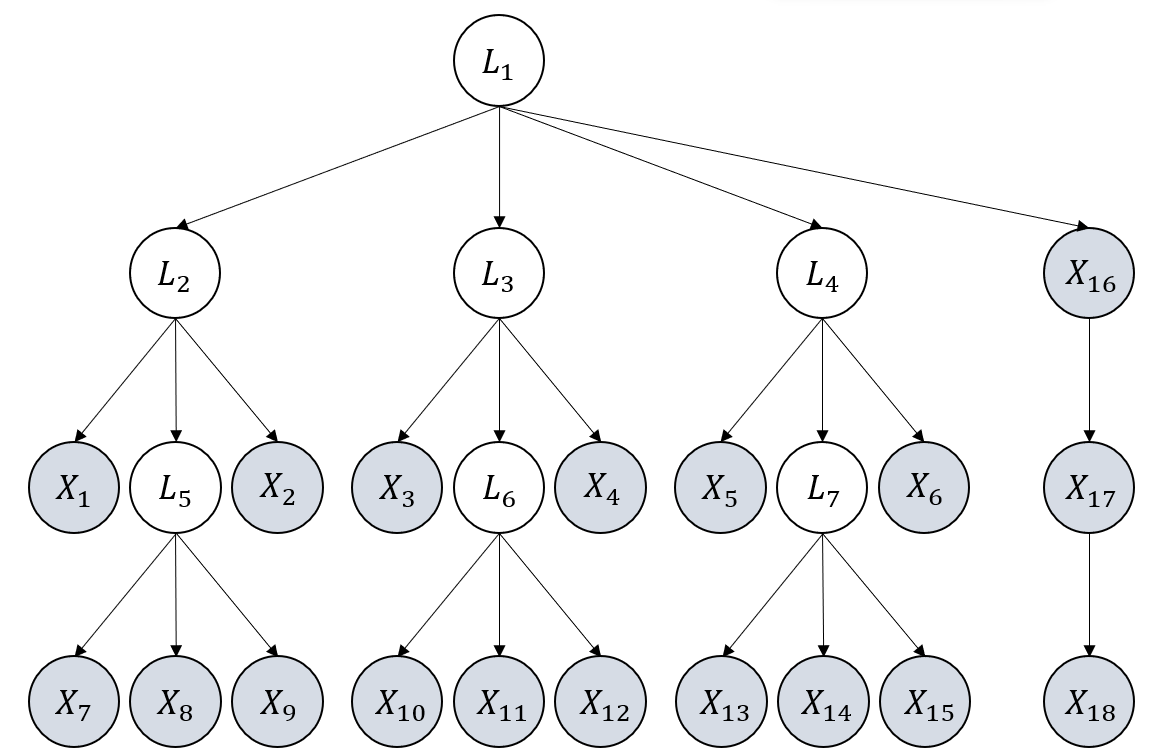}
  \caption{Example 2.}
\end{subfigure}
\hfill
\begin{subfigure}[b]{0.245\textwidth}
  \includegraphics[width=\textwidth]{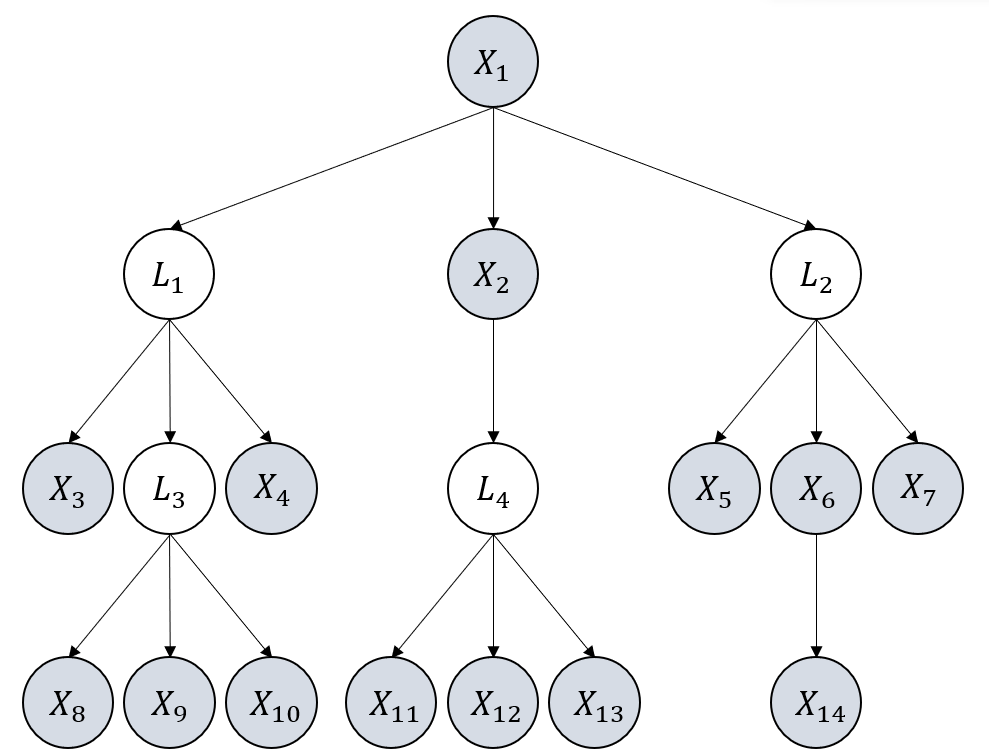}
  \caption{Example 3.}
\end{subfigure}
\hfill
\begin{subfigure}[b]{0.245\textwidth}
  \includegraphics[width=\textwidth]{figures/treelatent4.png}
  \caption{Example 4.}
\end{subfigure}
\caption{Examples of  Latent Tree Graphs.}
  \label{fig:example latent tree}
  \vspace{-0mm}
\end{figure}

\begin{figure}[t]
  \vspace{-0mm}
  \begin{subfigure}[b]{0.5\textwidth}
    \includegraphics[width=\textwidth]{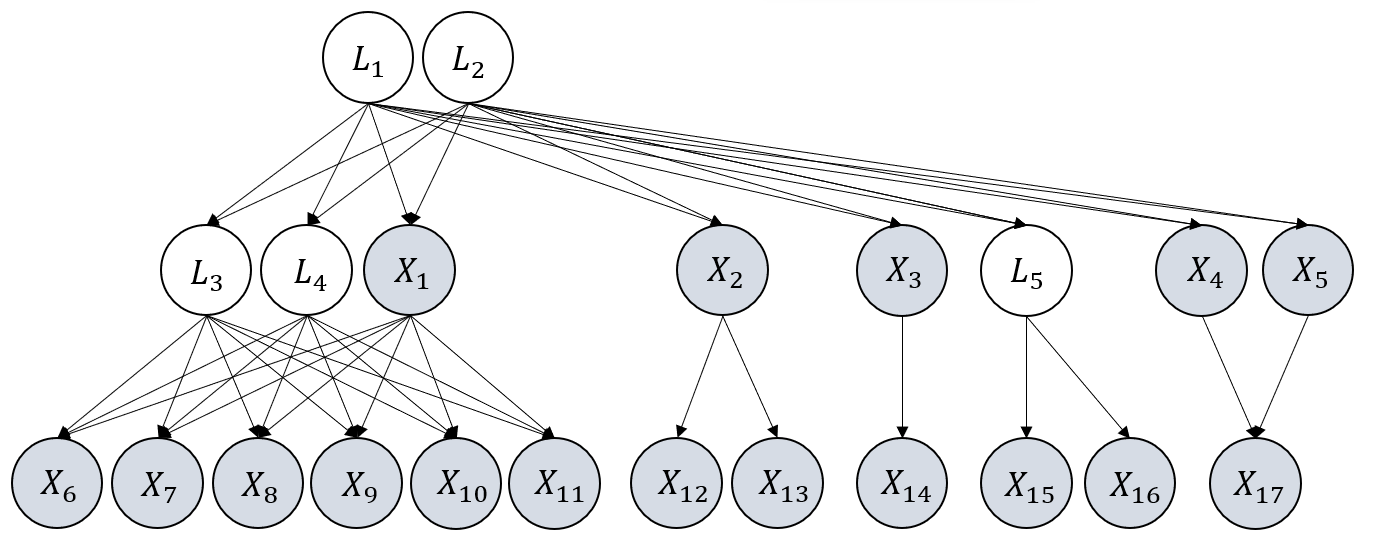}
    \caption{Example 1.}
\end{subfigure}
\hfill
\begin{subfigure}[b]{0.5\textwidth}
  \includegraphics[width=\textwidth]{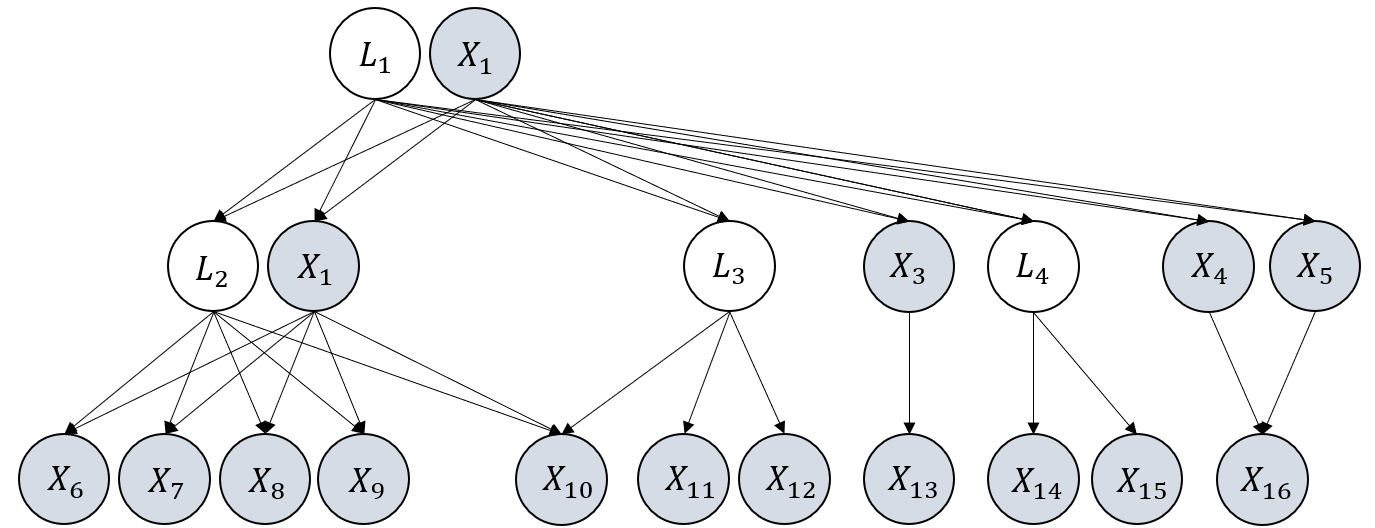}
  \caption{Example 2.}
\end{subfigure}
\hfill
\begin{subfigure}[b]{0.5\textwidth}
  \includegraphics[width=\textwidth]{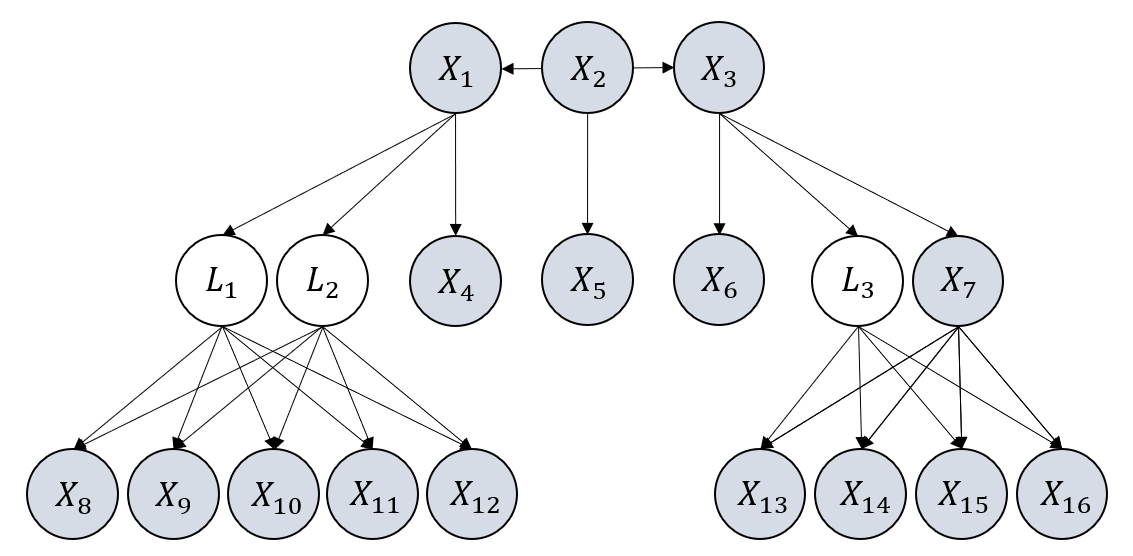}
  \caption{Example 3.}
\end{subfigure}
\hfill
\begin{subfigure}[b]{0.5\textwidth}
  \includegraphics[width=\textwidth]{figures/general_latent4.png}
  \caption{Example 4.}
\end{subfigure}
\caption{Examples of  Latent General Graphs.}
  \label{fig:example latent general}
  \vspace{-0mm}
\end{figure}

\subsection{Evaluation Metric Details}
\label{appendix:comb and perm}

The definition of F1 is as follows.
$\text{F1}=\frac{2*\text{Recall}*\text{Precision}}{\text{Recall}+\text{Precision}}$, $\text{Recall}=\frac{\text{TP}}{\text{TP}+\text{FN}}$,
   and $\text{Precision}=\frac{\text{TP}}{\text{TP}+\text{FP}}$, where TP, FP, and FN denote True Positive, False Positive, and False Negative, respectively.

For a fair comparison, we need to align the latent variables in the output graph $\graphp$ of a method with the latent variables in the ground truth graph $\graph$.
To this end,
we first pad each result by adding latents that have no edge to any other variables to match the number of latents in the ground truth graph.
  On the other hand, if the number of latents is more than that of the ground truth $\graph$, all different combinations will be tried.
  Finally, we try all different permutations of latent variables to test the F1 score.
  For each method, the final F1 score is taken as  the best F1 score among all possible combinations and permutations.
  
\subsection{More Details of Experiments on Synthetic Data}
\label{appendix:Computational cost}
Our code is implemented with Python 3.7. 
Asymptotically speaking, if the ground gruth graph is a DAG,
then there will be no cycle in our result.
However, in the finite sample case, rank test results could be self-contradictary.
Therefore in our implementation we explictly prevent that 
 by checking whether cycles may occur every time before concluding a cluster.
As different methods employ different statistical tests that may perform differently, 
 the hyperparameter $\alpha$ is chosen from $\{0.1, 0.05, 0.01, 0.005\}$ 
 in favor of each method to ensure their best performance and thus a fair comparison.
For the proposed method we employ $\alpha=0.005$ for the procedure of finding latent variables, while for the first stage we empirically find that using a rather big $\alpha$ would be better. This is because the first stage of PC is good at deleting edges and thus bad at recall, and the following procedure would expects input with high recall rather than high precision.
We conduct all the experiments with
single Intel(R) Xeon(R) CPU E5-2470. Our proposed method and GIN~\citep{xie2020generalized} take around 3 hours to finish all the experiments (three random seeds and three different sample sizes), and Hier. rank~\citep{huang2022latent} takes around 1 hour. 
PC~\citep{spirtes2000causation} and FCI~\citep{spirtes2013causal} take around 10 minutes, while RCD~\citep{maeda2020rcd} takes around two days to finish the experiments. 
For GIN, RCD, and Hier. rank, 
we employ their original implementation while for PC and FCI we use the 
causal-learn python package \url{https://causal-learn.readthedocs.io/en/latest/}.

The time complexity of our proposed algorithm is upper bounded by 
$\mathcal{O}(l \sum_{k=1}^{K} \sum_{t=0}^{k} \binom{n}{t} \binom{n-t}{k+1-t})$, where $n$ 
is the number of measured variables, $K$ is the cardinality of the largest cover of the estimated graph, 
with $K \ll n$, and $l$ is the number of levels of the estimated graph, with $l \ll n$.

It is also possible to exhaustively enumerate all possible graphs
and check whether one of them may be aligned with
observational rank information from data.
However, that would be very computationally expensive.
Assume that the underlying graph consists of $n$ measured variables and $m$ latent variables.
To conduct an exhaustive search for causal clusters, 
we need to enumerate all possible numbers of latent vars and then enumerate all possible structures,
 which results in an approximate total of  $\sum_{i=1}^{m}3^{(i+n)(i+n-1)/2}$ possible combinations. 
 In our synthetic data, which comprises an average of 15 measured variables and 4 latent variables, 
 the mere act of enumeration already demands $10^{66}$ seconds (in our Python environment), which is computationally unacceptable.

\subsection{Detailed Information of the Big Five Personality Dataset}
\label{sec appendix: info of big5}

Data was collected through an interactive online personality test \url{https://openpsychometrics.org/}.
Participants were informed that their responses would be recorded and used for research at the beginning of the test and asked to confirm their consent at the end of the test.
Items were rated on a five-point scale:
1=Disagree, 2=Slightly disagree, 3=Neutral, 4=Slightly agree, 5=Agree (0=missed). Datapoints with missing values have been filtered out.
Some additional information is also collected including
Race, Age, and Gender but are not used in our experiment.
The Markov equivalence class of Figure~\ref{fig:big5} is generated
by using our proposed method, while we further apply GIN~\citep{xie2020generalized} to determine directions between latent variables.
The five  personality dimensions are Openness, Conscientiousness, Extraversion, Agreeableness, and Neuroticism (O-C-E-A-N). Below are the raw questions. E.g., E1 denotes the first question for the Extraversion score.

E1	I am the life of the party.\\
E2	I don't talk a lot.\\
E3	I feel comfortable around people.\\
E4	I keep in the background.\\
E5	I start conversations.\\
E6	I have little to say.\\
E7	I talk to a lot of different people at parties.\\
E8	I don't like to draw attention to myself.\\
E9	I don't mind being the center of attention.\\
E10	I am quiet around strangers.\\
N1	I get stressed out easily.\\
N2	I am relaxed most of the time.\\
N3	I worry about things.\\
N4	I seldom feel blue.\\
N5	I am easily disturbed.\\
N6	I get upset easily.\\
N7	I change my mood a lot.\\
N8	I have frequent mood swings.\\
N9	I get irritated easily.\\
N10	I often feel blue.\\
A1	I feel little concern for others.\\
A2	I am interested in people.\\
A3	I insult people.\\
A4	I sympathize with others' feelings.\\
A5	I am not interested in other people's problems.\\
A6	I have a soft heart.\\
A7	I am not really interested in others.\\
A8	I take time out for others.\\
A9	I feel others' emotions.\\
A10	I make people feel at ease.\\
C1	I am always prepared.\\
C2	I leave my belongings around.\\
C3	I pay attention to details.\\
C4	I make a mess of things.\\
C5	I get chores done right away.\\
C6	I often forget to put things back in their proper place.\\
C7	I like order.\\
C8	I shirk my duties.\\
C9	I follow a schedule.\\
C10	I am exacting in my work.\\
O1	I have a rich vocabulary.\\
O2	I have difficulty understanding abstract ideas.\\
O3	I have a vivid imagination.\\
O4	I am not interested in abstract ideas.\\
O5	I have excellent ideas.\\
O6	I do not have a good imagination.\\
O7	I am quick to understand things.\\
O8	I use difficult words.\\
O9	I spend time reflecting on things.\\
O10	I am full of ideas.\\

\subsection{More Analysis of the Results for the Big Five}
\label{appendix:bigfive analysis}
A prevalent theory of personality is that personality dimensions (factors or traits) are latent causes of the responses to personality inventory items, which are indicators of the latent construct. For instance, extraversion yields high scores for the indicators "I like to go to parties" and "I like people." Thus, the responses to inventory items are outcomes of one's position on the latent dimension. However, there is also the suggestion of a network perspective in which personality structure is viewed in terms of microcausal connections in a complex network \citep{wright2017factor} and personality dimensions emerge out of the connectivity structure \citep{cramer2012dimensions}. This is a radical divergence from the conventional viewpoint that dimensions are causes of the relevant indicators. For instance, a network would show that instead of being two distinct markers of extraversion, one might say "I like to go to parties" because "I like people" \citep{cramer2012dimensions}; or in the case of openness, “I am full of ideas” because “I have a vivid imagination”.
Our result in Figure~\ref{fig:big5} indicates, however, that our method adheres to a network perspective while identifying groups of closely connected items that are predictable under a latent dimension model. Further, it can be observed that causal links occur between latent dimensions, between observed indicators, and among latents and indicators. Following are interesting aspects of our results.

\textbf{(i)	L1, L2, L3 and L5.}
While L1, L2, and L3 clearly delineate conscientiousness, agreeableness, and extraversion as causes of the corresponding item responses, it is different in the case of openness (L5). Openness to experience can lead to excellent ideas brought about by active imagination, reflection, and understanding of things. Moreover, those who develop a rich vocabulary will have the propensity to think critically and read more, behaviors that also birth ideas.

\textbf{(ii) L1$\rightarrow$L6$\rightarrow$L3.}
Conscientiousness and openness are most frequently associated with achievement \citep{gatzka2021aspects}. In our results, those who are organized, efficient at tasks, thorough, systematic, or exacting will comprehend things quickly, demonstrate sophistication in language, or are good at introspection. These could instill a sense of confidence and assurance that encourages assertive, verbal, and bold behaviors, among other extraversion markers.

\textbf{(iii) L1$\rightarrow$L2$\rightarrow$L3.}
People who score highly on conscientiousness are frequently perceived as perfectionists, high achievers, overly focused on personal goals, preoccupied with flawless task execution, overly demanding, and headstrong~\citep{le2011too}\citep{curcseu2019personality}.  Our findings suggest that due to such behaviors, highly conscientious individuals would judge other people on their accomplishments and results, without giving consideration to others' feelings. Consequently, they act distant, uncommunicative, or unsociable. These can be reasons why they have been found not to engage in group behaviors that lead to straining relationships and tend to take criticism poorly \citep{curcseu2019personality}, as well as refrain from conversing about interpersonal issues because they have no bearing on achieving task objectives \citep{curcseu2019personality}. 
On the other hand, those who are both highly conscientious and agreeable put the needs of others before their own \citep{lord2007neo}, at times to the point of pleasing others by overlooking their mistakes, doing things for others because they cannot say no, not disclosing performance gaps and withholding dissident opinions out of aversion to conflict and a lack of competitiveness \citep{graziano2002agreeableness}\citep{howard2010owner}\citep{curcseu2019personality}. Thus, they will resolve problems on their own, not to draw attention to themselves, have little to say, or would rather stay in the background in order to be integrated. Those who score low on agreeableness are perceived to be unempathetic, unfriendly, and untrustworthy, consequently leading to introvertive behaviors as well.

\textbf{(iv) L1 and L3 together as common causes.}
Conscientious individuals who care about being liked by others despite being focused, detail-oriented, and exacting, will make efforts to make people feel at ease amid these behaviors. Otherwise, conscientious individuals who care not about what others think of them could be quick to lambast others if they perform poorly. Low-conscientious people, those who are disorganized, messy, sloppy, and negligent, will also tend to make people at ease in order to remain in their good graces. 

\textbf{(v)	No latent variable for neuroticism.}
Our method did not discover any latent variable that is supposed to correspond to neuroticism. One possible interpretation would be that the question "[N10]: I often feel blue." is designed so well that it fully captured the sense of neuroticism.

\textbf{(vi) Responses to indicators influence other responses.}
It is the question items, not the latent dimensions, 
that can be perceived to have caused the succeeding responses,
 as in the case of N10$\rightarrow$N8$\rightarrow$N7, N10$\rightarrow$O9, O2$\rightarrow$O4,
  and O1$\rightarrow$O8, all of which are plausible. Mood swings are common with depression, and frequent mood swings cause emotions to fluctuate rapidly and intensely, switching between positive and negative emotions. Some people may find themselves reflecting a lot on things because they are trying to figure out what makes them frequently feel blue and how to cope with it. A person who has trouble understanding abstract concepts is unlikely to be particularly interested in them. Finally, one who has a rich vocabulary will not be constrained from using unusual words that are hard to comprehend.


\section{Related Work and Broader Impacts}
\subsection{Related Work}
\label{sec appendix: related work}
  Causal discovery aims to identify causal relationships from observational data.
  Most existing approaches are based on the assumption that there are no latent confounders \citep{spirtes2000causation,chickering2002optimal,lingam,hoyer2009ANM,zhang2009PNL},
  and yet this assumption barely holds for real-life problems.
  Thus, causal discovery methods that can handle the existence of latent variables are crucial.
  Existing causal discovery methods for  handling latent variables can be categorized into the following folds.
  
  \textbf{(i) Conditional independence constraints.} The FCI algorithm \citep{spirtes2000causation} and its variants        \citep{colombo2012learning,Pearl:2000:CMR:331969,Akbari2021}.
      This line of work checks conditional independence  over observed variables 
       to identify the causal structure over observed variables up to a maximal ancestral graph.
       They can deal with both linear and nonlinear causal relationships,
         but there are large indeterminacies in their results,
         e.g., the existence of an edge and confounders. Plus, they cannot consider causal relationships between latent variables.
         Based on CI tests, \citet{triantafillou2015constraint} proposes a method that can co-analyze multiple datasets that share common variables and sort the significance tests to address conflicts from statistical errors. 
      \textbf{(ii) Tetrad condition.} This line of work makes use of the rank constraints of every $2 \times 2$ off-diagonal sub-covariance matrix
      to  locate latent variables and thus find the causal skeleton based on linear relationships between variables
       \citep{Silva-linearlvModel, Kummerfeld2016, Shuyan20,Pearl88}. 
       One limitation of this line of work is that they assume  each measured variable is influenced by only one latent parent, and each latent variable must have more than three pure measured children.
      \textbf{(iii) Matrix decomposition.} This line of work proposes to decompose the precision matrix into a low-rank matrix and a sparse matrix, where the former represents the causal structure from latent variables to measured variables and the latter represents the causal structure over measured variables, under certain assumptions \citep{RankSparsity_11, RankSparsity_12,anandkumar2013learning}.
      E.g., \citet{anandkumar2013learning} decomposed the covariance matrix into a low-rank matrix and a diagonal matrix, by assuming three times more measured variables than latent variables. 
      \textbf{(iv) Over-complete independent component analysis (ICA).}
      Over-complete ICA allows more source signals than observed signals,
      and thus can be used to learn the causal structure with latent variables \citep{shimizu2009estimation},
      and yet they normally do not consider the causal structure among latent variables. The estimation of over-complete ICA models could be hard to reach global optimum without further assumptions \citep{entner2010discovering, tashiro2014parcelingam}.
      \textbf{(v) Generalized independent noise (GIN).} The GIN condition is
       an extension of the independent noise condition when latent variables exist. Based on non-gaussianity it leverages higher-order statistics to identify 
        latent structures.
        E.g., \citet{xie2020generalized} allows 
        multiple latent parents behind every pair of observed variables and can identify causal directions among latent variables, and yet it requires at least twice measured children as latent variables. \cite{dai2022independence} proposes a transformed version of GIN to handle measurement errors.
      \textbf{(vi) Mixture oracles-based.}  \citet{kivva2021learning} proposes a mixture oracles-based method
       to identify the causal structure in the presence of latent variables where the causal relationships can be nonlinear. It is based on assumptions that
        the latent variables are discrete and each latent variable has measured variables as children.
    \textbf{(vii) Rank deficiency.} Recently \citet{huang2022latent} proposes to leverage rank deficiency of sub-covariance of observed variables to find the underlying causal structure in the presence of latent variables.
    Our method differs in that we consider a more general setting, i.e., we allow hidden variables can be causally related to each other, form a hierarchical structure (i.e., the children of hidden variables can still be hidden), and even serve as both confounders and intermediate variables for observed variables. 
    Our graphical condition for identifiability also generalizes the condition in \citep{huang2022latent} to cases where edges between observed variables are allowed.
    \textbf{(viii) Heterogeneous data.} \citet{CDNOD_jmlr} considere a special type of latent confounders that can be represented as a function of domain index or a smooth function of time. This line of work makes use of domain index or time index as a surrogate  to remove confounders' influence and consequently identify causal structure over observed variables.
    \textbf{(ix) Score based.} \citet{agrawal2021decamfounder} propose a score-based method for latent variable causal discovery, by assuming additional structure among latent 
    confounders.

The most related work to our method is Hier. Rank \cite{huang2022latent}.
Compared to Hier. Rank, our graphical conditions are not only strictly but also much weaker.
Our conditions are strictly weaker in the sense that Hier. Rank can be taken as a special case of the proposed method by disallowing direct edges between observed variables, which is formally captured by our Corollary~\ref{cor:pcandrank}.
Our conditions are much weaker in the sense that, basically we allow latent variables and observed variables to be flexibly related and exist everywhere in a graph, which is illustrated in Figure~\ref{fig:compare graphs with each method} (b) v.s. Figure~\ref{fig:compare graphs with each method} (d). 
The reason why we are able to identify these latent structures that Hier. Rank cannot identify, lies in that we utilize the rank constraints in a more flexible and comprehensive fashion. 
Specifically, Hier. Rank only uses the part of the rank information $\text{rank}(\Sigma_{\mathbf{A},\mathbf{B}})$ where $\mathbf{A}\cap\mathbf{B}=\emptyset$, while
the proposed method uses the $\text{rank}(\Sigma_{\mathbf{A},\mathbf{B}})$ where $\mathbf{A}$ and $\mathbf{B}$ are rather arbitrary and thus more t-separations can be inferred.
The extra graphical information allows us to make use of, e.g., Lemma ~\ref{lemma: rank and d-sep} for identifying edges between observed variables, and Theorem~\ref{theorem:unique_rank_of_atomic_cover} to identify atomic covers that are partially hidden partially observed.

\begin{center}
\begin{table}[tb]
\vspace{-0mm}
 
 \caption{ \small{Structural Hamming Distance (SHD) of compared methods on different types of latent graphs where the values are averaged over three random seeds. \textbf{The smaller the better.}} }
   \vspace{-0mm}
   \label{tab:SHD v}
  \footnotesize
  \center 
\begin{center}
\begin{tabular}{|c|c|c|c|c|c|c|c|}
  \hline  \multicolumn{2}{|c|}{} &\multicolumn{6}{|c|}{\textbf{SHD for skeleton among all variables $\set{V}_\graph$ (both $\set{X}_\graph$ and $\set{L}_\graph$)}}\\
  \hline 
  \multicolumn{2}{|c|}{Algorithm}  & {\color{white}{0}}\textbf{Ours}{\color{white}{0}} & Hier. rank  & {\color{white}{0}}PC{\color{white}{0}} & {\color{white}{0}}FCI{\color{white}{0}} & {\color{white}{0}}GIN{\color{white}{0}} &{\color{white}{0}}RCD{\color{white}{0}}\\
  \hline 
    & 2k 
    & {\color{white}{0}}\textbf{6.9} & {\color{white}{0}}9.3 & 23.7 & 23.7 &   20.5 & 22.2\\
  \cline{2-8}
  {\emph{Latent+tree}}
    &5k 
    & {\color{white}{0}}\textbf{3.2} & {\color{white}{0}}9.0 & 25.0 & 24.3 &   21.2 & 23.5 \\
  \cline{2-8}
    &10k
    & {\color{white}{0}}\textbf{0.7} &{\color{white}{0}}9.0  &25.1 & 24.3 &    20.0& 24.0 \\
  \hline 
    & 2k 
    & {\color{white}{0}}\textbf{4.6} & {\color{white}{0}}8.1 &14.3 &  14.7 & 10.8&15.4 \\
    \cline{2-8}
    {\emph{Latent+measm}}
    &5k 
    & {\color{white}{0}}\textbf{3.8} & {\color{white}{0}}7.7 & 15.0&15.0 &  {\color{white}{0}}9.2&16.2 \\
  \cline{2-8}
    &10k
    & {\color{white}{0}}\textbf{2.9} & {\color{white}{0}}7.4 & 15.5 & 14.8 &  {\color{white}{0}}9.2& 16.0\\
  \hline 
    & 2k 
    & \textbf{27.1} & 28.1 & 38.0 & 37.6 &   36.4&36.5  \\
  \cline{2-8}
  {\emph{Latent general}} &5k 
  & \textbf{23.0} & 26.0& 38.2 & 36.8 &   33.8& 32.5 \\
  \cline{2-8}&10k 
  & \textbf{21.4}& 26.0 &39.0  & 37.1 &  34.1& 36.1\\
  \hline 
\end{tabular}
\end{center}
\vspace{-0mm}
\end{table}
\end{center}

\section{Additional Information}
\subsection{Empirical Result using SHD}
\label{result using SHD}
In this section we further show the performance of each method using the Structural Hamming Distance (SHD). As shown in Table~\ref{tab:SHD v}, The SHD of the proposed RLCD method to the ground truth is consistently smaller than all comparative methods under all the settings, which again validates RLCD in the finite sample cases.

\begin{figure}[tb]
  \vspace{-0mm}
  \centering
  \begin{subfigure}[b]{0.37\textwidth}
    \centering
    \includegraphics[width=\textwidth]{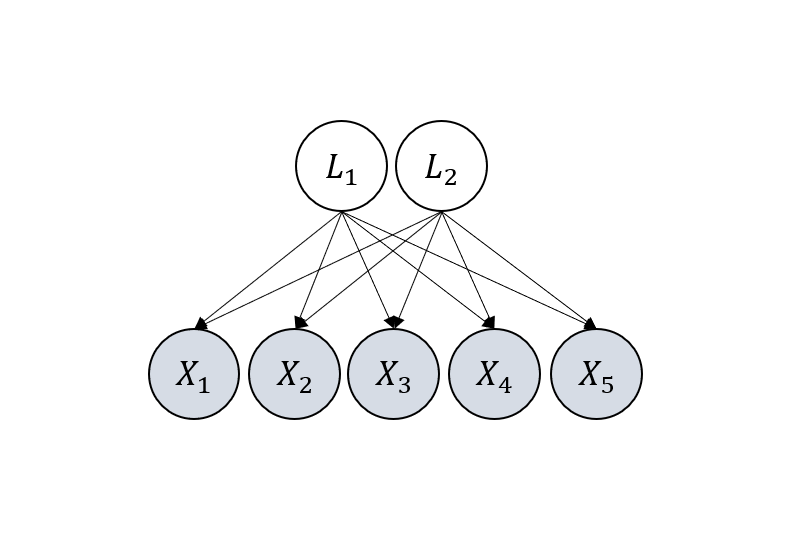}
    \caption{{In $\graph_1$, Condition~\ref{cond:basic} is not satisfied, in the sense that $\node{L_1}\node{L_2}$ do not have enough pure children and neighbours.}} 
  \end{subfigure}
  \begin{subfigure}[b]{0.37\textwidth}
    \centering
    \includegraphics[width=\textwidth]{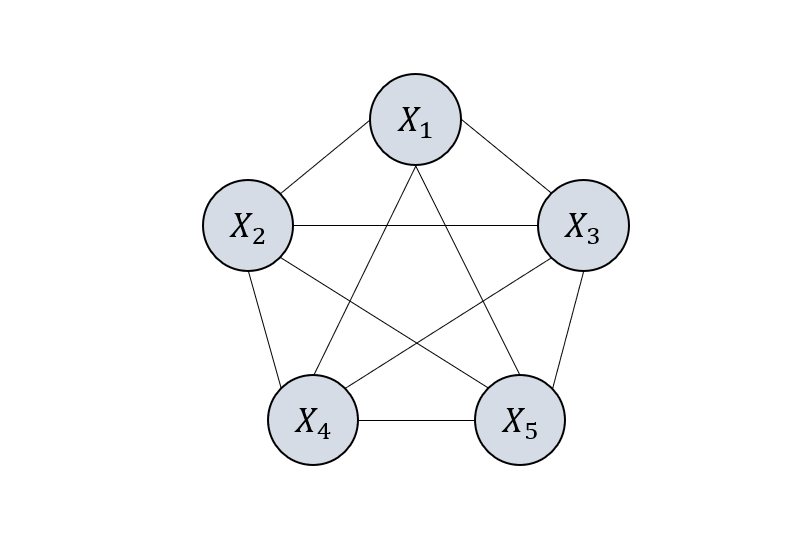}
    \caption{{Given $\graph_1$, RLCD outputs $\graph_1'$, which is not informative but correct as they are rank-equivalent.}}
  \end{subfigure}
  \begin{subfigure}[b]{0.37\textwidth}
    \centering
    \includegraphics[width=\textwidth]{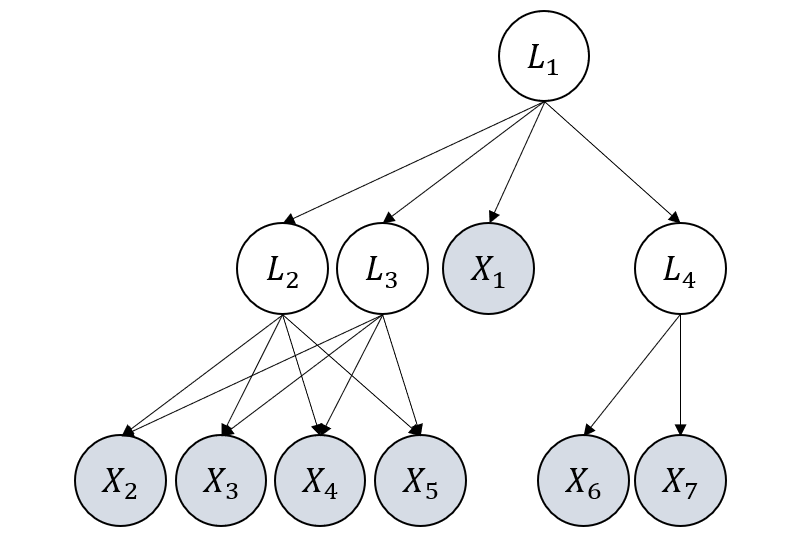}
    \caption{{In $\graph_2$, Condition~\ref{cond:basic} is not satisfied, in the sense that $\node{L_4}$ does not have enough pure children and neighbours.}}
  \end{subfigure}
  \begin{subfigure}[b]{0.37\textwidth}
    \centering
    \includegraphics[width=\textwidth]{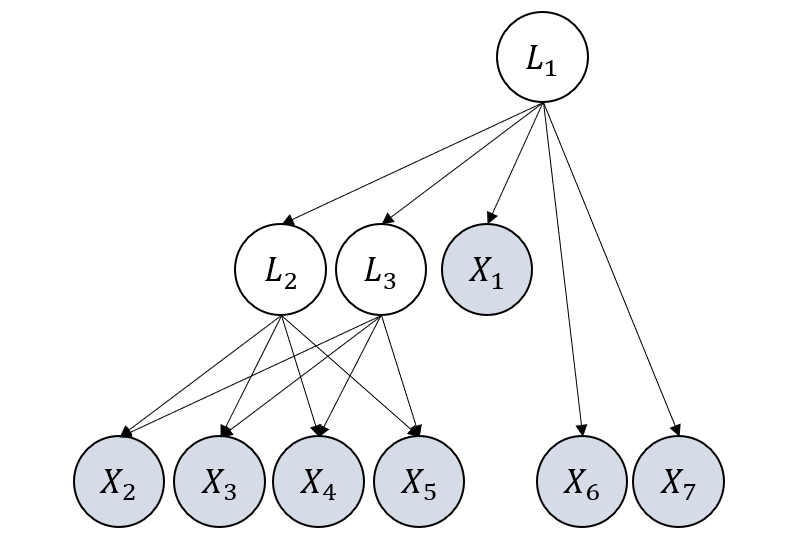}
    \caption{{Given $\graph_2$, RLCD outputs $\graph_2'$, which is correct and informative, though not exactly the same as $\graph_2$.}}
  \end{subfigure}
  \begin{subfigure}[b]{0.37\textwidth}
    \centering
    \includegraphics[width=\textwidth]{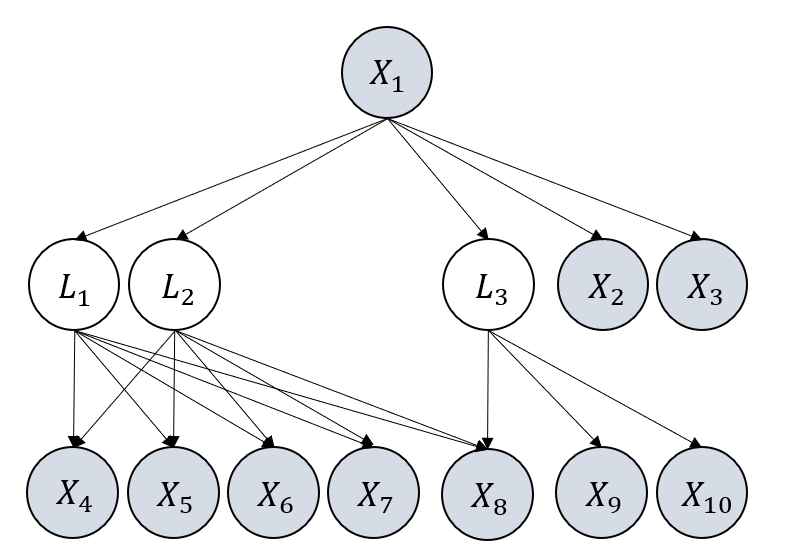}
    \caption{{In $\graph_3$, the Condition~\ref{cond:vstructure} is not satisfied, because $|\{\node{L_1},\node{L_2}\}|+|\{\node{L_3}\}| > |\{\node{X_1}\}|+|\{\node{X_8}\}|$.}}
  \end{subfigure}
  \begin{subfigure}[b]{0.37\textwidth}
    \centering
    \includegraphics[width=\textwidth]{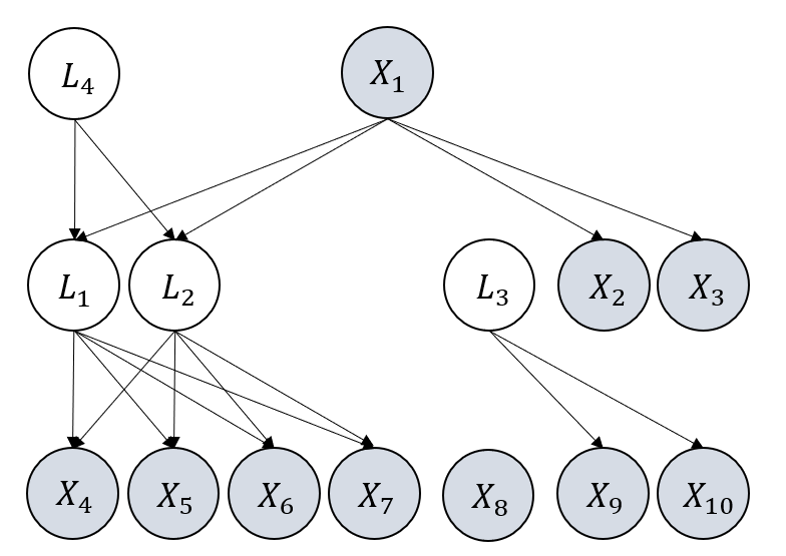}
    \caption{{Given $\graph_3$, RLCD outputs $\graph_3'$, which is incorrect but we can infer from the result that conditions are violated.}}
  \end{subfigure}
  \caption{{Examples to show that even when graphical conditions are not satisfied, the proposed RLCD can provide the correct result (though may be uninformative), or infer that the condition is violated.}}
  \label{fig:violation of conditions}
\end{figure}

\subsection{Violation of Graphical Conditions}
In this section, we shall discuss what would the output of the proposed method be when graphical conditions~\ref{cond:basic}  \ref{cond:vstructure} are not satisfied.

(i) The result is correct but uninformative. For instance, in Figure~\ref{fig:violation of conditions} (a), $\node{L_1}\node{L_2}$ do not belong to any atomic cover, as they do not have enough pure children plus neighbours. Thus, Condition~\ref{cond:basic} is not satisfied. In this scenario, though the output of the proposed method is not informative, the result is correct. It is not informative in the sense that the result in Figure~\ref{fig:violation of conditions} (b) fails to inform us the existence of latent variables. The result is correct in the sense that it correctly outputs the CI skeleton, which is rank-equivalent to the ground truth $\graph_1$. In other words, the output graph $\graph_1'$ and the ground truth $\graph_1$ are able to entail the same set of observational rank constraints and no algorithm can differentiate them solely by rank information.

(ii) The result is correct and informative, though it uses a more compact graph as the representation of the rank-equivalence class.
An example is given in Figure~\ref{fig:violation of conditions} (c), where $\node{L_4}$ does not belong to any atomic cover (as no enough pure children plus neighbours), and thus $\graph_2$ does not satisfy Condition~\ref{cond:basic}. However, the proposed RLCD still outputs the correct and informative result $\graph_2'$.  The output $\graph_2'$ is correct in the sense that $\graph_2'$ and  $\graph_2$ are rank-equivalent, i.e., they entail the same set of observational rank constraints, and the local violation of Condition~\ref{cond:basic} does not harm the correctness of other substructures. The output is also informative as $\graph_2'$ is a compact representation of the rank-equivalence class that contains $\graph_2$; one can easily infer many other members of the class from $\graph_2'$.

(iii) The result is incorrect, but from the result we can infer that the conditions are violated. An example is given in Figure~\ref{fig:violation of conditions} (e) where for the v structure in $\graph_3$  we have $|\{\node{L_1},\node{L_2}\}|+|\{\node{L_3}\}| > |\{\node{X_1}\}|+|\{\node{X_8}\}|$ and thus Condition~\ref{cond:vstructure} is not satisfied. In this scenario, RLCD will output $\graph_3'$, which is incorrect. However, we can easily infer that this result is abnormal, as the result is not consistent with the CI skeleton: in $\graph_3'$, there are 3 subgroups that are not connected to each other, while the CI skeleton would inform us that all the variables are directly or indirectly connected.

We note that the above analysis holds in the large sample limit. In the finite sample cases, the result of statistical tests could be incorrect and thus self-contradictory.

\subsection{Triangle Structure}
\label{def of triangle}
The definition of  triangle is as follows. If three variables are mutually adjacent, then they form a triangle. For example, in Figure~\ref{fig:triangle}, $\node{L_1},\node{L_3},\node{L_4}$ form a triangle and $\node{X_5},\node{X_6},\node{X_7}$ form a triangle. Our Condition~\ref{cond:basic} rules out the case of the triangle among $\node{L_1},\node{L_3},\node{L_4}$ as it involves latent variables, but does not rule out the triangle among $\node{X_5},\node{X_6},\node{X_7}$.

  \begin{figure}[t]
    \setlength{\abovecaptionskip}{0mm}
\setlength{\belowcaptionskip}{0mm}
  \centering 
  \includegraphics[width=0.5\linewidth]{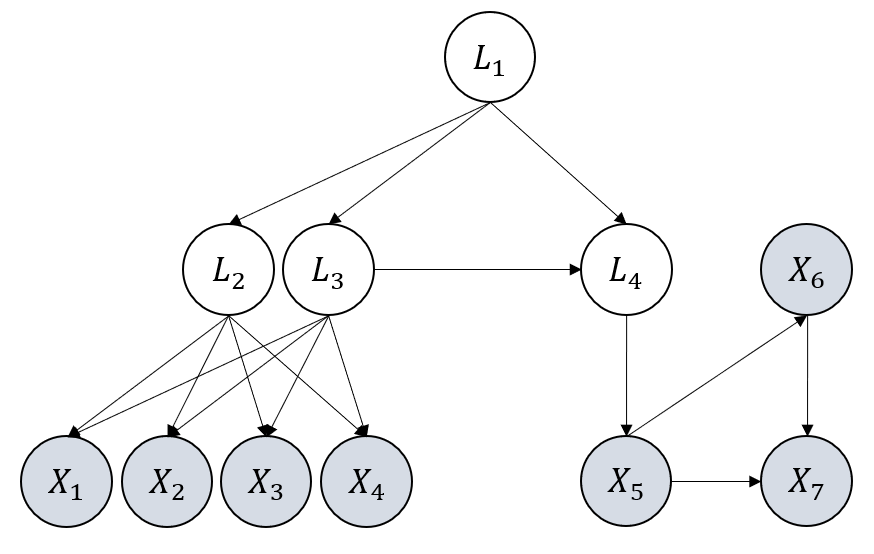}
  \caption{{An illustrative example to show the triangle structure.}}
  \label{fig:triangle}
\end{figure}

\begin{figure}[tb]
  \vspace{-0mm}
  \centering
  \begin{subfigure}[b]{0.65\textwidth}
    \centering
    \includegraphics[width=\textwidth]{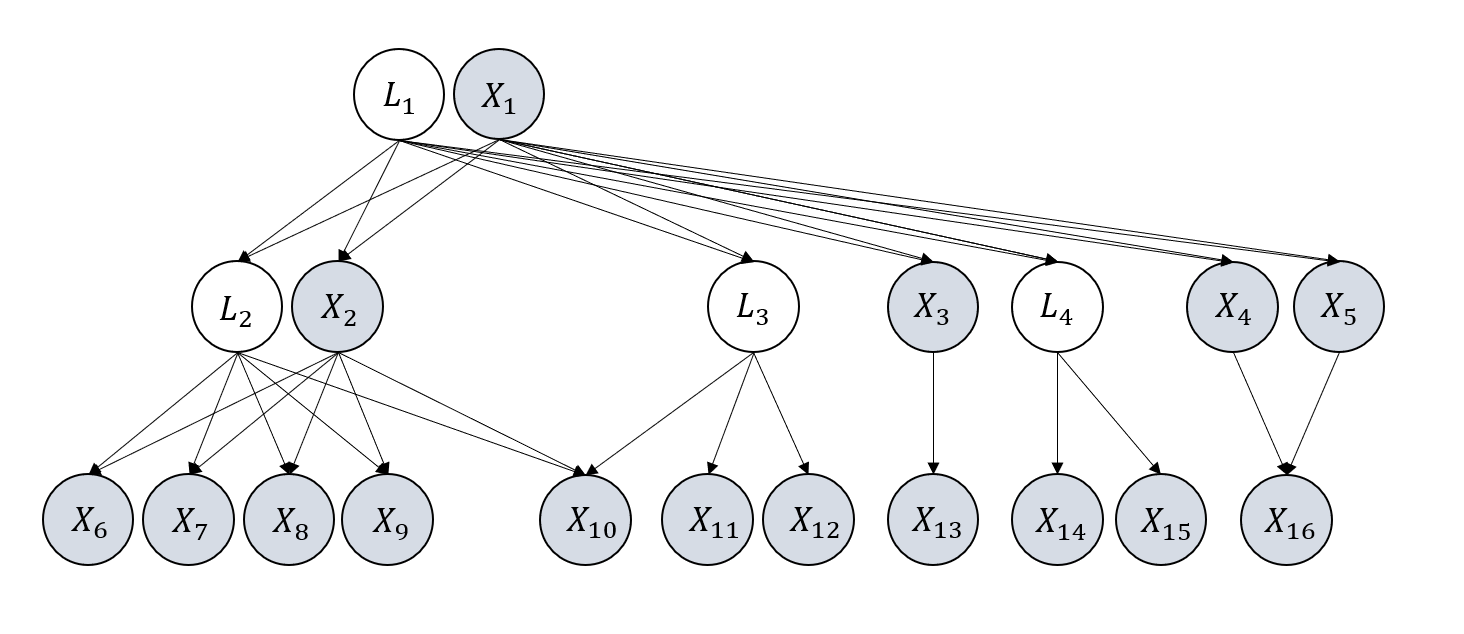}
    \caption{{Input  $\graph$, a latent general graph.}}
  \end{subfigure}
  \begin{subfigure}[b]{0.35\textwidth}
    \centering
    \includegraphics[width=\textwidth]{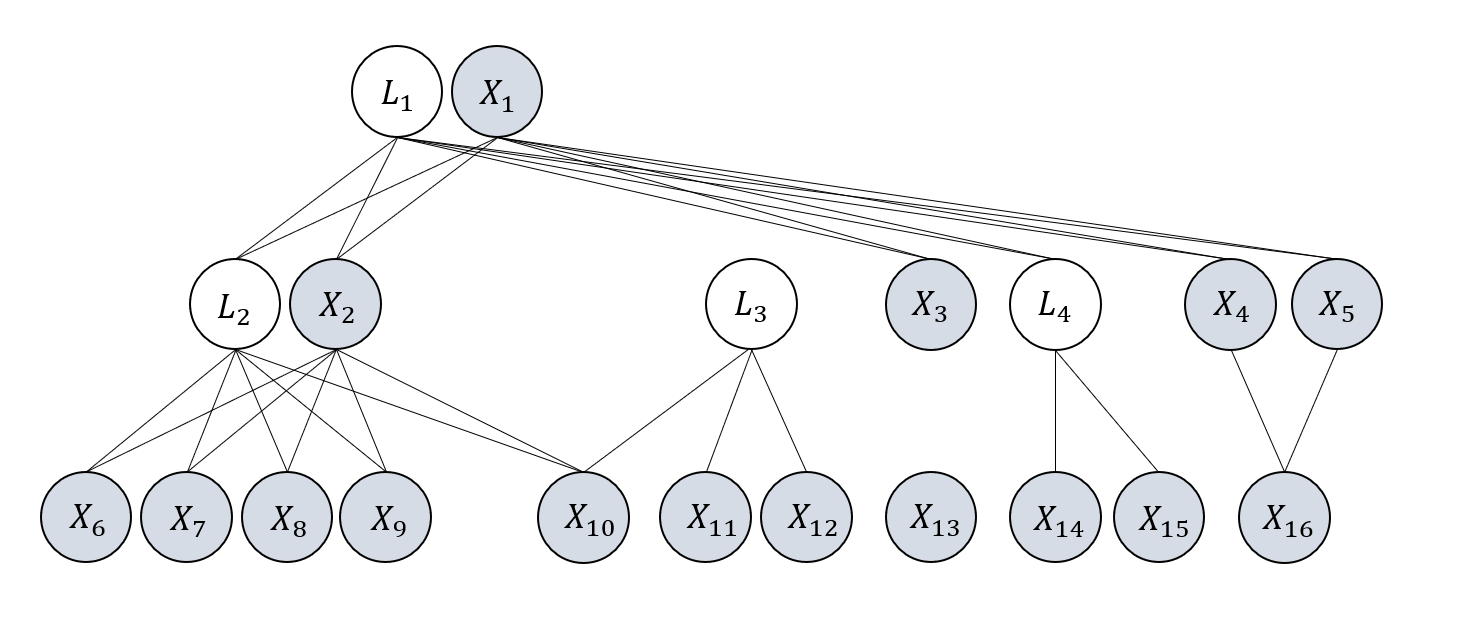}
    \caption{{The output of RLCD, $\graph_1$.}}
  \end{subfigure}
  \begin{subfigure}[b]{0.35\textwidth}
    \centering
    \includegraphics[width=\textwidth]{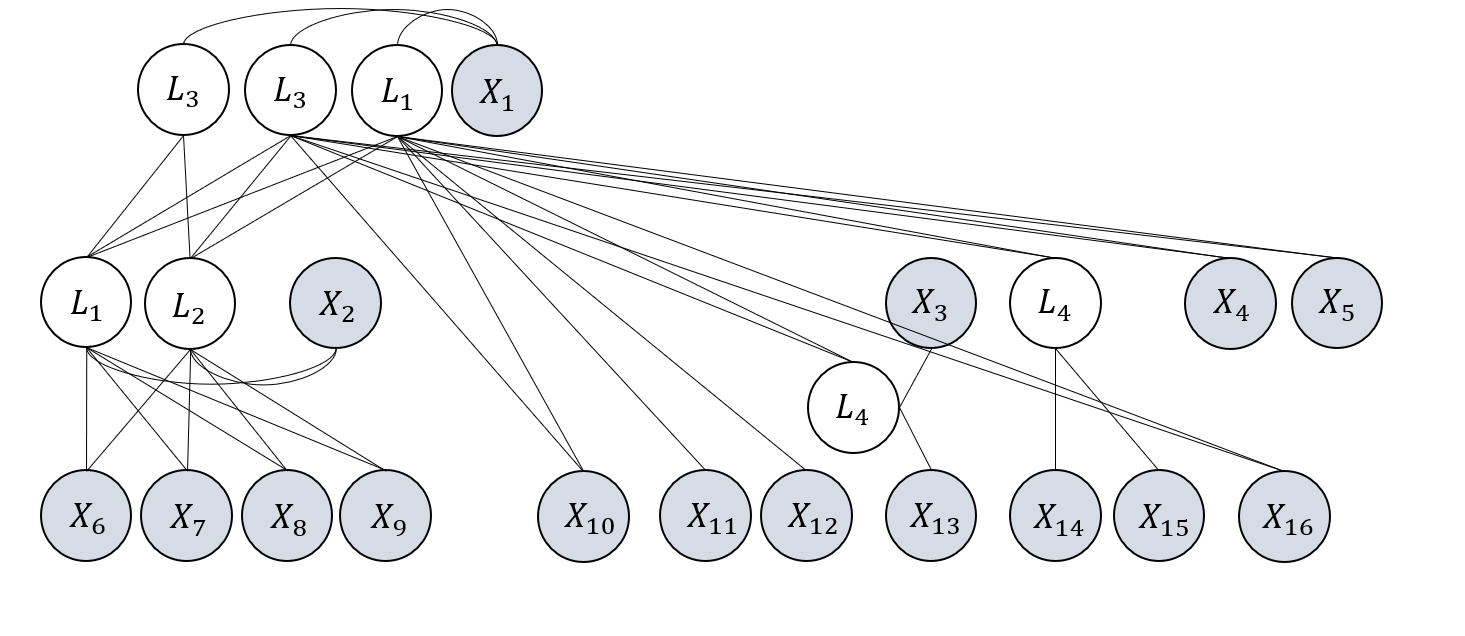}
    \caption{{The output of Hier. Rank, $\graph_2$.}}
  \end{subfigure}
  \begin{subfigure}[b]{0.35\textwidth}
    \centering
    \includegraphics[width=\textwidth]{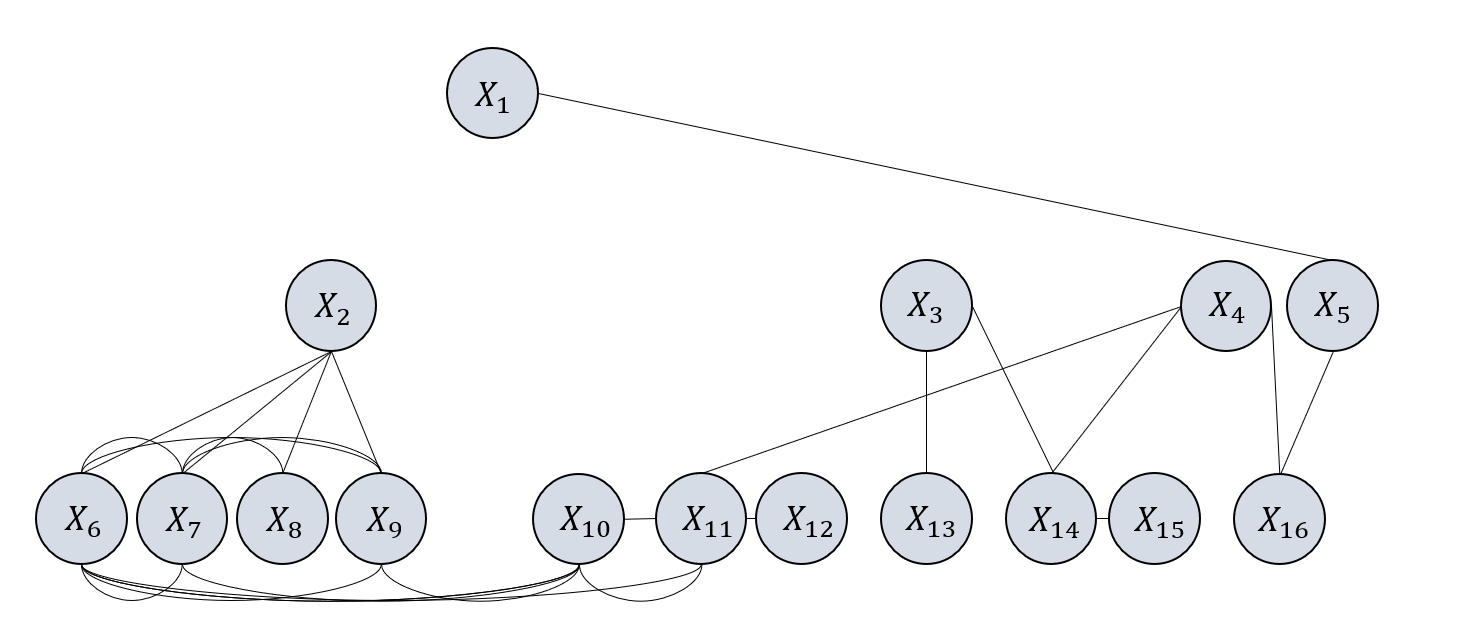}
    \caption{{The output of FCI, $\graph_3$.}}
  \end{subfigure}
  \begin{subfigure}[b]{0.35\textwidth}
    \centering
    \includegraphics[width=\textwidth]{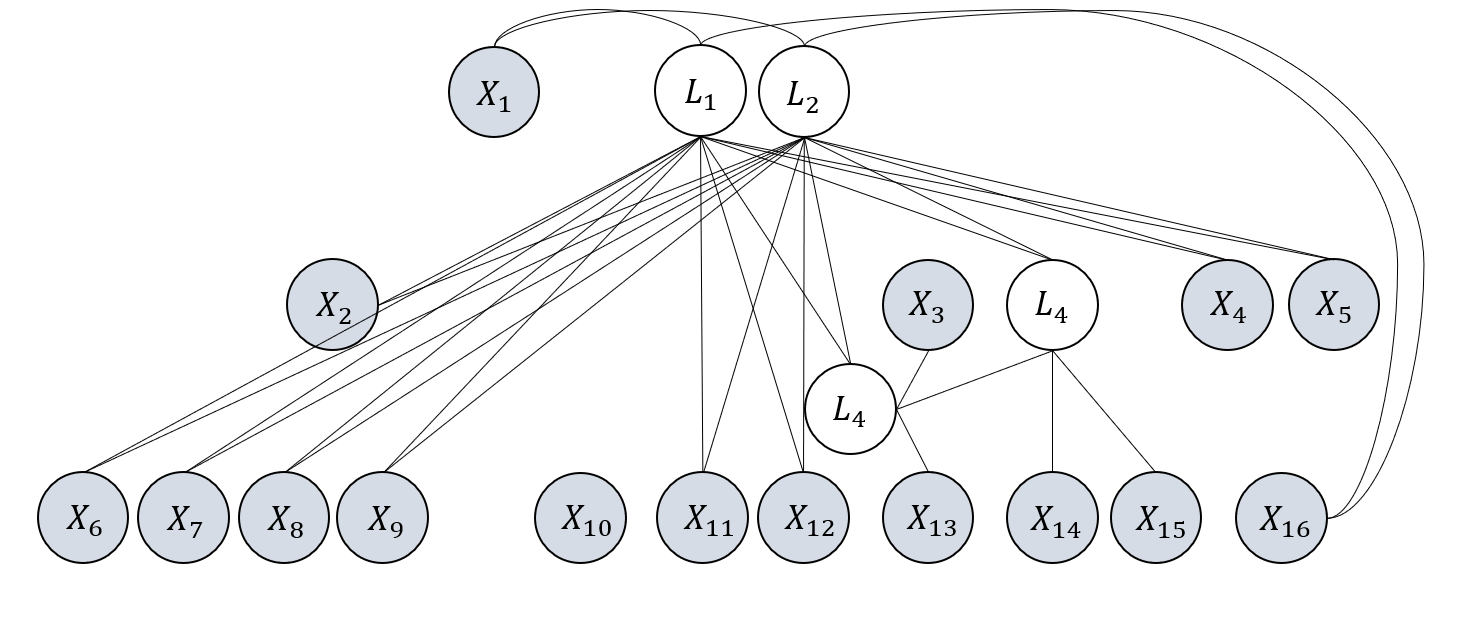}
    \caption{{The output of GIN, $\graph_4$.}}
  \end{subfigure}
  \caption{{Illustrative examples of the output of each method given data generated from $\graph$. We only show the skeleton here for simplicity.}}
  \label{fig:output of each method}
\end{figure}

\subsection{Illustrative Outputs of Each Method.}
Here we give illustrative outputs of each method to give an intuitive understanding of why the proposed RLCD performs the best. Here we use 10000 data points and only show the skeleton for simplicity.

The underlying graph $\graph$ is shown in Figure~\ref{fig:output of each method} (a) and it has 16 observed variables and 4 latent variables. The output of the proposed RLCD, $\graph_1$ is given in Figure~\ref{fig:output of each method} (b) and it is nearly the same as the input $\graph$. As it is the finite sample case, there are still two edges missing due to the error of the statistic test. However, by increasing the sample size, the missing two edges can also be recovered.

The output of Hier. rank, $\graph_2$ is given in Figure~\ref{fig:output of each method} (c). It introduces redundant latent variables compared to the ground truth, and many edges are incorrect. The underlying reason is that Hier. rank cannot handle the scenario where observed variables are directly adjacent and it does not allow edges from an observed variable to a latent variable.

The output of FCI, $\graph_3$ is given in Figure~\ref{fig:output of each method} (d). It can neither identify the cardinality nor the location of latent variables. Furthermore, many edges between observed variables are missing or incorrect.
This might be due to the fact that FCI relies on CI tests conditioned on multiple variables but tests of independence conditional on
large numbers of variables have very low power \cite{spirtes2001anytime}.

The output of GIN, $\graph_4$ is given in Figure~\ref{fig:output of each method} (e). The number of latent variables it discovered is correct, but they do not exactly correspond to the latent variables in the ground truth $\graph$. Plus, many edges are incorrect, and the reason is similar to that for Hier. rank, i.e., GIN cannot handle direct edges between observed variables and edges from an observed variable to a latent variable

\subsection{Detailed Discussion about Combining Result of Phases 2 \& 3 with Phase 1.}
\label{combine result}

(i) ``Transfer the estimated DAG $\graph''$ to Markov equivalence class" (the first part of Line 7 in Algorithm~\ref{alg:all}): Here, we transfer the output of Phases 2 \& 3, i.e., $\graph''$, to a Completed Partial Directed Acyclic Graph (CPDAG), which represents the corresponding Markov Equivalence Class. 
Algorithms for transferring a DAG to CPDAG have been well studied \citep{chickering2002optimal,chickering2002learning,chickering2013transformational} with implementations available from e.g., causal-learn \citep{causallearn} or causaldag \citep{squires2018causaldag} python package.

(ii) ``Update $\graph'$ by $\graph''$" (the second part of Line 7 in Algorithm~\ref{alg:all}):
Note that the input to Phases 2 \& 3 is $\set{X}_\setset{Q}\cup\set{N}_\setset{Q}$, and let $\set{L_\setset{Q}}$ be the newly discovered latent variables during Phases 2 \& 3.
By
Theorem~\ref{theorem:necessary and sufficient under cond},  latent variables must be in treks between variables in $\set{X}_\setset{Q}$. Therefore, we just need to consider the edges among $\set{X}_\setset{Q}\cup\set{L_\setset{Q}}$.
Specifically, we first
  delete all the edges among $\set{X}_\setset{Q}$ in graph $\graph'$, then add $\set{L_\setset{Q}}$ to $\graph'$, and finally add 
all the edges among $\set{X}_\setset{Q}\cup\set{L_\setset{Q}}$ from $\graph''$ to $\graph'$.

(iii) ``Orient remaining causal directions that can be inferred from v structures"  (Line 8 in Algorithm~\ref{alg:all}): Given $\graph'$, a partially directed acyclic graph (PDAG), here we orient all remaining causal directions that can be determined.
This can be easily achieved by following stage 2 of PC \citep{spirtes2000causation}: we use the CI results detected in Phase 1 to find v-structures among observed variables and apply Meek's rule \citep{meek2013causal} to infer the remaining directions that can be decided. After that, we then transfer this PDAG to a CPDAG (well studied in \cite{chickering2002optimal,chickering2002learning,chickering2013transformational} with implementations available) to get the corresponding Markov equivalence class.

\end{appendices}

\end{document}